\documentclass[Afour,sageh,times]{sagej}

\usepackage{moreverb,url}

\usepackage[colorlinks,bookmarksopen,bookmarksnumbered,citecolor=red,urlcolor=red]{hyperref}

\newtheorem{theorem}{Theorem}
\newtheorem{corollary}{Corollary}
\usepackage{dirtytalk}
\usepackage{subcaption}
\usepackage{xcolor}
\usepackage{soul}
\usepackage{gensymb}
\newcommand{\mb}[1]{\mathbf{#1}}

\newcommand\BibTeX{{\rmfamily B\kern-.05em \textsc{i\kern-.025em b}\kern-.08em
T\kern-.1667em\lower.7ex\hbox{E}\kern-.125emX}}

\def\volumeyear{2026}

\begin{document}

\runninghead{Flight and Gosselin}

\title{A Generalized Theory of Load Distribution in Redundantly-actuated Robotic Systems}

\author{Joshua Flight\affilnum{1} and Clément Gosselin\affilnum{1}}

\affiliation{\affilnum{1}Département de génie mécanique et de génie industriel, Université Laval, Canada.}

\corrauth{Joshua Flight, Département de génie mécanique et de génie industriel, Université Laval, Québec, Québec, G1V 0A6, Canada.}

\email{joshua.flight.1@ulaval.ca}

\begin{abstract}
This paper presents a generalized theory which describes how applied loads are distributed within rigid bodies handled by redundantly-actuated robotic systems composed of multiple independent closed-loop kinematic chains. The theory fully characterizes the feasible set of manipulating wrench distributions for a given resultant wrench applied to the rigid body and has important implications for the force-control of multifingered grippers, legged robots, cooperating robots, and other overconstrained mechanisms. We also derive explicit solutions to the wrench synthesis and wrench analysis problems. These solutions are computationally efficient and scale linearly with the number of applied wrenches, requiring neither numerical methods nor the inversion of large matrices. Finally, we identify significant shortcomings in current state-of-the-art approaches and propose corrections. These are supported by illustrative examples and a simulation that demonstrate the advantages of the improved methods.
\end{abstract}


\maketitle

\section{Introduction}
Modern robotics relies increasingly on systems in which a single object is manipulated or supported by multiple independently actuated kinematic chains. Examples include multifingered robotic hands (\cite{Yoshikawa1991, Zuo1999, Kiatos2021, Andrychowicz2020, OpenAI2019}), legged walking robots (\cite{Gerardo2018, Medeiros2020}), and cooperating manipulators handling a common object (\cite{Walker1991, Williams1993, Chung2005, Erhart2015, Verginis2023}).

In these examples, the mapping of applied wrenches to the resultant wrench that they apply to the object is written as 

\begin{equation}
    \mb{h}_o = \mb{G} \mb{h}
    \label{eq:staticEq}
\end{equation}
where $\mb{h}_o = \begin{bmatrix}
    \mb{f}_o^T & \mb{t}_o^T
\end{bmatrix}^T$ is the resultant wrench which includes a pure force $\mb{f}_o$ and a pure torque $\mb{t}_o$, $\mb{h} = \begin{bmatrix}
    \mb{h}_1^T & \ldots & \mb{h}_n^T
\end{bmatrix}^T$ is the stacked vector of wrenches applied to the object by the kinematic chains where each $\mb{h}_i$ is defined similarly to $\mb{h}_o$, and $\mb{G}$ is the \emph{grasp matrix} (\cite{prattichizzo2008grasping}). 

What we consider to be the \say{kinematic chains} and the \say{object} depends on the context. For robotic hands, the kinematic chains are the fingers, which apply wrenches to the grasped object; for legged robots, the kinematic chains are the legs, which apply wrenches to the ground; for cooperating manipulators, the kinematic chains are the individual manipulators, which apply wrenches to a common object. Each example is illustrated in Figure \ref{fig:Examples}.

\begin{figure}[ht]
\centering
    \subfloat[Example of a multifingered gripper that applies pure forces to the handled object. Adapted from \cite{Kumar1988}.]{%
        \includegraphics[width=.7\linewidth]{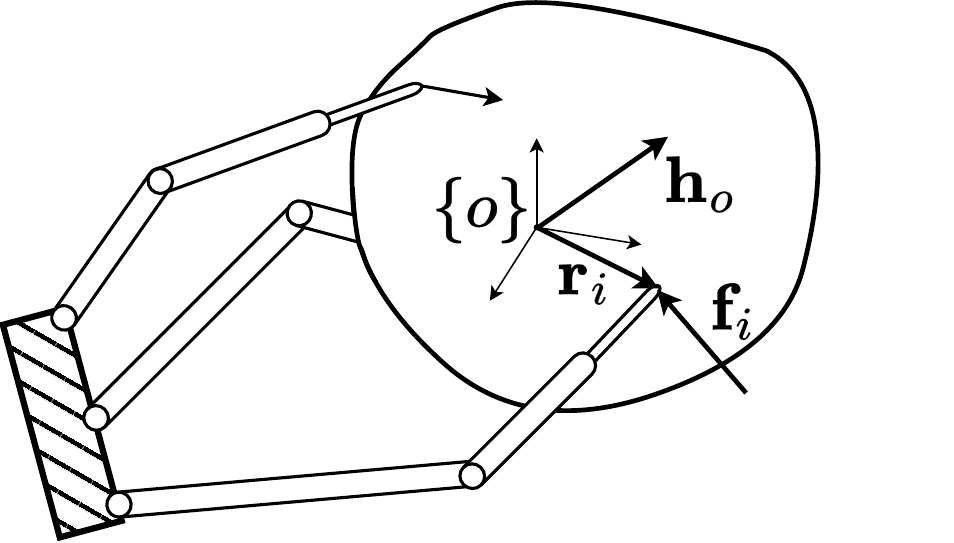}%
        \label{subfig:GripperExample}%
    } \\
    \subfloat[Example of a legged robot that applies pure forces to the ground.  Adapted from \cite{Gor2013}.]{%
        \includegraphics[width=.7\linewidth]{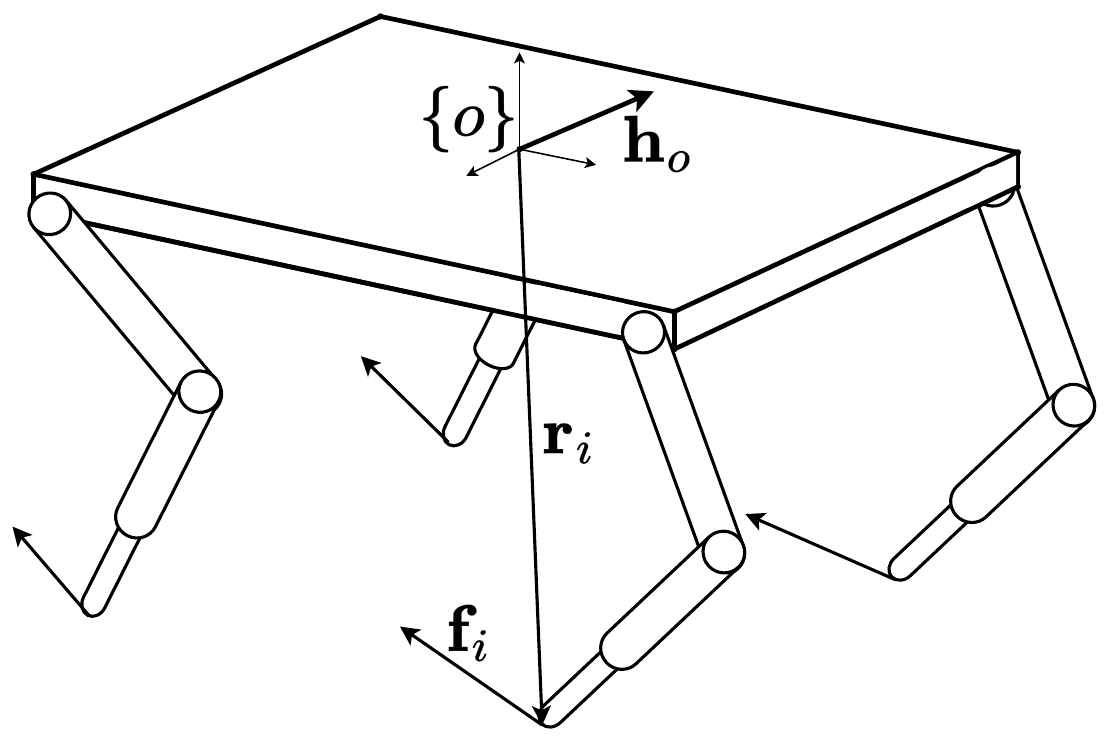}%
        \label{subfig:LeggedRobotExample}%
    } \\
    \subfloat[Example of cooperating manipulators each applying a pure force $\mb{f}_i$ and a pure torque $\mb{t}_i$ to the object.]{%
        \includegraphics[width=.7\linewidth]{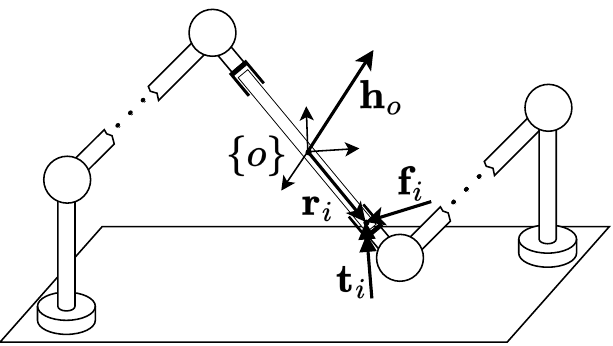}%
        \label{subfig:CoopManipsExample}%
    }
    \caption{Examples of rigid objects handled by multiple kinematic chains.}
    \label{fig:Examples}
\end{figure}

In these systems, the wrenches applied to the object are typically redundant and thus infinitely many applied wrench combinations can produce the same net wrench on the object. On the one hand, this redundancy may be leveraged to improve grip tightness, distribute applied loads more evenly among the actuators, or to allow singular configurations to be avoided. On the other hand, the applied wrenches $\mb{h}_i$ may have antagonistic components which produce internal loads (i.e., combinations of applied wrenches that cancel and thus are invisible to the net wrench but produce nonzero stresses within the object). Internal loads do not affect the acceleration of the object, yet they have important consequences for the performance of the system. They can:

\begin{itemize}
    \item induce undesired contact pressures that can damage a grasped object;
    \item create friction and wear that reduce the efficiency and lifespan of equipment;
    \item bias force/torque sensors leading to issues in compensation and control;
    \item hinder feasible force/torque capabilities due to actuator limits.
\end{itemize}

Because internal loads are intrinsic to these hyperstatic structures, and given their near-ubiquity in robotic applications, it is of great importance to be able to characterize them and, depending on the application, limit or exploit them.

\section{Related Work}
\subsection{\textcolor{black}{Wrench Synthesis}}
\label{sec:RelatedWork-WrenchSynthesis}
The principles of rigid body dynamics have been understood for centuries, and they are a fundamental tool for the analysis of robotic systems. They are of particular importance for the study of load distribution in objects handled by redundantly-actuated robotic systems. Despite having been studied for decades, no consensus has emerged from the robotics literature concerning the characterization and determination of internal loads generated within the object by the wrenches applied by the kinematic chains.

The study of pure forces applied to rigid objects in \cite{Kumar1988} is among the first works on the subject. The authors define \emph{interaction forces} as the difference between the projections of pairwise forces along the lines that join their application points. \textcolor{black}{Therefore, a system is free of interaction forces if, for every pair of applied forces $\mb{f}_i$ and $\mb{f}_j$ at locations $\mb{r}_i$ and $\mb{r}_j$:}

\begin{equation}
    (\mb{f}_j - \mb{f}_i) \cdot (\mb{r}_j - \mb{r}_i) = 0.
    \label{eq:zeroInteractionForceCondition}
\end{equation}
\textcolor{black}{They term the set of forces satisfying this condition the \emph{equilibrating force distribution}.} 



\textcolor{black}{To eliminate the components in the null-space of the grasp matrix, which they showed generated interaction forces, and determine the equilibrating force distribution $\mb{f}_e$, they proposed to use the unweighted Moore-Penrose pseudo-inverse solution}

\begin{equation}
    \mb{f}_e = \mb{G}^{\dagger} \mb{h}_o
    \label{eq:EquilibratingForceDistribution}
\end{equation}
\textcolor{black}{where $\mb{f}_e = \begin{bmatrix}
    \mb{f}^T_{e,1} & \ldots & \mb{f}^T_{e,n}
\end{bmatrix}^T$ is the stacked vector of equilibrating forces and $\mb{G}^{\dagger} = \mb{G}^T (\mb{G} \mb{G}^T )^{-1}$.}

\textcolor{black}{Another important (and so far uncontested) result from }\cite{Kumar1988} \textcolor{black}{is showing that the equilibrating force distribution returned by (}\ref{eq:EquilibratingForceDistribution}\textcolor{black}{) is a \emph{helicoidal vector field}}\endnote{\textcolor{black}{A helicoidal vector field is defined as a vector field whose field lines are a family of helices with the same axis. Each helical field line can be generated by a combination of rotation about a fixed axis and translation along that same axis.}}. \textcolor{black}{While true in some cases, we have found that this statement does not hold generally. We provide clarification on this statement and identify the conditions under which the equilibrating force distribution does form a helicoidal vector field in this paper.}

\cite{Walker1991} later made the distinction between these interaction forces and \emph{internal loads} by showing that the unweighted Moore-Penrose inverse of $\mb{G}$ does return a system of forces which satisfies (\ref{eq:zeroInteractionForceCondition}), but also generates what is more generally referred to as internal loads. They proposed a \emph{nonsqueezing} weighted Moore-Penrose pseudo-inverse of the grasp matrix to be used in place of $\mb{G}^{\dagger}$ in (\ref{eq:EquilibratingForceDistribution}) to find wrench distributions free of internal loads, although this result was later contested in \cite{Chung2005}. Nonetheless, this more general concept of internal loads also allowed torque components to be considered, since it is not clear how to extend the analogy of interaction forces to applied torques (\cite{Erhart2015}). Hence, it should be noted that Kumar and Waldron's method is \emph{only} applicable in situations involving only pure forces applied to the object. In this paper, we use the term \emph{manipulating wrench distribution} to refer to sets of wrenches which do not generate internal loads.

The difference between interaction forces and internal forces was studied in \cite{Zuo1999} and conditions under which these two force sets become equivalent were proposed. 

Physically meaningful interpretations of internal loads were presented in \cite{Yoshikawa1991, Williams1993}. In \cite{Yoshikawa1991}, the pure forces applied to a rigid object by multifingered robotic hands were analyzed and conditions that must be satisfied by the manipulating forces were proposed. In \cite{Williams1993}, the authors develop a physical model of the grasped object called the \emph{virtual linkage}. The authors use this physical model to form an extended set of equilibrium equations that includes the internal force and torque components. 

\cite{Erhart2015} made significant contributions to this line of research by ($i$) demonstrating that the manipulating wrench distribution is not unique for a given resultant wrench in some contexts and ($ii$) proposing a parametrized Moore-Penrose pseudo-inverse of the grasp matrix, denoted $\mb{G}^+_M$, that enables the calculation of the entire set of feasible manipulating wrench distributions. 

\textcolor{black}{The authors use analytical dynamics and employ a virtual rigid body framework where the handled object is represented by a virtual rigid body of mass $m^*_o$ and inertia tensor $\mb{J}^*_o$ assumed to be at rest}\endnote{The asterisks ($^*$) are used to distinguish the dynamic properties of the virtual rigid body from those of the actual physical object.}. \textcolor{black}{Then, for $n$ applied wrenches, they define a dynamically equivalent system of $n$ lumped elements each with mass $m^*_i$ and inertia tensor $\mb{J}^*_i$, and whose centers of mass are coincident with the wrench application points. Together, these lumped elements have the same mass, center of mass, and inertia as the virtual rigid body. By calculating the constrained accelerations induced at each wrench application point by the resultant wrench applied to the rigid body, they are able to determine the manipulating wrenches required to generate the constrained accelerations of each lumped element. The inertia parameters $(m^*_i, \mb{J}^*_i)$ are the parametrization variables used to span the space of solutions for the manipulating wrench distribution. They arrive at the following theorem: }

\begin{theorem}[As proposed in \cite{Erhart2015}]
    \label{th:BadParamMPInv}
    The load distribution given by 
    \begin{equation}
        \mb{G}^+_M = \begin{bmatrix}
            \frac{m^*_1}{m^*_o} \mb{I}_3 & m^*_1 [\mb{J}^*_o]^{-1} \mb{S}(\mb{r}_1)^T \\
            \mb{0}_3 & \mb{J}^*_1 [\mb{J}^*_o]^{-1} \\
            \vdots & \vdots \\
            \frac{m^*_n}{m^*_o} \mb{I}_3 & m^*_n [\mb{J}^*_o]^{-1} \mb{S}(\mb{r}_n)^T \\
            \mb{0}_3 & \mb{J}^*_n [\mb{J}^*_o]^{-1}
        \end{bmatrix}
        \label{eq:BadParametrizedInverse}
    \end{equation}
    for some positive-definite weighting coefficients $m^*_i \in \mathbb{R}$ and $\mb{J}^*_i \in \mathbb{R}^{3 \times 3}$ with 
    \begin{align}
        m^*_o &= \sum_{i=1}^n m^*_i \label{eq:EH-MassSum}\\
        \mb{J}^*_o &= \sum_{i=1}^n \mb{J}^*_i + \sum_{i=1}^n \mb{S}(\mb{r}_i) m^*_i \mb{S}(\mb{r}_i)^T \label{eq:EH-InertiaTensor}\\
        \sum_{i=1}^n & \mb{r}_i m^*_i = \mb{0}_{3 \times 1} \label{eq:EH-CoM}
    \end{align}
    is free of internal loads.
\end{theorem}
\textcolor{black}{Note that $\mb{S}(\mb{r}_i)$ denotes the skew-symmetric matrix that represents the cross product with vector $\mb{r}_i$ such that $\mb{S}(\mb{r}_i) \mb{a} = \mb{r}_i \times \mb{a}$ where $\mb{a}$ can be any 3-vector and that $\mb{S}(\mb{r}_i)^T = -\mb{S}(\mb{r}_i)$ is a property of skew-symmetric matrices.}

\textcolor{black}{Though it is elegant, we have identified many issues with the above theorem and have found it to be largely incomplete. First, we have found that the constraints }(\ref{eq:EH-MassSum})-(\ref{eq:EH-CoM}) \textcolor{black}{are insufficient for guaranteeing valid manipulating wrench distributions and that an additional constraint must be added. Second, we have found that the form of $\mb{G}^+_M$ shown in} (\ref{eq:BadParametrizedInverse}) \textcolor{black}{is incorrect and that the proposed corollary relating $\mb{G}^+_M$ to the unweighted Moore-Penrose pseudo-inverse of $\mb{G}$ is false. Third, }\cite{Erhart2015} \textcolor{black}{incorrectly characterize the uniqueness of the solution and leave open the question of how the choice of parameters $(m^*_i, \mb{J}^*_i)$ relates to the dynamics of the actual handled object, and conversely, how the dynamic properties of the object limit the choice of these parameters. We provide corrections to these issues in this paper.}

\subsection{\textcolor{black}{Wrench Analysis}}
Wrench analysis is a less developed area of the literature. Although some of the wrench synthesis techniques described above have been adapted for use in the decomposition of applied wrenches into sets of manipulating wrenches and constraint wrenches (\cite{Yoshikawa1991, Williams1993}), the latest contributions to this line of research have focused on establishing physical plausibility criteria for the decomposition (\cite{Schmidts2016, Donner2018}). 

\cite{Schmidts2016} \textcolor{black}{were the first to propose these criteria}. \textcolor{black}{They arrived at this conclusion by first stating (without much justification) that each applied force can be written as }

\begin{equation}
    \mb{f}_i = \mb{f}_{m,i} + \mb{f}_{c,i}
    \label{eq:DecompSchmidtsetal}
\end{equation}
\textcolor{black}{where $\mb{f}_i$ is the $i$-th force applied to the object by a robotic manipulator, and $\mb{f}_{m,i}$ and $\mb{f}_{c,i}$ are the manipulating and constraint force components of $\mb{f}_i$. They then state that any component of an applied force cannot be greater than the applied force itself and derive physical plausibility inequalities based on this assumption. They incorporate these criteria into a constrained optimization problem which they use to perform the decomposition. }

Later, \cite{Donner2018} expanded the work done in \cite{Schmidts2016} to include torques and used the same physical plausibility criteria, but showed graphically that the results returned by the method proposed in \cite{Schmidts2016} were erroneous. They therefore proposed a corrected optimization problem that returns results that are consistent with their assumptions about manipulating and constraint wrenches. 

We have found that the decomposition (\ref{eq:DecompSchmidtsetal}) used in \cite{Schmidts2016} and \cite{Donner2018} inaccurately describes the nature of load distribution in objects handled by redundant robotic systems, as it violates some of the fundamental principles of rigid body motion. In this paper, we propose a corrected decomposition algorithm which corrects these oversights, thus making it the first explicit wrench decomposition algorithm that guarantees both physically plausibility and consistency with rigid body motion.

\subsection{Contributions}
While the work of \cite{Erhart2015} and \cite{Donner2018} represents the current state of the art in wrench synthesis and decomposition, this paper makes a significant advancement by establishing a unified framework specifically adapted to the real-time force-control of redundantly-actuated robotic systems. The specific contributions of this work are the following: 

\begin{enumerate}
    \item \textcolor{black}{\textbf{A Unifying Theory of Load Distribution:} We identify and correct critical inconsistencies in the assumptions and theorems that have guided wrench synthesis and analysis research; some of which have stood uncontested for decades. We propose a rigorous theoretical framework that unifies these two fundamental problems in robot force-control as a final solution to this long-standing problem.}
    
    \item \textcolor{black}{\textbf{General, Closed-Form Solutions:} The parametrized Moore-Penrose pseudo-inverse $\mb{G}^+_M$ is obtained directly in closed-form, eliminating the need to invert large matrices. This is a significant advancement over modern wrench decomposition algorithms that rely on iterative optimization routines and allows for the precise and stable determination of manipulating and constraint wrench components without the computational overhead of optimization.}
    
    \item \textcolor{black}{\textbf{A Complete Mathematical Characterization of the Solution:} We propose the first complete characterization of the uniqueness of the manipulating wrench distribution and describe the physical significance of wrench components in the null-space of the grasp matrix by dividing it into two novel sub-spaces.}
\end{enumerate}

The remainder of this paper is structured as follows: Section \ref{sec:ProblemDefinition} defines the problems which we aim to solve; Sections \ref{sec:preliminaries} and \ref{sec:Udwadia-KalabaImplications} present key concepts and considerations relevant to our discussion; Section \ref{sec:GeneralizedTheory} presents our theorem; Section \ref{sec:Implementation} shows how to use our theorem for wrench synthesis and analysis; Section \ref{sec:implications} discusses the implications of our results for other related theorems proposed in the literature; and the advantages of our proposed methods are demonstrated in detailed case studies given in Section \ref{sec:Examples} and an extensive simulation experiment given in Section \ref{sec:Simulation}.

\section{\textcolor{black}{Problem Definition}}
\label{sec:ProblemDefinition}
\subsection{Problem Statement}
The generalized theory presented in this work aims to solve two interconnected problems that are fundamental to the force control of redundantly-actuated robotic systems: 

\begin{enumerate}
    \item \textbf{Wrench synthesis}: Given a desired resultant wrench $\mb{h}_o$ to be applied to the object by the kinematic chains, compute a distribution of wrenches to generate it while eliminating or prescribing internal loads.
    \item \textbf{Wrench analysis (decomposition)}: Given an arbitrary set of $n$ wrenches applied to the object, decompose this set into a set of manipulating wrenches and a set of constraint wrenches
\end{enumerate}

\subsection{Assumptions}
\textcolor{black}{In this work, we assume:}
\begin{enumerate}
    \item \textcolor{black}{that the object handled by the kinematic chains is rigid, and;}
    \item \textcolor{black}{that the vectors $\mb{r}_i = \begin{bmatrix}
        x_i & y_i & z_i
    \end{bmatrix}^T$ that define the application points of the wrenches applied to the object with respect to its center of mass are known.}
\end{enumerate}

\subsection{Applicability}
\textcolor{black}{The resultant wrench $\mb{h}_o$ to be applied to the object by the kinematic chains to complete a given dynamic trajectory is given by }

\begin{equation}
    \mb{h}_o = \mb{M}(\mb{x}) \ddot{\mb{x}}_o + \mb{C}(\mb{x}_o, \dot{\mb{x}}_o) \dot{\mb{x}}_o + \mb{g}(\mb{x}) + \mb{h}_{ext}
    \label{eq:FullDynamics}
\end{equation}
\textcolor{black}{where $\mb{h}_o$ is the resultant wrench applied to the object, $\mb{x}_o$, $\dot{\mb{x}}_o$, and $\ddot{\mb{x}}_o$ are the position, velocity, and acceleration of the rigid body at the instant, respectively, $\mb{M}(\mb{x})$ and $\mb{C}(\mb{x}_o, \dot{\mb{x}}_o)$ are the inertia and centripetal matrices, $\mb{g}(\mb{x})$ is the gravitational force, and $\mb{h}_{ext}$ is the wrench applied to the rigid body by its environment. }

\textcolor{black}{The method for wrench synthesis that we propose is an \emph{instantaneous wrench distribution solver}. It is independent from the individual components that make up $\mb{h}_o$ and simply serves to calculate the required set of wrenches to be applied by the kinematic chains to generate $\mb{h}_o$ while either eliminating or prescribing internal loads at a given instant. It can be used as a solver in broader control algorithms used in modern robotics applications that involve dynamic trajectories and changing contact topology, since even in these applications it is assumed that the wrench application point vectors $\mb{r}_i$ are known at each time step. }

\textcolor{black}{In most modern robotic systems, contact wrenches are often solved for implicitly using whole-body or model-predictive controllers that rely on iterative optimization routines. We do not intend for our proposed results to replace them, but rather to complement them by offering the following advantages:}

\begin{enumerate}
    \item \textcolor{black}{\textbf{Computational efficiency:} Modern controllers rely on iterative optimization routines which can be computationally expensive and suffer from convergence issues, especially in high-frequency control loops. Our proposed methods are explicit, closed-form solutions with predictable computational times. This could prove essential for synthesizing wrench distributions in high-frequency control loops, or, if optimization-based solvers are necessary, providing the solver with a physically informed initial guess }(\say{warm start}) \textcolor{black}{to accelerate convergence.}

    \item \textcolor{black}{\textbf{Analytical decomposition and metrics:} Modern controllers return a wrench vector that satisfies the constraints (friction cones, torque limits, etc.), but rarely distinguish between the wrench components that cause motion (the manipulating wrenches) and those that squeeze the object (the constraint wrenches). The method that we propose for wrench decomposition is mathematically rigorous and also explicit. This is useful for quantifying the squeezing on delicate objects or analyzing solutions returned by black-box controllers by calculating performance metrics with our physically informed analysis methods.}
\end{enumerate}

\section{Preliminaries}
\label{sec:preliminaries}
\subsection{Definitions}
In this work, we consider $n$ wrenches applied to a single rigid body in a robotic system. Our notation is consistent with the one used in \cite{Erhart2015}. The origin of the object-fixed frame $\{o\}$ is coincident with the rigid-body's center of mass (CoM). Denote by $\mb{p}_o$ the position of the rigid body's CoM with respect to the origin of the world frame $\{w\}$ and by $\mb{q}_o$ the Euler angles that define its orientation with respect to $\{w\}$. Let $\mb{p}_i$ denote the position vector of the $i$-th wrench application point in the world frame and let $\mb{r}_i$ denote the position of this same point with respect to the origin of the object frame expressed in frame $\{w\}$. We define a mobile frame $\{ i \}$ for each wrench application point with its origin located at $\mb{r}_i$ and an orientation with respect to the world frame defined by a vector of Euler angles $\mb{q}_i$. The relevant quantities are shown in Figure \ref{fig:kinematicQuantities}.

\begin{figure}
    \centering
    \includegraphics[width=0.75\linewidth]{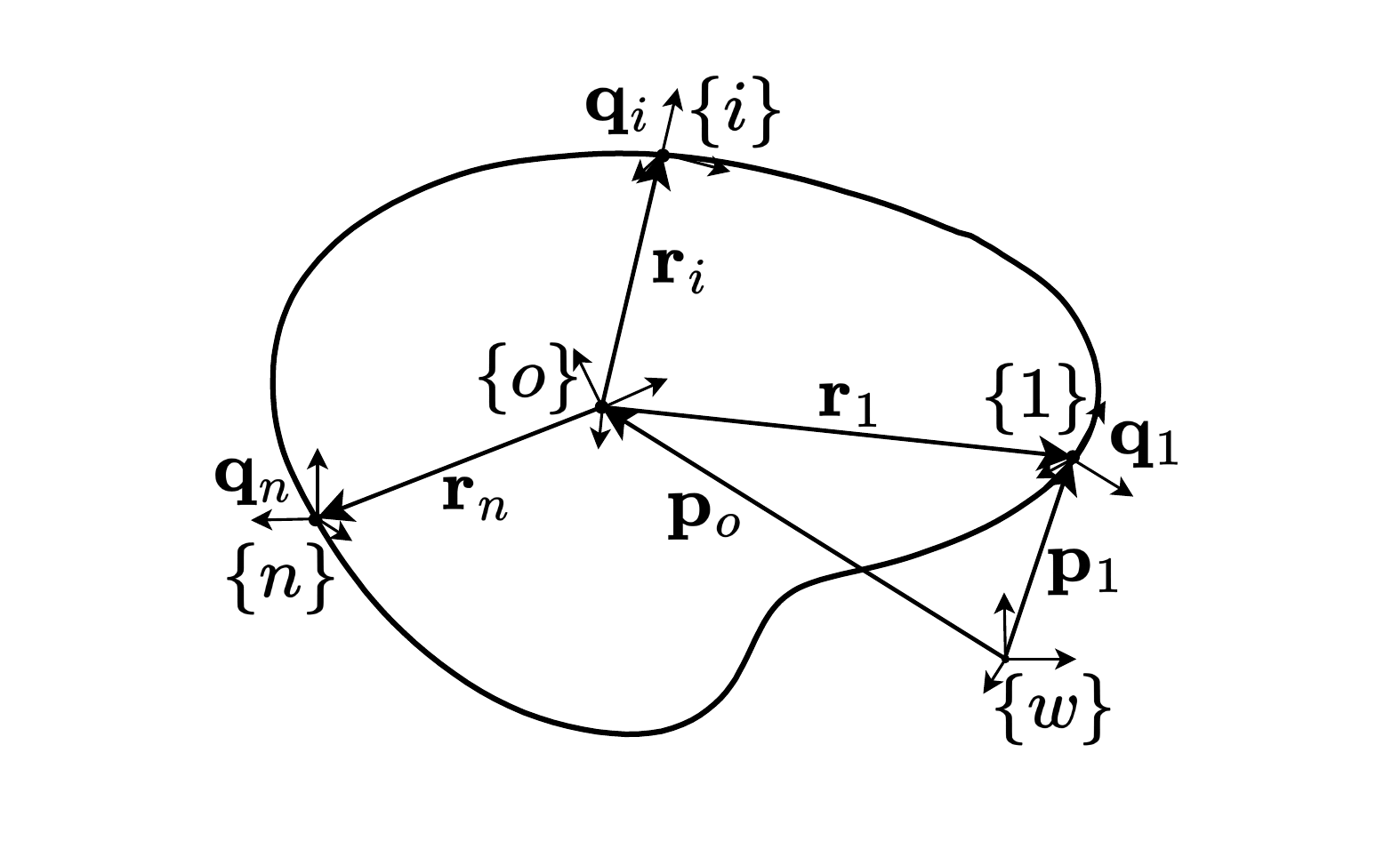}
    \caption{Relevant kinematic quantities (adapted from \cite{Erhart2015}). }
    \label{fig:kinematicQuantities}
\end{figure}

\subsection{\textcolor{black}{Rigid Body Dynamics and the Grasp Matrix}}
A rigid body is defined as a system of elements where the relative position and orientation between any two elements remains constant, regardless of the loads applied. For a rigid body with mass $m_o$ and inertia tensor $\mb{J}_o$ with respect to its CoM, the Newton-Euler equations of motion are written

\begin{align}
    \mb{f}_o &= m_o \ddot{\mb{p}}_o, \label{eq:NEforce}\\
    \mb{t}_o &= \mb{J}_o \dot{\boldsymbol{\omega}}_o + \boldsymbol{\omega}_o \times (\mb{J}_o \boldsymbol{\omega}_o), \label{eq:NEtorque}
\end{align}

The resultant wrench is the sum of the $n$ wrenches applied by the kinematic chains. If the $i$-th kinematic chain applies a force $\mb{f}_i$ and torque $\mb{t}_i$ at location $\mb{r}_i$ with respect to the CoM, the resultant force and torque are

\begin{equation}
    \mb{f}_o = \sum_{i=1}^n \mb{f}_i,
    \quad \mb{t}_o = \sum_{i=1}^n \left(\mb{t}_i +\mb{r}_i \times \mb{f}_i \right),
\end{equation}
respectively. This allows us to define the grasp matrix $\mb{G}$ shown in (\ref{eq:staticEq}). The form of $\mb{G}$ depends on the contact model. 

\paragraph*{Case 1 - Pure Forces (Point Contact)} 
\textcolor{black}{The kinematic chains only apply pure forces. Here, $\mb{h} = \begin{bmatrix}
    \mb{f}_1^T & \ldots & \mb{f}_n^T
\end{bmatrix}^T \in \mathbb{R}^{3n}$ and $\mb{G}$ takes the form} 

\begin{equation}
    \mb{G} = \begin{bmatrix}
        \mb{I}_3 & \cdots & \mb{I}_3 \\
        \mb{S}(\mb{r}_1) & \cdots & \mb{S}(\mb{r}_n)
    \end{bmatrix}
    \label{eq:GPureForces}
\end{equation}

\paragraph*{Case 2 - General Wrenches (Rigid Contact)} 
\textcolor{black}{The kinematic chains apply both forces and torques. Here, $\mb{h} = \begin{bmatrix}
    \mb{f}_1^T & \mb{t}_1^T & \ldots & \mb{f}_n^T & \mb{t}_n^T
\end{bmatrix}^T \in \mathbb{R}^{6n}$ and $\mb{G}$ becomes}

\begin{equation}
    \mb{G} = \begin{bmatrix}
        \mb{I}_3 & \mb{0}_3 & \cdots & \mb{I}_3 & \mb{0}_3 \\ 
        \mb{S}(\mb{r}_1) & \mb{I}_3 & \cdots & \mb{S}(\mb{r}_n) & \mb{I}_3
    \end{bmatrix}.
    \label{eq:GForcesandMoments}
\end{equation}

\subsection{Kinematic Constraints}
We define a mobile frame $\{i\}$ at each wrench application point which is fixed to the rigid body. Changes to the position and orientation of these mobile frames are therefore constrained to obey the rigid body motion of the object and we may define two constraint equations.

\paragraph*{1) Translational Constraint}
The relative position of any two mobile frames must remain constant. This can be written as 

\begin{equation}
    ^o \mb{r}_i =\text{constant}, \ \forall i
\end{equation}

\noindent
where $^o \mb{r}_i$ defines the location of the origin of mobile frame $i$ in the object-fixed frame $\{o\}$. The location of the same point expressed in the world frame $\{w\}$ is obtained with $\mb{p}_i = \mb{p}_o + \mb{Q} ^o\mb{r}_i$ with $\mb{Q}$ being the rotation matrix which represents the rotation from the world frame to the object-fixed frame. We differentiate to obtain the velocity equation $\dot{\mb{p}}_i = \dot{\mb{p}}_o + \boldsymbol{\omega}_o \times \mb{r}_i$ where $\boldsymbol{\omega}_o$ is the angular velocity of the rigid body. Differentiating again, we obtain the acceleration equation 

\begin{equation}
    \ddot{\mb{p}}_i = \ddot{\mb{p}}_o + \dot{\boldsymbol{\omega}}_o \times \mb{r}_i + \boldsymbol{\omega}_o \times (\boldsymbol{\omega}_o \times \mb{r}_i).
    \label{eq:LinearKinematicConstraint}
\end{equation}

\paragraph*{2) Rotational Constraint}
In addition to the relative position of points on the rigid body remaining constant, the relative orientation of any frame $\{ i \}$ with respect to frame $\{ o \}$ must also remain constant. This translates to every mobile frame having the same angular velocity as the rigid body (i.e. $\boldsymbol{\omega}_i = \boldsymbol{\omega}_o \ \forall i$). Differentiation of this constraint yields

\begin{equation}
    \dot{\boldsymbol{\omega}}_i = \dot{\boldsymbol{\omega}}_o.
    \label{eq:AngularKinematicConstraint}
\end{equation}

Eqs. (\ref{eq:LinearKinematicConstraint}) and (\ref{eq:AngularKinematicConstraint}) are the constrained acceleration equations that must be satisfied by every point on the rigid body (\cite{Erhart2015}).

For a system of $n$ wrenches, the kinematic constraints imposed on the $n$ mobile frames fixed to the rigid body at the wrench application points can be written in matrix form as

\begin{equation}
    \mb{A} \ddot{\mb{x}} = \mb{b},
    \label{eq:KinConstraints}
\end{equation}
where $\ddot{\mb{x}} = \begin{bmatrix}
    \ddot{\mb{x}}_1^T & \ldots & \ddot{\mb{x}}_n^T
\end{bmatrix}^T$ is the stacked vector of constrained accelerations of the frames where $\ddot{\mb{x}}_i$ is defined similarly to $\ddot{\mb{x}}_o$, and

\begin{equation}
    \mb{A} = \begin{bmatrix}
        -\mb{I}_3 & \mb{S}(\mb{r}_{21}) & \mb{I}_3 & \mb{0} &  &  & \\
        \mb{0} & -\mb{I}_3 & \mb{0} & \mb{I}_3 &  &  & \\
        \vdots & \vdots &  &  & \ddots & & \\
        -\mb{I}_3 & \mb{S}(\mb{r}_{n1}) &  &  &  & \mb{I}_3 & \mb{0} \\
        \mb{0} & -\mb{I}_3 &  &  &  & \mb{0} & \mb{I}_3
    \end{bmatrix}
\end{equation}
where $\mb{r}_{ji}$ is a relative position vector from $\mb{r}_i$ to $\mb{r}_j$ (\cite{Erhart2015}). The centripetal terms are grouped in vector $\mb{b}$ which is written as

\begin{equation}
    \mb{b} = \begin{bmatrix}
        \mb{S}(\boldsymbol{\omega}_2)^2 \mb{r}_{21} \\ \mb{0} \\ \vdots \\ \mb{S}(\boldsymbol{\omega}_n)^2 \mb{r}_{n1} \\ \mb{0}
    \end{bmatrix}.
\end{equation}

\subsection{Dynamically Equivalent Systems}
\textcolor{black}{Dynamically equivalent systems are often used in analytical dynamics to simplify the dynamic analyses of rigid bodies with continuous mass distributions by discretizing them into sets of lumped elements. A valid dynamically equivalent system has the same mass, center of mass, and inertia tensor (about all axes) as the rigid object. Thus, they behave identically in terms of force and motion. In this work, we construct a dynamically equivalent system of the grasped object by defining a set of $n$ rigidly connected discrete elements with mass $m_i$ and inertia tensor $\mb{J}_i$ with their centers of mass at the wrench application points (see Figure }\ref{fig:equimomentalSystem}). \textcolor{black}{The CoM of the $i$-th lumped element is therefore coincident with the origin of the $i$-th mobile frame and is subject to the same kinematic constraints.} 

\begin{figure}
    \centering
    \includegraphics[width=0.85\linewidth]{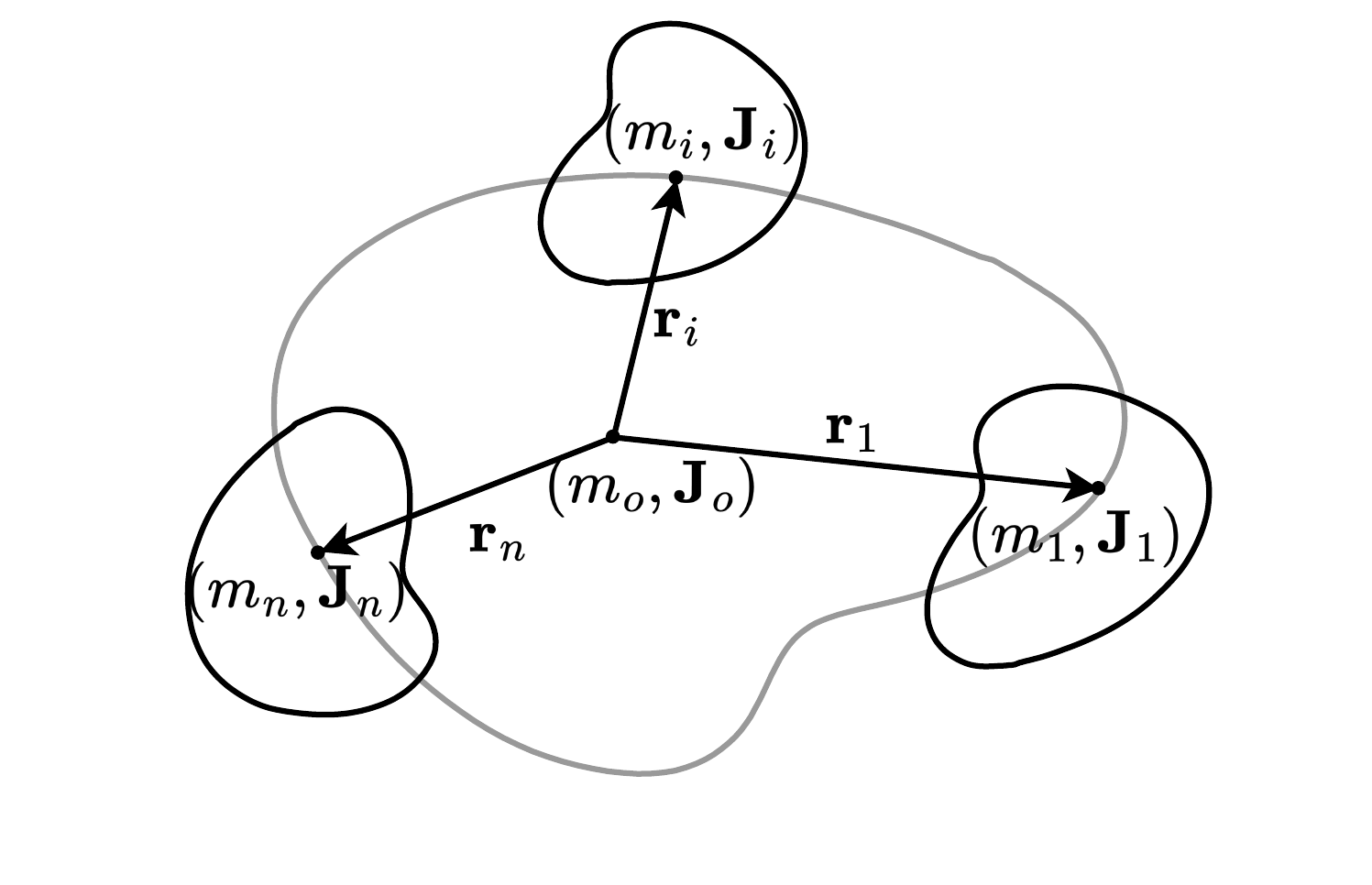}
    \caption{Example of a dynamically equivalent system of $n$ LMIEs.}
    \label{fig:equimomentalSystem}
\end{figure}

We call these discrete elements Lumped Mass-Inertia Elements (LMIEs) and define them as rigid bodies of arbitrary shape with known mass $m_i$ and inertia tensor $\mb{J}_i$. A dynamically equivalent system of $n$ LMIEs each with properties $(m_i, \mb{J}_i)$ for a rigid object with properties $(m_o, \mb{J}_o)$ must satisfy

\begin{align}
    m_o &= \sum_{i=1}^n m_i \  \text{(mass equivalence)}, \label{eq:MassSumEquivalence}\\
    \mb{J}_o &= \sum_{i=1}^n \mb{J}_i + \sum_{i=1}^n \mb{S}(\mb{r}_i) m_i \mb{S}(\mb{r}_i)^T \ \text{(inertia equiv.)}, \label{eq:MomentofInertiaEquivalence}\\
    \sum_{i=1}^n & \mb{r}_i m_i = \mb{0}_{3 \times 1} \ \text{(CoM equivalence)}. \label{eq:CoMEquivalence}
\end{align}

\subsection{\textcolor{black}{The Udwadia-Kalaba Equation}}
\label{subsec:Udwadia-KalabaEquation}
\textcolor{black}{The Udwadia-Kalaba equation }(\cite{Udwadia2010}) \textcolor{black}{is derived from Gauss's principle of least constraint and provides an alternative perspective on constrained rigid body motion. As it is the foundation for what will follow, we summarize the main results here and refer the reader to }\cite{Udwadia1992, Udwadia2010} \textcolor{black}{for further details.}

\textcolor{black}{Consider an \emph{unconstrained} system of $n$ LMIEs. Their equations of motion can be written as}

\begin{equation}
    \mb{M}(\mb{x}, t) \ddot{\mb{x}}^d = \hat{\mb{h}}(\mb{x}, \dot{\mb{x}}^d, t),
    \label{eq:UnconstrainedLMIEDynamics}
\end{equation}
\textcolor{black}{where $\mb{M}(\mb{x}, t) \in \mathbb{R}^{6n \times 6n}$ is the symmetric positive definite inertia matrix of the system containing the inertial parameters of the LMIEs (i.e., $\mb{M}(\mb{x}, t) = \text{diag}(
m_1 \mb{I}_3, \mb{J}_1, \ldots, m_n \mb{I}_3, \mb{J}_n)$). $\hat{\mb{h}}(\mb{x}, \dot{\mb{x}}^d, t)$ is the known stacked vector of wrenches acting on the LMIEs that includes the applied wrenches $\mb{h}$ and any other centripetal forces that may be present. $\dot{\mb{x}}^d, \ddot{\mb{x}}^d \in \mathbb{R}^6$ are the stacked  vectors of unconstrained velocities and accelerations of the LMIEs, respectively where superscript $d$ stands for \emph{desired}. Because the LMIEs are considered unconstrained, $\dot{\mb{x}}^d$ can take any arbitrary value, and $\ddot{\mb{x}}^d$ is the acceleration the LMIEs \emph{would have had}, had they not been fixed to the rigid body in motion }(\cite{Udwadia2002}).

\textcolor{black}{In the presence of kinematic constraints }(\ref{eq:KinConstraints}), \textcolor{black}{the actual \emph{constrained} accelerations of the LMIEs will differ from the free unconstrained accelerations they would have had. This incompatibility between the desired motion and constrained motion causes additional forces of constraint to be applied to the LMIEs }(\cite{Udwadia1992}). \textcolor{black}{The dynamic equation that describes the constrained motion of the LMIEs now takes the form }

\begin{equation}
    \mb{M}(\mb{x}, t) \ddot{\mb{x}} = \hat{\mb{h}}(\mb{x}, \dot{\mb{x}}, t) + \hat{\mb{h}}_c(\mb{x}, \dot{\mb{x}}, t)
    \label{eq:constrainedEquationofMotion}
\end{equation}
\textcolor{black}{where $\hat{\mb{h}}_c(\mb{x}, \dot{\mb{x}}, t) = \begin{bmatrix}
    \hat{\mb{h}}_{c,1}^T & \ldots & \hat{\mb{h}}_{c,n}^T
\end{bmatrix}^T$ is the generalized stacked vector of constraint wrenches acting on the LMIEs. It is these constraint wrenches which attempt to deform the rigid body and cause the generation of internal loads. They are added to the applied wrenches such that the resulting motion of each LMIE is consistent with the motion of the rigid body induced by the resultant wrench $\mb{h}_o$. This is illustrated in Figure }\ref{fig:ConstrainedAccelerations}. 

\begin{figure}[t!]
    \subfloat[Three wrenches $\mb{h}_i$ applied to a rigid body generate a resultant wrench $\mb{h}_o$.]{%
        \includegraphics[width=.4\linewidth]{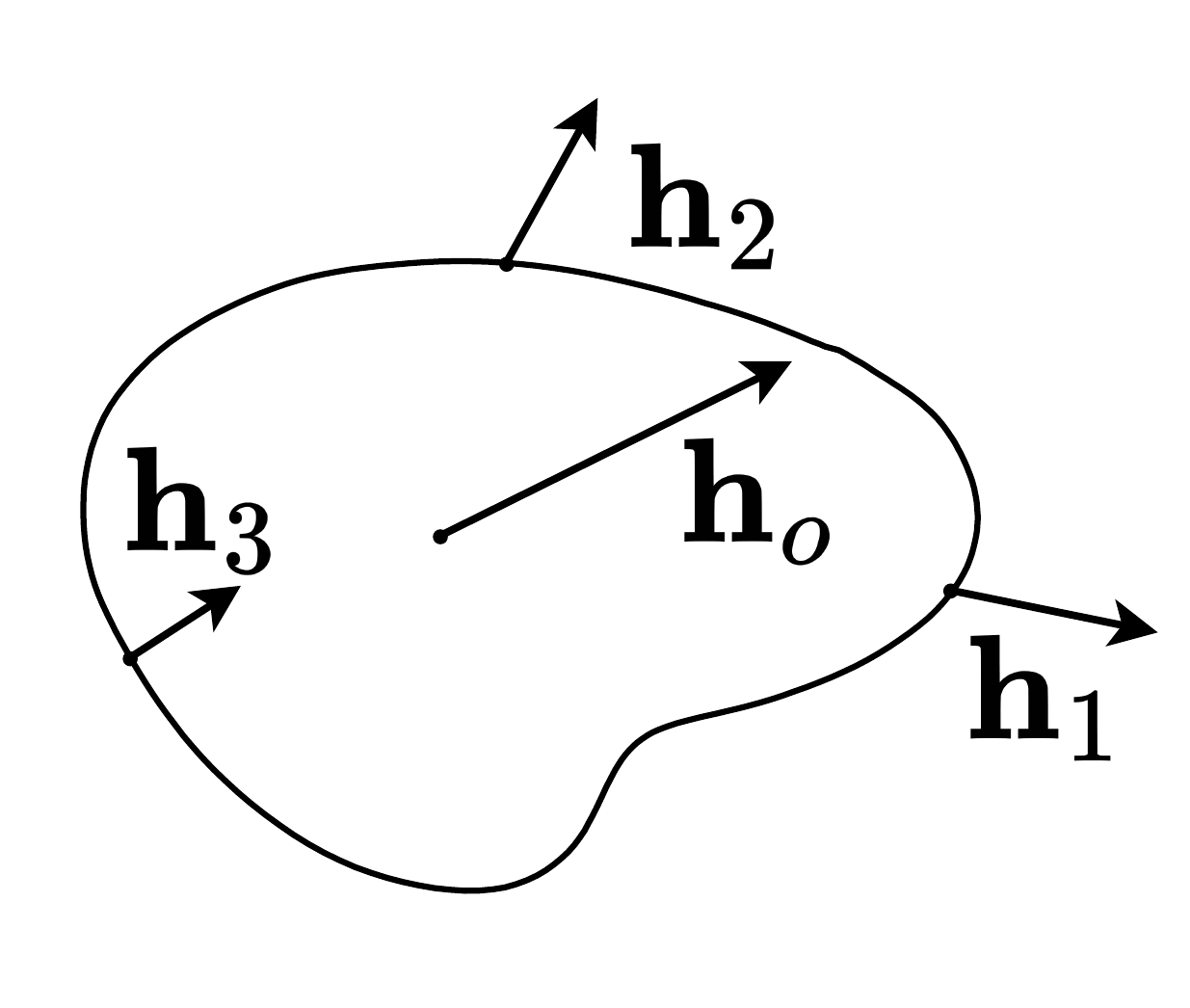}%
        \label{subfig:ConstrainedAccelerations1}%
    } \qquad
    \subfloat[The resultant wrench $\mb{h}_o$ induces rigid body acceleration $\ddot{\mb{x}}_o$.]{%
        \includegraphics[width=.4\linewidth]{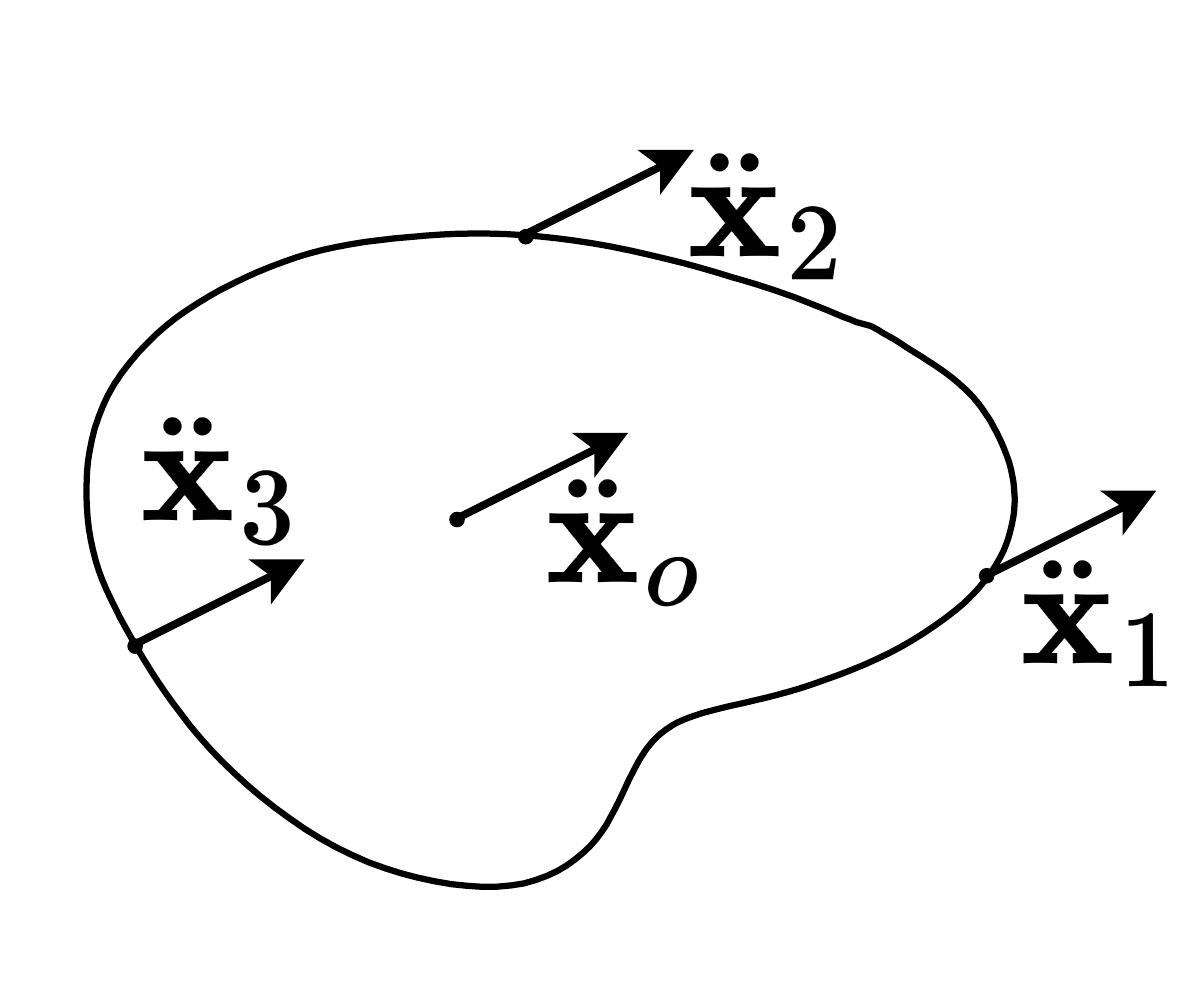}%
        \label{subfig:ConstrainedAccelerations2}%
    }\\
    \centering
    \subfloat[Constraint wrenches $\mb{h}_{c,i}$ ensure compliance with rigid body motion.]{%
        \includegraphics[width=.6\linewidth]{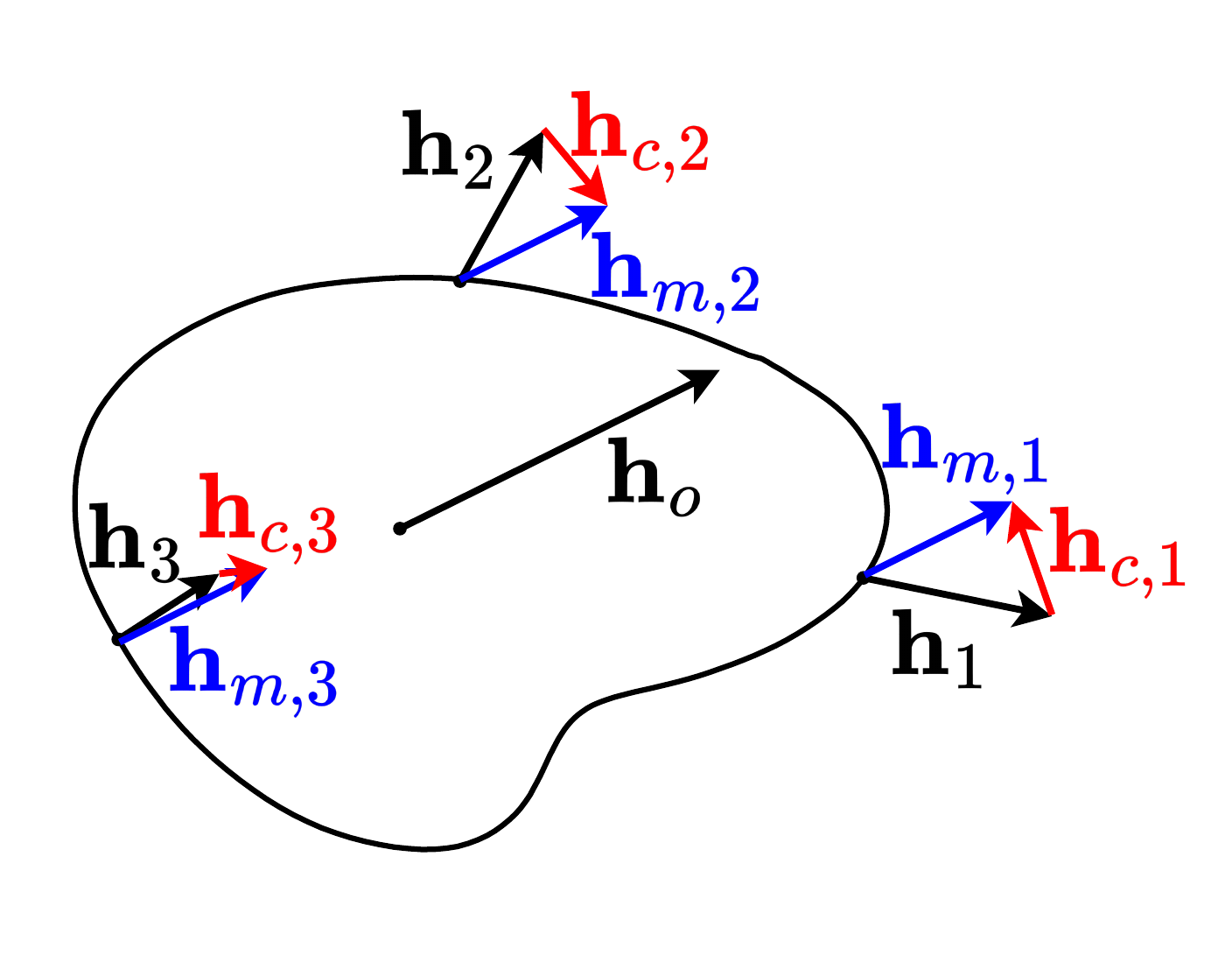}%
        \label{subfig:ConstrainedAccelerations3}%
    }
    \caption{Illustration of physical significance of constraint wrenches.}
    \label{fig:ConstrainedAccelerations}
\end{figure}

\textcolor{black}{The term on the left-hand side of} (\ref{eq:constrainedEquationofMotion}) \textcolor{black}{is the generalized stacked manipulating wrench vector $\hat{\mb{h}}_m = \begin{bmatrix}
    \hat{\mb{h}}_{m,1}^T & \ldots & \hat{\mb{h}}_{m,n}^T
\end{bmatrix}^T$. It is the set of wrenches required to impart the proper constrained accelerations on the LMIEs. We may therefore rewrite }(\ref{eq:constrainedEquationofMotion}) as 

\begin{equation}
     \hat{\mb{h}}_m = \hat{\mb{h}} + \hat{\mb{h}}_c
     \label{eq:GeneralizedWrenchEquilibrium}
\end{equation}
\textcolor{black}{where the arguments $(\mb{x}, \dot{\mb{x}}, t)$ have been omitted for clarity. Eq. }(\ref{eq:GeneralizedWrenchEquilibrium}) \textcolor{black}{is the fundamental equation which describes how loads are distributed in a rigid body in motion subject to multiple applied wrenches. }

The generalized constraint wrenches can be obtained explicitly with 

\begin{equation}
    \hat{\mb{h}}_c = \mb{M}^{\frac{1}{2}} \mb{B}^+ (\mb{b} - \mb{A} \ddot{\mb{x}}^d)
    \label{eq:explicitConstraintWrenches}
\end{equation}
where $\mb{B} = \mb{A} \mb{M^{-\frac{1}{2}}}$ and $\mb{e} = \mb{b} - \mb{A} \ddot{\mb{x}}^d$ is the kinematic error (i.e., the extent to which the unconstrained accelerations $\ddot{\mb{x}}^d_i$ of the LMIEs induced by the applied wrenches $\hat{\mb{h}}_i$ do not satisfy the kinematic constraints (\ref{eq:KinConstraints})). The reader is referred to \cite{Udwadia2010} for further details on this derivation.

\section{\textcolor{black}{Implications of the Udwadia-Kalaba Equation}}
\label{sec:Udwadia-KalabaImplications}
\subsection{\textcolor{black}{Internal Loads Generated by Centripetal Forces}}
\label{sec:CentripetalForceImplication}
\textcolor{black}{Some researchers have made the assumption that constraint wrenches are components of the applied wrenches }(\cite{Schmidts2016, Donner2018}) (see Eq. (\ref{eq:DecompSchmidtsetal})). \textcolor{black}{However, the Udwadia-Kalaba equation states that constraint wrenches are applied to an LMIE by the rigid body so as to make its constrained motion consistent with the kinematic constraints. In fact, constraint wrenches can be present when no wrenches are applied at all.}

\textcolor{black}{Imagine a planar, triangular rigid body spinning with a constant angular velocity $\omega_o$ about an axis that passes through its CoM and is normal to the plane. Each vertex of the triangle is at a distance $r$ from the CoM. A dynamically equivalent system of this rigid body can be constructed with an LMIE of mass $m_i$ located at each vertex. } 

\textcolor{black}{The velocity vectors $\dot{\mb{p}}_i$ of the vertices are perpendicular to the axis of rotation. Were they not rigidly attached to the rigid body (i.e., unconstrained), they would travel in a straight line in that direction. However, the normal acceleration $a_n = r \omega_o^2$ keep them spinning about the axis with the same angular velocity as every other point on the planar body such that the constraints imposed by rigid body motion are satisfied. This normal acceleration $a_{n,i}$ is therefore the required constrained acceleration of the $i$-th vertex, and if no external forces are applied, a constraint force is necessary to induce this constrained acceleration on the LMIE. This is illustrated in Figure }\ref{fig:SpinningTriangle}.

\begin{figure}
    \centering
    \includegraphics[width=0.6\linewidth]{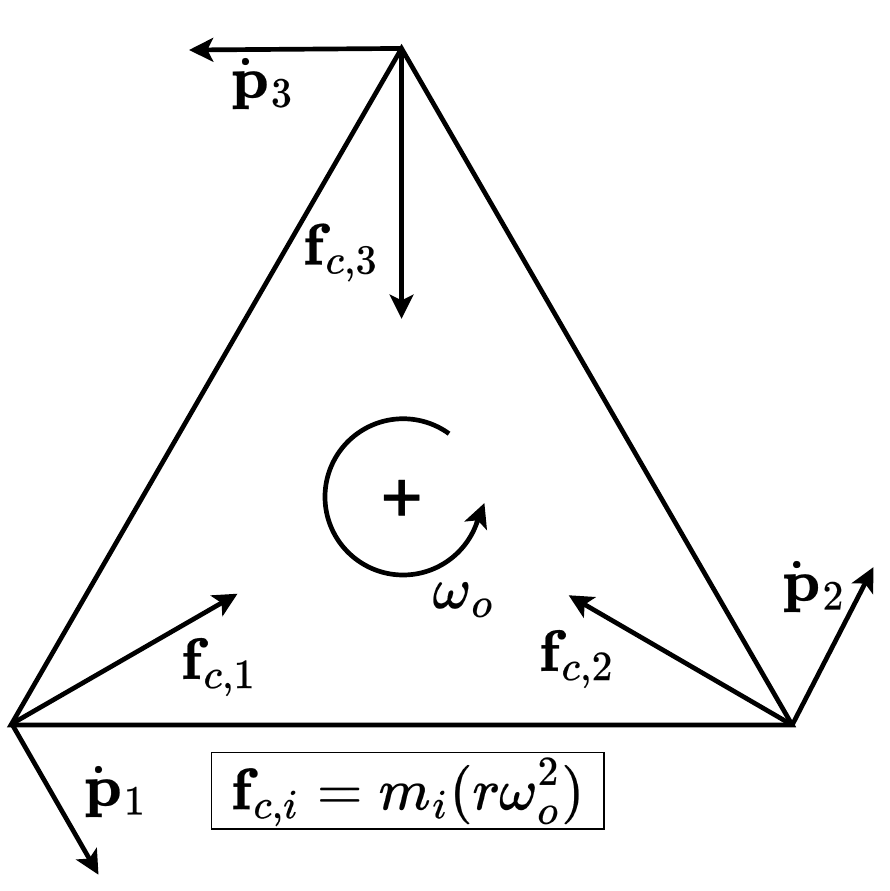}
    \caption{When a body is in motion, constraint wrenches are present even in the absence of applied wrenches.}
    \label{fig:SpinningTriangle}
\end{figure}

\textcolor{black}{In this work, we are primarily concerned with the internal loads generated by the interplay between the applied wrenches $\mb{h}_i$, not those due to the centripetal forces which can be present even when no external wrenches are applied. For this reason, we adopt Erhart and Hirche's virtual rigid body model where they substitute the actual rigid body of mass $m_o$, inertia tensor $\mb{J}_o$ and instantaneous velocity $\dot{\mb{x}}_o$ with a virtual rigid body of mass $m^*_o$ and inertia tensor $\mb{J}^*_o$ which we assume to be \emph{at rest}. Note that the inertial parameters $(m^*_o, \mb{J}^*_o)$ of the virtual rigid body cannot be chosen arbitrarily, as we will show in the next sub-sections. In what follows, any variable identified with an asterisk $(^*)$ corresponds to a dynamic quantity of the virtual rigid body or of the LMIEs in its dynamically equivalent system. }

\textcolor{black}{The Udwadia-Kalaba equation in its general form discussed in Section} \ref{subsec:Udwadia-KalabaEquation} \textcolor{black}{describes the constrained motion of the LMIEs using generalized wrench vectors $\hat{\mb{h}}$, $\hat{\mb{h}}_c$, and $\hat{\mb{h}}_m$ which also include centripetal forces. Since these are now zero in the rigid body at rest, we rewrite this equation as}

\begin{equation}
    \mb{h}_m = \mb{h} + \mb{h}_c 
    \label{eq:wrenchEquilibrium}
\end{equation}
\textcolor{black}{where $\mb{h}$ contains the wrenches applied to the LMIEs by the kinematic chains, $\mb{h}_c$ contains the constraint wrenches applied to the LMIEs by the rigid body to ensure the compatibility of their motion with the kinematic constraints, and $\mb{h}_m$ contains the manipulating wrenches required to induce the proper constrained accelerations on the LMIEs. None of them include additional contributions from inertial forces such as the $\boldsymbol{\omega}_o \times (\mb{J}_o \boldsymbol{\omega}_o)$ term in }(\ref{eq:NEtorque}).

\textcolor{black}{An interesting consequence of considering the virtual rigid body to be at rest is that the constrained accelerations of the LMIEs in its dynamically equivalent system are now given by}

\begin{equation}
    \ddot{\mb{x}}^*_i = \begin{bmatrix}
        \ddot{\mb{p}}^*_i \\ \dot{\boldsymbol{\omega}}^*_i
    \end{bmatrix} = \begin{bmatrix}
        \ddot{\mb{p}}^*_o + \dot{\boldsymbol{\omega}}^*_o \times \mb{r}_i \\ \dot{\boldsymbol{\omega}}^*_o
    \end{bmatrix},
\end{equation}
\textcolor{black}{since $\boldsymbol{\omega}^*_o = \mb{0}$. This is analogous to the twist vector of a rigid body in motion which defines the helicoidal field of velocities $\dot{\mb{p}}_i$ of the points on the rigid body. As a result, the vector field of constrained accelerations $\ddot{\mb{p}}_i$ of the LMIEs on the virtual rigid body also forms a helicoidal vector field.}

\subsection{\textcolor{black}{Scale Invariance of the Dynamically Equivalent System}}
\label{sec:ScaleInvarianceImplication}
\textcolor{black}{The virtual rigid body and its dynamically equivalent system are simply convenient mathematical tools used to distribute the resultant wrench $\mb{h}_o$ among the applied wrenches $\mb{h}$ while satisfying certain constraints to obtain the manipulating wrench distribution. They serve to discretize the handled object with a continuous mass distribution into a set of discrete elements (i.e., the LMIEs) and to ensure that we are only considering the internal loads generated by the applied wrenches. The instantaneous acceleration of the virtual rigid body $\ddot{\mb{x}}^*_o$ induced by $\mb{h}_o$ may differ from the instantaneous acceleration $\ddot{\mb{x}}_o$ of the actual object if $\mb{h}_o$ were applied to it. This does not negate the validity of our proposed framework, since calculating the virtual instantaneous accelerations induced by the wrenches is only an intermediary step which serves to translate the resultant wrench to the wrenches required to generate it at the wrench application points. In fact, the inertial properties of the virtual rigid body $(m^*_o, \mb{J}^*_o)$ need not match those of the handled object $(m_o, \mb{J}_o)$ exactly, they need only be \emph{proportional}.} 

\textcolor{black}{Suppose a wrench $\mb{h}_o$ is applied to the CoM a rigid body of mass $m_o$ and inertia tensor $\mb{J}_o$. The resulting acceleration is given by}

\begin{equation}
    \ddot{\mb{x}}_o = \mb{M}_o^{-1} \mb{h}_o
\end{equation}
\textcolor{black}{where the centripetal terms have been omitted so as not to consider the internal loads generated by centripetal forces as described in the previous sub-section. }

\textcolor{black}{If we construct a dynamically equivalent system of $n$ LMIEs for this rigid body, the constrained acceleration $\ddot{\mb{x}}_i$ of each LMIE can be calculated with the kinematic constraint equations} (\ref{eq:LinearKinematicConstraint}) and (\ref{eq:AngularKinematicConstraint}). \textcolor{black}{The manipulating wrenches needed to induce these constrained accelerations can in turn be found with }

\begin{equation}
    \mb{h}_{m,i} = \mb{M}_i\ddot{\mb{x}}_i, \quad \mb{M}_i = \begin{bmatrix}
        m_i \mb{I}_3 & \mb{0} \\ \mb{0} & \mb{J}_i
    \end{bmatrix}.
\end{equation}

\textcolor{black}{Now, let us do the same using a dynamically equivalent system of the virtual rigid body with mass $m^*_o = k m_o$ and inertia tensor $\mb{J}^*_o = k \mb{J}_o$ where $k$ is a proportionality constant. We can therefore write}

\begin{equation}
    \mb{M}^*_o = k \mb{M}_o.
\end{equation}
\textcolor{black}{In order for constraints }(\ref{eq:MassSumEquivalence})-(\ref{eq:CoMEquivalence}) \textcolor{black}{to still be satisfied, we must also scale the inertial properties of the LMIEs by the same amount: }

\begin{equation}
    \mb{M}^*_i = k \mb{M}_i \quad \forall i.
    \label{eq:ScaledInertia}
\end{equation}

\textcolor{black}{The acceleration induced on the dynamically equivalent system is then }

\begin{equation}
    \ddot{\mb{x}}^*_o = (\mb{M}^*_o)^{-1} \mb{h}_o = \frac{1}{k} \ddot{\mb{x}}_o
\end{equation}
\textcolor{black}{and, consequently, the constrained acceleration of every LMIE in this case is}

\begin{equation}
    \ddot{\mb{x}}^*_i = \frac{1}{k} \ddot{\mb{x}}_i.
    \label{eq:ScaledAcceleration}
\end{equation}

\textcolor{black}{Finally, if we denote by $\mb{h}'_{m,i}$ the manipulating wrench needed to induce the scaled constrained acceleration $\ddot{\mb{x}}^*_i$ on the $i$-th LMIE of mass $m^*_i$ and inertia tensor $\mb{J}^*_i$, we can obtain it with}

\begin{equation}
    \mb{h}'_{m,i} = \mb{M}^*_i \ddot{\mb{x}}^*_i.
    \label{eq:ScaledWrench}
\end{equation}

\textcolor{black}{Substituting }(\ref{eq:ScaledInertia}) and (\ref{eq:ScaledAcceleration}) into (\ref{eq:ScaledWrench}), \textcolor{black}{we see that}

\begin{equation}
    \mb{h}'_{m,i} = k \mb{M}_i \left(\frac{1}{k} \ddot{\mb{x}}_i\right) = \mb{h}_{m,i},
\end{equation}
\textcolor{black}{which is the same solution.}

\textcolor{black}{In other words, if the virtual rigid body has, for example, double the inertia of the physical object, its instantaneous acceleration will be half that of the physical object and so will the instantaneous accelerations of the LMIEs. However, the inertia of the LMIEs will be double those of the physical object, and the wrenches required to induce half the acceleration on a rigid body with double the inertia will be the same since the effects of the scaling cancel. }

\subsection{The Need for an Additional Constraint}
\label{sec:AdditionalConstraintJustification}
In their original theorem (Theorem \ref{th:BadParamMPInv}) Erhart and Hirche use the same constraints as those required for constructing a dynamically equivalent system (see Eqs. (\ref{eq:MassSumEquivalence})-(\ref{eq:CoMEquivalence})) to limit the feasible solutions for the virtual inertia parameters $(m^*_i, \mb{J}^*_i)$. However, we have found that these are not sufficient for calculating valid manipulating wrench distributions and that an additional constraint is required. 

The additional constraint pertains to the alignment of the manipulating torques. The resultant torque $\mb{t}_o$ is obtained with

\begin{equation}
    \mb{t}_o = \sum_{i=1}^n \mb{t}_{m,i} + \sum_{i=1}^n (\mb{r}_i \times \mb{f}_{m,i}),
    \label{eq:resultantTorque}
\end{equation}
where the second term on the right-hand side is the manipulating force-induced torque $\mb{t}_{fm}$. The manipulating torques $\mb{t}_{m,i}$ and the manipulating force-induced torque $\mb{t}_{fm}$ of an admissible manipulating wrench distribution will sum to $\mb{t}_o$ yet not have any antagonistic components that, when summed, cancel and have no effect on the net wrench applied to the object. This is illustrated in Figure \ref{fig:TorqueCondition} where a pure torque $\mb{t}$ and a force-induced torque $\mb{t}_f$ sum to a resultant torque $\mb{t}_o$. In Figure \ref{subfig:CompensatedTorques}, $\mb{t}_f$ has a component that is orthogonal to $\mb{t}_o$. This $x$-component must be compensated by $\mb{t}$ in order for their resultant to equal $\mb{t}_o$. This compensation induces constraint wrenches. Conversely, in Figure \ref{subfig:ManipulatingTorques}, both $\mb{t}_f$ and $\mb{t}$ are parallel and sum to $\mb{t}_o$ without need for compensation and therefore do not generate constraint wrenches.

\begin{figure}[t!]
    \subfloat[$\mb{t}_f$ and $\mb{t}$ have components perpendicular to $\mb{t}_o$ that cancel and cause internal loads.]{%
        \includegraphics[width=.45\linewidth]{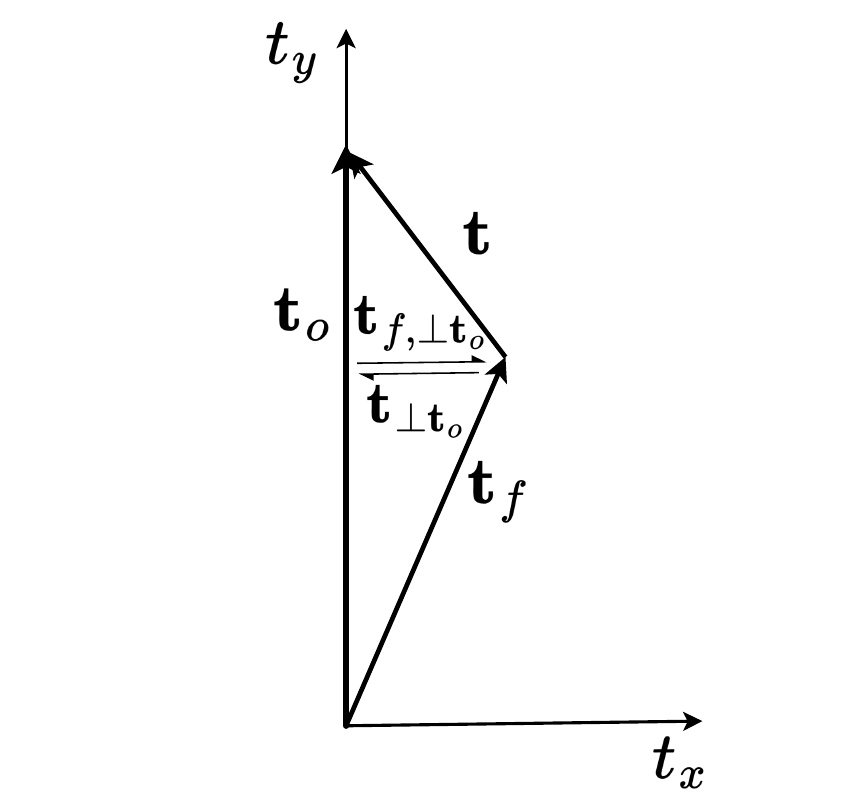}%
        \label{subfig:CompensatedTorques}%
    } \qquad
    \subfloat[$\mb{t}_f$ and $\mb{t}$ have no antagonistic components and sum to $\mb{t}_o$ without generating internal loads.]{%
        \includegraphics[width=.42\linewidth]{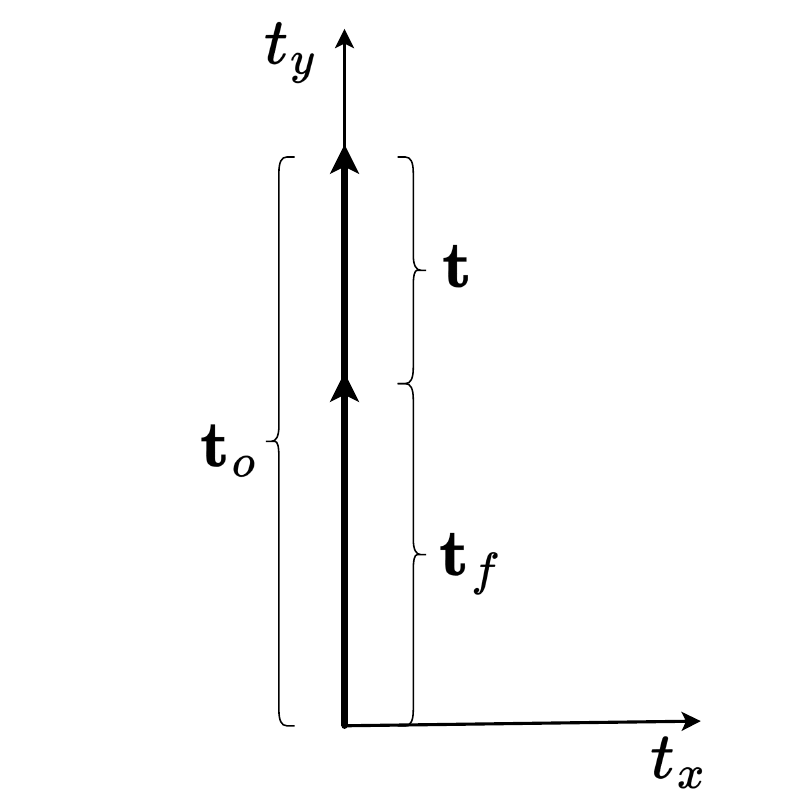}%
        \label{subfig:ManipulatingTorques}%
    }\\
    \centering
    \caption{Example of orthogonal torque component cancellation in non-parallel torques causing internal loads.}
    \label{fig:TorqueCondition}
\end{figure}

With this realization, we may rewrite (\ref{eq:resultantTorque}) as 

\begin{equation}
    ||\mb{t}_o || \mb{u}_o = \sum_{i=1}^n ||\mb{t}_{m,i}|| \mb{u}_o + ||\mb{t}_{fm}|| \mb{u}_o
\end{equation}
where $\mb{u}_o$ is a unit vector in the direction of $\mb{t}_o$. 

The angular acceleration of the virtual rigid body is found with 

\begin{equation}
    \dot{\boldsymbol{\omega}}^*_o = (\mb{J}^*_o)^{-1} \mb{t}_o.
    \label{eq:angularDynamicsWhole}
\end{equation}
To satisfy constraint (\ref{eq:AngularKinematicConstraint}), we may consider a set of $n+1$ sub-elements ($n$ sub-elements with inertia $\mb{J}^*_i$ and one sub-element with inertia $\sum_{i=1}^n \mb{S}(\mb{r}_i) m^*_i \mb{S}(\mb{r}_i)^T$). The angular acceleration of the $n$ elements with moment of inertia tensor $\mb{J}^*_i$ is given by 

\begin{equation}
    \dot{\boldsymbol{\omega}}^*_i = (\mb{J}^*_i)^{-1} \mb{t}_{m,i},
    \label{eq:angularDynamicsLMIE}
\end{equation}
and the angular acceleration of the system of $n$ point masses is given by 

\begin{equation}
    \dot{\boldsymbol{\omega}}^*_f = \left(\sum_{i=1}^n \mb{S}(\mb{r}_i) m^*_i \mb{S}(\mb{r}_i)^T \right)^{-1} \mb{t}_{fm}.
    \label{eq:angularDynamicsForces}
\end{equation}

Substituting (\ref{eq:angularDynamicsWhole}) and (\ref{eq:angularDynamicsLMIE}) into (\ref{eq:AngularKinematicConstraint}) yields

\begin{equation}
    \begin{aligned}
        (\mb{J}^*_o)^{-1} \mb{t}_o &= (\mb{J}^*_i)^{-1} \mb{t}_{m,i} \\
        ||\mb{t}_o|| (\mb{J}^*_o)^{-1} \mb{u}_o &= ||\mb{t}_{m,i}|| (\mb{J}^*_i)^{-1} \mb{u}_o \quad \forall i, \label{eq:manipulatingTorqueDynamicEquivalence}
\end{aligned}
\end{equation} 
since every $\mb{t}_{m,i}$ must be parallel to $\mb{t}_o$. For (\ref{eq:manipulatingTorqueDynamicEquivalence}) to hold for all vectors $\mb{u}_o$, we must have

\begin{equation}
    ||\mb{t}_o|| (\mb{J}^*_o)^{-1} = ||\mb{t}_{m,i}|| (\mb{J}^*_i)^{-1}.
\end{equation}
Solving for $\mb{J}^*_i$, we obtain

\begin{equation}
    \mb{J}^*_i = \frac{||\mb{t}_{m,i}||}{||\mb{t}_o||} \mb{J}^*_o,
\end{equation}
where we see that $\mb{J}^*_i$ and $\mb{J}^*_o$ must be proportional.

Repeating this process for the manipulating force-induced torque $\mb{t}_{fm}$ yields

\begin{equation}
    \sum_{i=1}^n \mb{S}(\mb{r}_i) m^*_i \mb{S}(\mb{r}_i)^T = \frac{||\mb{t}_{fm}||}{||\mb{t}_o||} \mb{J}^*_o.
\end{equation}

We therefore state that the inertia parameters $(m^*_i, \mb{J}^*_i)$ chosen for the LMIEs must also satisfy

\begin{equation}
        \mb{J}^*_i \ \propto \sum_{i=1}^n \mb{S}(\mb{r}_i) m^*_i \mb{S}(\mb{r}_i)^T \propto \ \mb{J}^*_o \ \forall i
\end{equation}
in order for $\mb{G}^+_M$ to return valid manipulating wrench distributions. \textcolor{black}{Note that the symbol $\propto$ means }\say{proportional to}.

\textcolor{black}{The reason for this oversight is likely that Erhart and Hirche only provide simplified examples in one and two dimensions. In these cases, there are either none or only one axis of rotation (no rotations are permitted in 1 dimension and only rotation about an axis normal to the plane is permitted in 2 dimensions). Hence, the proportionality of the rotational inertia was guaranteed because it was defined by a scalar value. In 3 dimensions, rotational inertia is defined by a $3 \times 3$ inertia tensor, and proportionality is an important consideration that cannot be omitted when choosing the parameters $(m^*_i, \mb{J}^*_i)$ as we have shown here. }

\section{Generalized Theory of Load Distribution in Robotic Systems}
\label{sec:GeneralizedTheory}
\subsection{\textcolor{black}{Overview}}
\textcolor{black}{Our theorem describes the conditions under which no internal loads are generated in a rigid body subjected to multiple wrenches applied by independent kinematic chains. Given the resultant wrench and wrench application points, it fully characterizes the feasible set of $n$ manipulating wrenches. It does so by applying the Udwadia-Kalaba equation to a dynamically equivalent system of the virtual rigid body consisting of $n$ LMIEs located at the wrench application points. }

\textcolor{black}{The theorem works by considering the instantaneous acceleration $\ddot{\mb{x}}^*_o$ that would be induced on the virtual rigid body at rest by the resultant wrench obtained from }(\ref{eq:FullDynamics}) \textcolor{black}{at a given time step in a dynamic trajectory. This instantaneous acceleration is distinct from the actual instantaneous acceleration $\ddot{\mb{x}}_o$ of the physical object. For example, in a static application where $\mb{h}_o = \mb{h}_{ext}$, the physical object is held still by wrenches applied to it by its environment. Nonetheless, we calculate the instantaneous acceleration $\ddot{\mb{x}}^*_o$ that would be induced on the virtual rigid body by $\mb{h}_o$ were it not held still, and use this to calculate the wrenches required to generate $\mb{h}_o$ while not generating internal loads.}

\textcolor{black}{The procedure can be explained in three steps. First, the instantaneous acceleration $\ddot{\mb{x}}^*_o$ of the virtual rigid body induced by the resultant wrench $\mb{h}_o$ is calculated. Second, the resulting constrained acceleration $\ddot{\mb{x}}^*_i$ of each LMIE is determined using the kinematic constraints. Third, the manipulating wrenches $\mb{h}_{m,i}$ required to induce these constrained accelerations are obtained from the dynamics equations of the LMIEs. }

\subsection{Fundamental Theorem of Load Distribution in Robotic Systems}
We now present our main result. 

\begin{theorem}
\label{th:InternalLoadsAsHelicoidalField}
     Let a virtual rigid body of mass $m^*_o$ and inertia tensor $\mb{J}^*_o$ at rest be subjected to a resultant wrench $\mb{h}_o$ with respect to its CoM obtained from (\ref{eq:FullDynamics}). Denote by 

    \begin{equation}
    \label{eq:RBDynamics}
        \ddot{\mb{x}}^*_o = \begin{bmatrix}
            \ddot{\mb{p}}^*_o \\ \dot{\boldsymbol{\omega}}^*_o
        \end{bmatrix} = (\mb{M}^*_o)^{-1} \mb{h}_o
    \end{equation}
    the instantaneous acceleration of the virtual rigid body induced by $\mb{h}_o$. Let the instantaneous acceleration at a location $\mb{r}_i$ with respect to the CoM of the virtual rigid body be obtained with the kinematic constraints such that

    \begin{equation}
        \ddot{\mb{x}}^*_i = \begin{bmatrix}
            \ddot{\mb{p}}^*_i \\ \dot{\boldsymbol{\omega}}^*_i
        \end{bmatrix} = \begin{bmatrix}
            \ddot{\mb{p}}^*_o + \dot{\boldsymbol{\omega}}^*_o \times \mb{r}_i \\ \dot{\boldsymbol{\omega}}^*_o
        \end{bmatrix}.
        \label{eq:LMIEConstrainedAcceleration}
    \end{equation}
    
    For each stacked wrench vector $\mb{h} = \begin{bmatrix}
        \mb{h}_1^T & \ldots & \mb{h}_n^T 
    \end{bmatrix}^T$ that solves (\ref{eq:staticEq}), if each applied wrench $\mb{h}_i$ imparts an instantaneous virtual acceleration 

    \begin{equation}
    \label{eq:LMIEDynamics}
        (\ddot{\mb{x}}^*_i)^d = (\mb{M}^*_i)^{-1} \mb{h}_i
    \end{equation}
    onto the LMIE at $\mb{r}_i$ for some positive-definite virtual masses and inertia tensors $m^*_i \in \mathbb{R}$ and $\mb{J}^*_i \in \mathbb{R}^{3 \times 3}$ with 
    \begin{align}
        m^*_o &= \sum_{i=1}^n m^*_i, \label{eq:virtualMassSum} \\
        \mb{J}^*_o &= \sum_{i=1}^n \mb{J}_i^* + \sum_{i=1}^n \mb{S}(\mb{r}_i) m^*_i \mb{S}(\mb{r}_i)^T, \label{eq:inertiaTensorEquivalence} \\
        \quad & \sum_{i=1}^n \mb{r}_i m^*_i = \mb{0}, \label{eq:virtualCoMEquivalence} \\
        \mb{J}^*_i \ \propto& \sum_{i=1}^n \mb{S}(\mb{r}_i) m^*_i \mb{S}(\mb{r}_i)^T \propto \ \mb{J}^*_o \ \forall i, \label{eq:inertiaProportionality}
    \end{align}
    then the necessary and sufficient condition for no internal loads is

    \begin{equation}
        (\ddot{\mb{x}}^*_i)^d = \ddot{\mb{x}}^*_i \ \Longrightarrow \ \mb{h}_i =\mb{h}_{m,i} \quad \forall i
    \end{equation}
\end{theorem}

The proof of this theorem is given in Appendix A. 

\subsection{The Helicoidal Virtual Acceleration Field as the Sum of Two Fields}
\textcolor{black}{In our characterization of manipulating wrenches, we are primarily concerned with the instantaneous accelerations $\ddot{\mb{x}}^*_i$ of the LMIEs that form the dynamically equivalent system of the virtual rigid body. Additionally, a manipulating force $\mb{f}_{m,i}$ can induce linear and angular acceleration in the rigid body, since it also applies a torque $\mb{r}_i \times \mb{f}_{m,i}$, whereas a manipulating torque $\mb{t}_{m,i}$ can only induce angular acceleration. Therefore, we deduce that all of the linear acceleration of the rigid body must be produced by the manipulating forces, while the angular acceleration can be produced by a combination of the manipulating forces and the manipulating torques. }

\textcolor{black}{The portion of the angular acceleration generated by either the manipulating forces or a given manipulating torque depends on its respective contribution to the overall inertia of the virtual rigid body. If all the angular acceleration is to be generated by the manipulating forces, then the inertia contribution of the masses of the LMIEs $\sum_{i=1}^n \mb{S}(\mb{r}_i) m^*_i \mb{S}(\mb{r}_i)^T$ must be equal to the total inertia we assign to the virtual rigid body $\mb{J}^*_o$ and we set $\mb{J}^*_i = \mb{0}, \ \forall i$. Conversely, if we wish for the manipulating torques to produce some of the angular acceleration, we must increase their contributions to the overall inertia of the system. For example, if the manipulating torques are to generate half of the angular acceleration while the manipulating forces generate the other half, we need to design our dynamically equivalent system such that the system of LMIE masses and the sum of the LMIE inertia tensors contribute equally to the overall inertia of the system. In this case, we set $\mb{J}^*_o = 2\sum_{i=1}^n \mb{S}(\mb{r}_i) m^*_i \mb{S}(\mb{r}_i)^T$ and $\sum^n_{i=1} \mb{J}^*_i = \sum_{i=1}^n \mb{S}(\mb{r}_i) m^*_i \mb{S}(\mb{r}_i)^T$.}

We can therefore write the instantaneous acceleration of the object as

\begin{equation}
    \ddot{\mb{x}}^*_o = \ddot{\mb{x}}^*_f + \ddot{\mb{x}}^*_t
    \label{eq:AccelFieldSum}
\end{equation}
where $\ddot{\mb{x}}^*_f = \begin{bmatrix}
    \ddot{\mb{p}}^{*T}_o & \dot{\Bar{\boldsymbol{\omega}}}^{*T}_f
\end{bmatrix}^T$ and $\ddot{\mb{x}}^*_t = \begin{bmatrix}
    \mb{0}_{3 \times 1} & \dot{\Bar{\boldsymbol{\omega}}}^{*T}_t
\end{bmatrix}^T$ are the components of acceleration of the virtual rigid body generated by the manipulating forces and manipulating torques, respectively. Note that vectors $\dot{\Bar{\boldsymbol{\omega}}}^*_f$ and $\dot{\Bar{\boldsymbol{\omega}}}^*_t$ are different from vectors $\dot{\boldsymbol{\omega}}^*_f$ and $\dot{\boldsymbol{\omega}}^*_i$ discussed in section \ref{sec:AdditionalConstraintJustification}. On the one hand, $\dot{\boldsymbol{\omega}}_f$ and $\dot{\boldsymbol{\omega}}_i$ denote the angular acceleration of sub-elements in the dynamically equivalent system of the virtual rigid body. $\dot{\boldsymbol{\omega}}_f$ is the angular acceleration of the sub-system of point masses $m^*_i$ with inertia $\sum_{i=1}^n \mb{S}(\mb{r}_i) m^*_i \mb{S}(\mb{r}_i)^T$ and $\dot{\boldsymbol{\omega}}^*_i$ is the angular acceleration of the $i$-th LMIE with moment of inertia $\mb{J}^*_i$. On the other hand, $\dot{\Bar{\boldsymbol{\omega}}}^{*}_f$ and $\dot{\Bar{\boldsymbol{\omega}}}^{*}_t$ are the individual contributions of the individual wrench components on the virtual rigid body's acceleration considering its \emph{total} inertia. Therefore, we have $\dot{\boldsymbol{\omega}}^*_i = \dot{\boldsymbol{\omega}}^*_f = \dot{\boldsymbol{\omega}}^*_o$ and $\dot{\boldsymbol{\omega}}^*_o = \dot{\Bar{\boldsymbol{\omega}}}^*_f + \dot{\Bar{\boldsymbol{\omega}}}^*_t$.

$\ddot{\mb{x}}^*_f$ and $\ddot{\mb{x}}^*_t$ both define helicoidal vector fields and have parallel directions $(\dot{\Bar{\boldsymbol{\omega}}}^*_f \parallel \dot{\Bar{\boldsymbol{\omega}}}^*_t)$. They are obtained with 

\begin{align}
    \ddot{\mb{x}}^*_f &= \begin{bmatrix}
        \frac{1}{m^*_o} \mb{I}_3 & \mb{0}_{3 \times 3} \\ \mb{0}_{3 \times 3} & (\mb{J}^*_o)^{-1}
    \end{bmatrix} \begin{bmatrix}
        \mb{f}_o \\ \sum_{i=1}^n \mb{r}_i \times \mb{f}_{m,i}
    \end{bmatrix}, \label{eq:forceHelicoid} \\
     \ddot{\mb{x}}^*_t &= \begin{bmatrix}
        \mb{0}_{3 \times 3} & \mb{0}_{3 \times 3} \\ \mb{0}_{3 \times 3} & (\mb{J}^*_o)^{-1}
    \end{bmatrix}  \begin{bmatrix}
        \mb{0}_{3 \times 1} \\ \sum_{i=1}^n \mb{t}_{m,i}
    \end{bmatrix}.
\end{align}

\subsection{Uniqueness of the Manipulating Wrench Distribution}
\label{sec:uniqueness}
The uniqueness of the solution is discussed briefly in \cite{Erhart2015}. Here, we aim to correct and complete their description.

The uniqueness of the manipulating wrench distribution is governed by the number of solutions for the LMIE inertia parameters $(m^*_i, \mb{J}^*_i)$ to constraints (\ref{eq:virtualMassSum})-(\ref{eq:inertiaProportionality}) from which we obtain 10 scalar equalities; one from (\ref{eq:virtualMassSum}), six from the six independent values in the $3 \times 3$ symmetric inertia tensor in (\ref{eq:inertiaTensorEquivalence}), and three from (\ref{eq:virtualCoMEquivalence}). Also remember that the inertia matrix of the dynamically equivalent system must be proportional to that of the physical object ($\mb{M}^*_o = k \mb{M}_o$).

Let us begin by setting $\mb{J}^*_i = \mb{0}_{3 \times 3} \ \forall i$ and considering only a system of $n$ point masses. This situation arises in applications which involve only pure forces applied to the object. \textcolor{black}{Since the application points $\mb{r}_i$ of the wrenches on the rigid body are assumed to be known at the instant, each additional point mass added to the system only introduces one new variable (its virtual mass $m^*_i$) that can be used to solve the system of equations defined by constraints} (\ref{eq:virtualMassSum})-(\ref{eq:inertiaProportionality}). \textcolor{black}{These constraints can be written in matrix form as }

\begin{equation}
\begin{aligned}
    \mb{R}& \mb{m}^* = \mb{c} \\
    \begin{bmatrix}
        1 & \ldots & 1 \\
        \mb{s}^*_1 & \ldots & \mb{s}^*_n \\
        \mb{r}_1 & \ldots & \mb{r}_n
    \end{bmatrix}_{10 \times n}& \begin{bmatrix}
        m^*_1 \\ \vdots \\ m^*_n
    \end{bmatrix} = \begin{bmatrix}
        m^*_o \\ \mb{j}^*_o \\ \mb{0}_{3 \times 1}
    \end{bmatrix}
\end{aligned}
\label{eq:VirtualMassSolutionAll}
\end{equation}
\textcolor{black}{where $m^*_o = k m_o$ is the mass of the dynamically equivalent system and $\mb{j}^*_o$ is a 6-dimensional vector which contains the 6 independent values in symmetric matrix $\mb{J}^*_o$. In other words, if }

\begin{equation}
    \mb{J}^*_o = k \mb{J}_o = k\begin{bmatrix}
        j_{o,11} & j_{o,12} & j_{o,13} \\
        j_{o,21} & j_{o,22} & j_{o,23} \\
        j_{o,31} & j_{o,32} & j_{o,33}
    \end{bmatrix}
\end{equation}
\textcolor{black}{with $j_{o,ij} = j_{o,ji}$, then }

\begin{equation}
    \mb{j}^*_o = k\begin{bmatrix}
        j_{o,11} \\ j_{o,22} \\ j_{o,33} \\ j_{o,12} \\ j_{o,13} \\ j_{o,23}
    \end{bmatrix}.
\end{equation}
\textcolor{black}{Vectors $\mb{s}^*_i, \ i = 1, ..., n$ contain the six independent elements of the symmetric matrices}

\begin{equation}
    \mb{S}(\mb{r}_i) \mb{S}(\mb{r}_i)^T = \begin{bmatrix}
        z_i^2 + y_i^2 & -x_i y_i & -x_i z_i \\
        -x_i y_i & x_i^2 + z_i^2 & -y_i z_i \\
        -x_i z_i & -y_i z_i & x_i^2 + y_i^2
    \end{bmatrix}.
\end{equation}
\textcolor{black}{Thus, }

\begin{equation}
    \mb{s}_i = \begin{bmatrix}
        z_i^2 + y_i^2 \\ x_i^2 + z_i^2 \\ x_i^2 + y_i^2 \\ -x_i y_i \\ -x_i z_i \\ -y_i z_i
    \end{bmatrix}.
\end{equation}

\textcolor{black}{From }(\ref{eq:VirtualMassSolutionAll}), \textcolor{black}{it is evident that constraints }(\ref{eq:virtualMassSum})-(\ref{eq:virtualCoMEquivalence}) \textcolor{black}{severely limit the feasible solutions for the vector of virtual masses $\mb{m}^*$. In }\cite{Erhart2015}, \textcolor{black}{the authors claimed that the solution was unique when $n=3$ and that there were infinitely many solutions when $n > 3$. However, the dimensions of matrix $\mb{R}$ in }(\ref{eq:VirtualMassSolutionAll}) \textcolor{black}{reveal that 10 pure forces are required to satisfy the constraints and infinitely many solutions exist when $n>10$. This is hardly practical in real robotic applications, since this would require robotic hands to have 10 independent fingers or legged robots to have 10 independent legs in order to satisfy the requirements for no internal load generation.}

\textcolor{black}{We propose an alternate method for finding solutions which are more practical. The most restrictive constraint is }(\ref{eq:inertiaTensorEquivalence}) \textcolor{black}{since it introduces six equality constraints. By relaxing this constraint, the problem becomes more tractable. Note that we must still satisfy constraint }(\ref{eq:inertiaProportionality}). 

\textcolor{black}{To widen the solution space for the vector of virtual masses $\mb{m}^*$ which satisfies }(\ref{eq:virtualMassSum})-(\ref{eq:inertiaProportionality}), \textcolor{black}{we can create a system of equations of lower rank by considering only constraints} (\ref{eq:virtualMassSum}) and (\ref{eq:virtualCoMEquivalence}). Then, we simply set $\mb{J}^*_o = \sum_{i=1}^n \mb{S}(\mb{r}_i) m^*_i \mb{S}(\mb{r}_i)^T$ to satisfy constraints (\ref{eq:inertiaTensorEquivalence}) and (\ref{eq:inertiaProportionality}). Constraints (\ref{eq:virtualMassSum}) and (\ref{eq:virtualCoMEquivalence}) may be written in matrix form as 

\begin{equation}
\begin{aligned}
    \mb{R}& \mb{m}^* = \mb{c} \\
    \begin{bmatrix}
        1 & \ldots & 1 \\ \mb{r}_1 & \ldots & \mb{r}_n
    \end{bmatrix}_{4 \times n}& \begin{bmatrix}
        m^*_1 \\ \vdots \\ m^*_n
    \end{bmatrix} = \begin{bmatrix}
        m^*_o \\ \mb{0}_{3 \times 1}
    \end{bmatrix}.
    \label{eq:VirtualMassSolutionReduced}
\end{aligned}
\end{equation}

In this reduced system, $\mb{m}^*$ is uniquely determined in a 3-dimensional space if $n = 4$ and admits infinitely many solutions if $n > 4$. 

\textcolor{black}{The need for only four pure forces is much more practical for real robotic applications. A solution obtained from} (\ref{eq:VirtualMassSolutionReduced}) \textcolor{black}{will generate internal loads as they are defined in }\cite{Udwadia1992} and \cite{Erhart2015}, \textcolor{black}{but can be seen as a more practical alternative if 10 independent forces are not available.}

Also, since $m^*_i \geq 0 \ \forall i$, the CoM of the virtual rigid body must lie within the convex hull of points $\{ \mb{r}_1, \ldots, \mb{r}_n \}$. It is the value of $m^*_i$ which determines the magnitude of the manipulating force $\mb{f}_{m,i}$ needed to generate the constrained acceleration $\ddot{\mb{p}}^*_i$ of the $i$-th LMIE. 

Finally, let us consider that the $n$ wrenches may also apply pure torques to the rigid body. The manipulating torques are constrained to be parallel to $\mb{t}_o$, so each additional manipulating torque that may be applied to the rigid body only introduces one additional dimension to the solution space for the manipulating wrench distribution. The feasible set $\mathcal{T}_m$ for the manipulating torque magnitudes can be described as 

\begin{equation}
    \mathcal{T}_m = \{ \hat{\mb{t}}_m \in \mathbb{R}^n \mid \sum^n_{i=1} \hat{t}_{m,i} \leq ||\mb{t}_o || \},
    \label{eq:manipTorqueFeasibleSet}
\end{equation}
where $\hat{t}_{m,i}$ is the $i$-th component of $\hat{\mb{t}}_m = \begin{bmatrix}
    || \mb{t}_{m,1}|| & \ldots & || \mb{t}_{m,n}||
\end{bmatrix}^T$, which is the vector of manipulating torque magnitudes. Finally, the dimension $d$ of the feasible set of manipulating wrench distributions for $n$ wrenches and a given resultant wrench $\mb{h}_o$ is obtained with

\begin{equation}
    d = n + \text{dim}(\mathcal{N}(\mb{R}))
\end{equation}
where $\mathcal{N}(\mb{R})$ is the null space of matrix $\mb{R}$ whether written in the form shown in (\ref{eq:VirtualMassSolutionAll}) or (\ref{eq:VirtualMassSolutionReduced}). If $d=0$, the solution is uniquely determined. If only forces may be applied to the object, the dimension of the solution space reduces to $d = \text{dim}(\mathcal{N}(\mb{R}))$.

\subsection{Unique Representation of the Internal Loading State of the Rigid Body}
\label{sec:internalLoadingState}
The fact that an infinite number of manipulating wrench distributions can exist may seem counterintuitive, especially in the context of wrench analysis. If we have a given arbitrary applied wrench distribution $\mb{h}$ and wish to find all feasible decompositions into the manipulating wrench components $\mb{h}_{m}$, which induce motion, and the constraint wrench components $\mb{h}_{c}$, which contribute to the \say{squeezing} of the object, the choice of solution for $\mb{h}_{m}$ will affect the corresponding constraint wrench vector $\mb{h}_{c}$ (see Eq. (\ref{eq:wrenchEquilibrium})). However, for a given wrench distribution applied to the real object, we do not expect the squeezing on the object to depend on our choice of $\mb{h}_{m}$. There should therefore exist a unique way of describing this internal loading state, despite there being infinitely many solutions for $\mb{h}_{m}$. 

Assume there are an infinite number of possible manipulating wrench distributions for a given resultant wrench $\mb{h}_o$. Let $\mathcal{H}_m$ and $\mathcal{H}_c$ be the infinite sets of feasible solutions for the manipulating and constraint wrench distributions, respectively. Every vector in set $\mathcal{H}_m$ can be obtained by applying Theorem \ref{th:InternalLoadsAsHelicoidalField}. The mapping between $\mathcal{H}_m$ and $\mathcal{H}_c$ is one-to-one, and $\mathcal{H}_c$ can be described as 

\begin{equation}
    \mathcal{H}_c = \{ \mb{h}_{cj} \in\mathbb{R}^{6n} \mid \mb{h}_{cj} = \mb{h}_{mj} - \mb{h}, \ \mb{h}_{mj} \in \mathcal{H}_m \}
\end{equation}
where $\mb{h}_{mj}$ and $\mb{h}_{cj}$ are the $j$-th elements of $\mathcal{H}_m$ and $\mathcal{H}_c$, respectively, and $\mb{h}$ is the unique vector containing the applied wrenches. $\mathcal{H}_c$ is therefore simply a translation of $\mathcal{H}_m$. 

In \cite{Kumar1988}, interaction forces are defined as the components of the applied forces that lie in the null space of the grasp matrix, $\mathcal{N}(\mb{G})$. However, we have found that this reasoning cannot be applied to the manipulating wrench distribution, since it may have a null-space component yet still not generate internal loads. 

Let $\mb{h}_{m1}$ and $\mb{h}_{m2}$ be two manipulating wrench distributions in the feasible set $\mathcal{H}_m$ for a resultant wrench $\mb{h}_o$ such that 

\begin{align}
    \mb{G} \mb{h}_{m1} &= \mb{h}_o, \\
    \mb{G} \mb{h}_{m2} &= \mb{h}_o.
\end{align}
If we subtract these two equations, we obtain 

\begin{equation}
    \mb{G}(\mb{h}_{m2} - \mb{h}_{m1}) = \mb{0}_{6 \times 1}.
\end{equation}
Hence, the vector $\Delta \mb{h}_m = \mb{h}_{m2} - \mb{h}_{m1}$ also lies in $\mathcal{N}(\mb{G})$, yet adding $\Delta \mb{h}_m$ to a manipulating wrench distribution does not generate internal loads since the resulting vector is also a manipulating wrench distribution. 

Note that constraint wrenches do still lie in $\mathcal{N}(\mb{G})$ since they must satisfy the condition 

\begin{equation}
    \mb{G} \mb{h}_c = \mb{0}_{6 \times 1},
\end{equation}
i.e., they are invisible to the net wrench $\mb{h}_o$. 

With this observation, we are able to find a unique representation of the set of constraint wrenches $\mb{h}_c$ in a feasible set $\mathcal{H}_c$ by defining two sub-spaces within $\mathcal{N}(\mb{G})$. 

We assume that the grasp matrix $\mb{G}$ is full rank. The null space of $\mb{G}$ is therefore $(6n-6)$-dimensional. Let $\mathcal{N}_m$ be the $d$-dimensional subspace within $\mathcal{N}(\mb{G})$ that corresponds to the vectors $\Delta \mb{h}_m$ which may be added to a particular solution for the manipulating wrench distribution, denoted $\mb{h}_{m,p}$, to span the entire set of feasible solutions for $\mb{h}_m$. Any vector in the feasible set $\mathcal{H}_m$ can then be written as

\begin{equation}
    \mb{h}_m = \mb{h}_{m,p} + \Delta \mb{h}_m, \quad \Delta \mb{h}_m \in \mathcal{N}_m.
\end{equation}
Alternatively, this can be written as 

\begin{align}
    \mb{h}_m = \mb{h}_{m,p} + \mb{K} \boldsymbol{\lambda}_m
\end{align}
where $\mb{K}$ is a $(6n \times d)$ matrix whose columns are a set of basis vectors that span $\mathcal{N}_m$ and $\boldsymbol{\lambda}_m$ is a $d$-dimensional vector which uniquely determines the location of a given vector $\Delta \mb{h}_m$ within this space.  

The choice of $\mb{h}_{m,p}$ will affect $\mathcal{N}_m$. For consistency, we propose the convention of using the unique forces-only solution $(\mb{t}_{m,i} = \mb{0}_{3 \times 1}, \ \forall i)$ for $\mb{h}_m$ as  $\mb{h}_{m,p}$ when $\mb{R}$ is square, or the unique forces-only solution for $\mb{h}_m$ that corresponds to the minimum-norm solution for $\mb{m}^*$ as  $\mb{h}_{m,p}$ when $\mb{R}$ is a wide matrix. 

The second $(6n - 6 - d)$-dimensional subspace within $\mathcal{N}(\mb{G})$, denoted $\mathcal{N}_c$, may be used to generate or eliminate constraint wrenches. We know that a vector $\mb{h}_c$ within this space lies in the null space of $\mb{G}$ and is orthogonal to subspace $\mathcal{N}_m$ (i.e. orthogonal to the column space of $\mb{K}$). These conditions may be written as 

\begin{equation}
    \mb{G} \mb{h}_c = \mb{0}_{6 \times 1}, \ \text{and} \ \mb{K}^T \mb{h}_c = \mb{0}_{d \times 1}.
\end{equation}
Collecting these linear equalities, we obtain

\begin{equation}
    \begin{bmatrix}
        \mb{G} \\ \mb{K}^T
    \end{bmatrix} \mb{h}_c = \mb{0}_{(6+d) \times 1}
\end{equation}
and thus, 

\begin{equation}
    \mathcal{N}_c = \mathcal{N}\left( \begin{bmatrix}
        \mb{G} \\ \mb{K}^T
    \end{bmatrix} \right).
\end{equation}
where $\mathcal{N}(\cdot)$ denotes the null-space of the argument. Finally, we may define $6n - 6 - d$ linearly independent basis vectors that span $\mathcal{N}_c$ such that any vector $\mb{h}_c$ can be expressed as 

\begin{equation}
    \mb{h}_c = \mb{Z} \boldsymbol{\lambda}_c
\end{equation}
where $\mb{Z}$ is a matrix of dimension $6n \times (6n - 6 - d)$ whose columns correspond to linearly independent basis vectors that span $\mathcal{N}_c$ and $\boldsymbol{\lambda}_c$ is a $(6n - 6 - d)$-dimensional vector. An interesting consequence of defining $\mb{h}_c$ in this manner is that, for a wrench distribution which admits infinitely many solutions, we may write

\begin{equation}
    \mb{Z}^{\dagger} \mb{h}_{ck} = \boldsymbol{\lambda}_c \ \forall k, \ \text{with} \ \mb{h}_{ck} \in \mathcal{N}_c,
    \label{eq:lambda_c}
\end{equation}
where $\mb{Z}^{\dagger} = (\mb{Z}^T \mb{Z})^{-1} \mb{Z}^T$ is the left Moore-Penrose pseudo-inverse of $\mb{Z}$ and $\mb{h}_{ck}$ is the $k$-th vector in the infinite set $\mathcal{N}_c$. The physical significance of this result is that every constraint wrench vector $\mb{h}_{cj}$ in a given feasible set $\mathcal{H}_c$ has the same corresponding vector $\boldsymbol{\lambda}_c$ which uniquely defines the location in the constraint wrench space $\mathcal{N}_c$ of the set of applied wrenches $\mb{h}$ and thus uniquely defines the internal loading state of the virtual rigid body. In other words, for a given set of applied wrenches, $\boldsymbol{\lambda}_c$ is independent of the choice of solution for $\mb{h}_m$.

\section{\textcolor{black}{Implementation}}
\label{sec:Implementation}
\subsection{Theorem \ref{th:InternalLoadsAsHelicoidalField} for Wrench Synthesis}
The problem of wrench synthesis consists in prescribing the wrenches $\mb{h}_i$ to be applied by the kinematic chains in the context of robot force-control to generate an output wrench $\mb{h}_o$ while either eliminating internal loads or generating a prescribed internal loading state on the object. When doing so with our proposed framework, we must first decide how closely we wish the inertial properties of the virtual body to match the inertial properties of the real body. This will determine how many wrenches are required to eliminate internal loads (see Section \ref{sec:uniqueness}). Then, with the results presented in Section \ref{sec:internalLoadingState}, we are able to write the applied wrenches as 

\begin{equation}
    \mb{h} = \mb{h}_{m,p} + \mb{K} \boldsymbol{\lambda}_m - \mb{Z} \boldsymbol{\lambda}_c.
    \label{eq:wrenchSynthesis}
\end{equation}

For a given $\mb{h}_o$, $\mb{h}_{m,p}$ uniquely describes the set of manipulating wrench distributions. The departure from this unique solution within $\mathcal{N}_m$ is governed by $\boldsymbol{\lambda}_m$ and dictates how the load is balanced between the manipulating wrench components. Finally, $\boldsymbol{\lambda}_c$ uniquely defines the internal loading state induced by the remaining components in $\mathcal{N}(\mb{G})$. When $\mb{h}_o$ is given, we may obtain $\mb{h}_{m,p}$ and $\boldsymbol{\lambda}_m$ from Theorem \ref{th:InternalLoadsAsHelicoidalField}. If we wish to apply resultant wrench $\mb{h}_o$ without generating internal loads, we set $\boldsymbol{\lambda}_c = \mb{0}$. Alternatively, a designer may prescribe specific internal loading states at each time $t$ for a given task with vector $\boldsymbol{\lambda}_c$.

\subsection{Theorem \ref{th:InternalLoadsAsHelicoidalField} for Wrench Analysis}
\label{sec:WrenchAnalysis}
It is trivial to apply Theorem \ref{th:InternalLoadsAsHelicoidalField} to wrench analysis where the aim is to decompose an arbitrary set of $n$ wrenches $\mb{h}_i$ applied at locations $\mb{r}_i$ for $i = 1, ..., n$ into a set of $n$ manipulating wrenches $\mb{h}_{m,i}$ and $n$ constraint wrenches $\mb{h}_{c,i}$. Since every $\mb{h}_i$ is known, we may calculate $\mb{h}_o$ from (\ref{eq:staticEq}). Then, a valid manipulating wrench distribution $\mb{h}_m$ can be obtained from $\mb{h}_o$ by applying Theorem \ref{th:InternalLoadsAsHelicoidalField}. The corresponding constraint wrench distribution $\mb{h}_c$ may be easily obtained from the Udwadia-Kalaba equation (\ref{eq:wrenchEquilibrium}) as 

\begin{equation}
    \mb{h}_c = \mb{h}_m - \mb{h}.
    \label{eq:constraintWrenchAnalysis}
\end{equation}

This is a very simple and non-iterative method for performing wrench decomposition. The fact that it is non-iterative also makes it much more suitable for real-time force-control compared to the optimization-based approaches proposed in \cite{Schmidts2016} and \cite{Donner2018}.

\section{Implications of Theorem \ref{th:InternalLoadsAsHelicoidalField}}
\label{sec:implications}
We now revisit several frameworks proposed in the literature based on the results presented in the previous section and provide corrections where necessary. 

\subsection{Implications for the Parametrized Moore-Penrose Pseudo-inverse of the Grasp Matrix}
\label{sec:Implications-ParamMPInv}
\cite{Erhart2015} propose the parametrized Moore-Penrose inverse of the grasp matrix shown in (\ref{eq:BadParametrizedInverse}). As mentioned in Section \ref{sec:RelatedWork-WrenchSynthesis}, we have found this form of $\mb{G}^+_M$ to be incorrect, and we have already justified the need for an additional constraint in Section \ref{sec:AdditionalConstraintJustification}. We therefore propose the following corrected theorem: 

\setcounter{theorem}{0}

\begin{theorem}[Corrected]
    \label{th:GoodParamMPInv}
    The load distribution given by 
    \begin{equation}
        \mb{G}^+_M = \begin{bmatrix}
            \frac{m^*_1}{m^*_o} \mb{I}_3 & m^*_1 \mb{S}(\mb{r}_1)^T [\mb{J}^*_o]^{-1}  \\
            \mb{0}_3 & \mb{J}^*_1 [\mb{J}^*_o]^{-1} \\
            \vdots & \vdots \\
            \frac{m^*_n}{m^*_o} \mb{I}_3 & m^*_n \mb{S}(\mb{r}_n)^T [\mb{J}^*_o]^{-1}  \\
            \mb{0}_3 & \mb{J}^*_n [\mb{J}^*_o]^{-1}
        \end{bmatrix}
        \label{eq:GoodParametrizedInverse}
    \end{equation}
    for some positive-definite weighting coefficients $m^*_i \in \mathbb{R}$ and $\mb{J}^*_i \in \mathbb{R}^{3 \times 3}$ subject to constraints (\ref{eq:virtualMassSum})-(\ref{eq:inertiaProportionality}) is free of internal loads.
\end{theorem}

\setcounter{theorem}{2}

The proof of this theorem is shown in Appendix A. Eq. (\ref{eq:GoodParametrizedInverse}) allows for the direct calculation of the parametrized Moore-Penrose pseudo-inverse of $\mb{G}$ which, for $n$ wrenches, is $6 \times 6n$. We therefore circumvent the computational cost associated with the inversion of such a large matrix and the largest matrix that must be inverted is instead the symmetric inertia tensor $\mb{J}^*_o$, which is $3 \times 3$. When the number of applied wrenches is increased, we simply add additional rows to $\mb{G}^+_M$ and hence the complexity of the problem scales linearly with the number of applied wrenches. The problem of finding valid manipulating wrench distributions is therefore reduced to choosing combinations of the virtual inertia parameters $(m^*_i, \mb{J}^*_i)$ that satisfy constraints (\ref{eq:virtualMassSum})-(\ref{eq:inertiaProportionality}).

Erhart and Hirche also provide the following corollary to their theorem: 

\begin{corollary}[As presented in \cite{Erhart2015}]
    An equal distribution of the manipulator weights according to $m^*_i = 1$ and $\mb{J}^*_i = \mb{I}_3$ yields

    \begin{equation}
        \mb{G}^{\dagger} = \frac{1}{n} \begin{bmatrix}
            \mb{I}_3 &  \Bar{\mb{J}}^{-1} \mb{S}(\mb{r}_1)^T \\ \mb{0}_{3} & \Bar{\mb{J}}^{-1} \\ 
            \vdots & \vdots \\
            \mb{I}_3 & \Bar{\mb{J}}^{-1} \mb{S}(\mb{r}_n)^T \\ \mb{0}_{3} & \Bar{\mb{J}}^{-1}
        \end{bmatrix}
        \label{eq:BadMPInverseG}
    \end{equation}
    with $\Bar{\mb{J}} = \mb{I}_3 + \frac{1}{n} \sum_{i=1}^n \mb{S}(\mb{r}_i) \mb{S}(\mb{r}_i)^T$ and (\ref{eq:BadMPInverseG}) being equivalent to the Moore-Penrose inverse of $\mb{G}$.
\end{corollary}

However, we have found examples where this statement does not hold. We propose a corrected form of their corollary here with additional conditions that must be satisfied. The proof is shown in Appendix A.

\setcounter{corollary}{0}

\begin{corollary}[Corrected]
\label{cor:MPInverseG}
    The parametrized Moore-Penrose pseudo-inverse of the grasp matrix $\mb{G}^+_M$ is equal to the unweighted Moore-Penrose pseudo-inverse of the grasp matrix $\mb{G}^{\dagger}$ iff a set of $n$ LMIEs with virtual mass $m^*_i = 1$ and virtual inertia tensor $\mb{J}^*_i = \mb{I}_3$ is a valid solution to constraints (\ref{eq:virtualMassSum})-(\ref{eq:inertiaProportionality}). In this case, $\mb{G}^+_M$ is written as
    \begin{equation}
        \mb{G}^+_M = \mb{G}^{\dagger} = \frac{1}{n} \begin{bmatrix}
            \mb{I}_3 & \mb{S}(\mb{r}_1)^T \Bar{\mb{J}}^{-1} \\ \mb{0}_{3 \times 3} & \Bar{\mb{J}}^{-1} \\ 
            \vdots & \vdots \\
            \mb{I}_3 & \mb{S}(\mb{r}_n)^T \Bar{\mb{J}}^{-1} \\ \mb{0}_{3 \times 3} & \Bar{\mb{J}}^{-1}
        \end{bmatrix}
        \label{eq:GoodMPInverseG}
    \end{equation}
    with $\Bar{\mb{J}} = \mb{I}_3 + \frac{1}{n} \sum_{i=1}^n \mb{S}(\mb{r}_i) \mb{S}(\mb{r}_i)^T$.
\end{corollary}

\setcounter{corollary}{1}

\subsection{Implications for the Equilibrating Wrench Distribution}
\cite{Kumar1988} state that the equilibrating force distribution returned by $\mb{G}^{\dagger}$ is free of interaction forces and that the forces in this distribution lie on a helicoidal vector field. Although the former statement is always true, the latter only holds if constraints (\ref{eq:virtualMassSum})-(\ref{eq:inertiaProportionality}) can be satisfied by $m^*_i = 1$ and $\mb{J}^*_i = \mb{I}_3 \ \forall i$.

It was shown in Section \ref{sec:CentripetalForceImplication} that the instantaneous acceleration of the virtual rigid body $\ddot{\mb{x}}^*_o$ defines a helicoidal vector field since we assume it to be at rest. It was also shown in Section \ref{sec:Implications-ParamMPInv} that $\mb{G}^+_M = \mb{G}^{\dagger}$ when $m^*_i = 1$ and $\mb{J}^*_i = \mb{I}_3 \ \forall i$ solves (\ref{eq:virtualMassSum})-(\ref{eq:inertiaProportionality}). As the authors of \cite{Kumar1988} consider only pure forces, we note that the forces-only solution can be found by omitting the blocks $\begin{bmatrix}
    \mb{0}_3 & \Bar{\mb{J}}^{-1}
\end{bmatrix}$ in (\ref{eq:GoodMPInverseG}).

With this realization, we can see that it is only by coincidence that the equilibrating force field forms a helicoidal field in some cases. Under the circumstances described above where $m^*_i = 1 \ \forall i$, it is trivial to see that $\mb{f}_{m,i} = \ddot{\mb{p}}^*_i \ \forall i$. Thus, since the virtual instantaneous acceleration field is known to be helicoidal, the vector field of manipulating forces will be too. In addition, since $\mb{G}^+_M = \mb{G}^{\dagger}$ when $m^*_i = 1 \ \forall i$, the manipulating force distribution and the equilibrating force distribution are identical in this case and the latter is therefore a helicoidal field as well. The result is that all three vector fields are identical and helicoidal when $m^*_i = 1 \ \forall i$ solves (\ref{eq:virtualMassSum})-(\ref{eq:inertiaProportionality}). Otherwise, neither the equilibrating force distribution nor the manipulating force distribution form helicoidal fields.

\subsection{Implications for Physically Feasible Wrench Decomposition Algorithms}
\textcolor{black}{Current state-of-the-art algorithms for wrench decomposition naively assume that manipulating wrenches and constraint wrenches are both components of the applied wrenches and perform the decomposition using an equation similar to the one shown in }(\ref{eq:DecompSchmidtsetal}). \textcolor{black}{This decision is not properly justified in }\cite{Schmidts2016} \textcolor{black}{nor in} \cite{Donner2018}. 

\textcolor{black}{As we have shown, the constraint wrench components are not generated \emph{by} the applied wrenches. Rather, they are \emph{added} to the applied wrenches in order for the resultant (the manipulating wrench) to generate the proper constrained acceleration of its corresponding point on the rigid body. }

\textcolor{black}{This is not a matter of opinion. It is a direct consequence of the principles of rigid body mechanics. The work of }\cite{Udwadia1992} \textcolor{black}{provides a clear, physically intuitive interpretation of constrained motion that clearly highlights the roles of the manipulating and constraint wrenches and makes this distinction clear.}

\textcolor{black}{As a result, we argue that there are two major issues with the method for wrench decomposition proposed in }\cite{Donner2018}. \textcolor{black}{First, any constraint wrenches found with }(\ref{eq:DecompSchmidtsetal}) \textcolor{black}{will be of the wrong sign, since they are on the wrong side of the wrench equilibrium equation which should have the form shown in }(\ref{eq:wrenchEquilibrium}). \textcolor{black}{Second, the feasible region for the manipulating and constraint wrench components will be needlessly limited by physical plausibility conditions derived from an erroneous assumption. The Udwadia-Kalaba equation, on the other hand, guarantees physical plausibility, having been derived from fundamental laws of rigid body mechanics.}

\section{\textcolor{black}{Case Studies}}
\label{sec:Examples}
\subsection{\textcolor{black}{Planar Point Mass}}
\textcolor{black}{This example would seldom be seen in real robotics applications, but enables us to demonstrate how our proposed framework corrects misconceptions in the literature related to wrench decomposition. To this end, we revisit the example from} \cite{Donner2018} \textcolor{black}{where three planar forces are applied to a point mass (see Figure} \ref{fig:Donneretal-ProblemSetup}).

\begin{figure}
    \centering
    \includegraphics[width=0.8\linewidth]{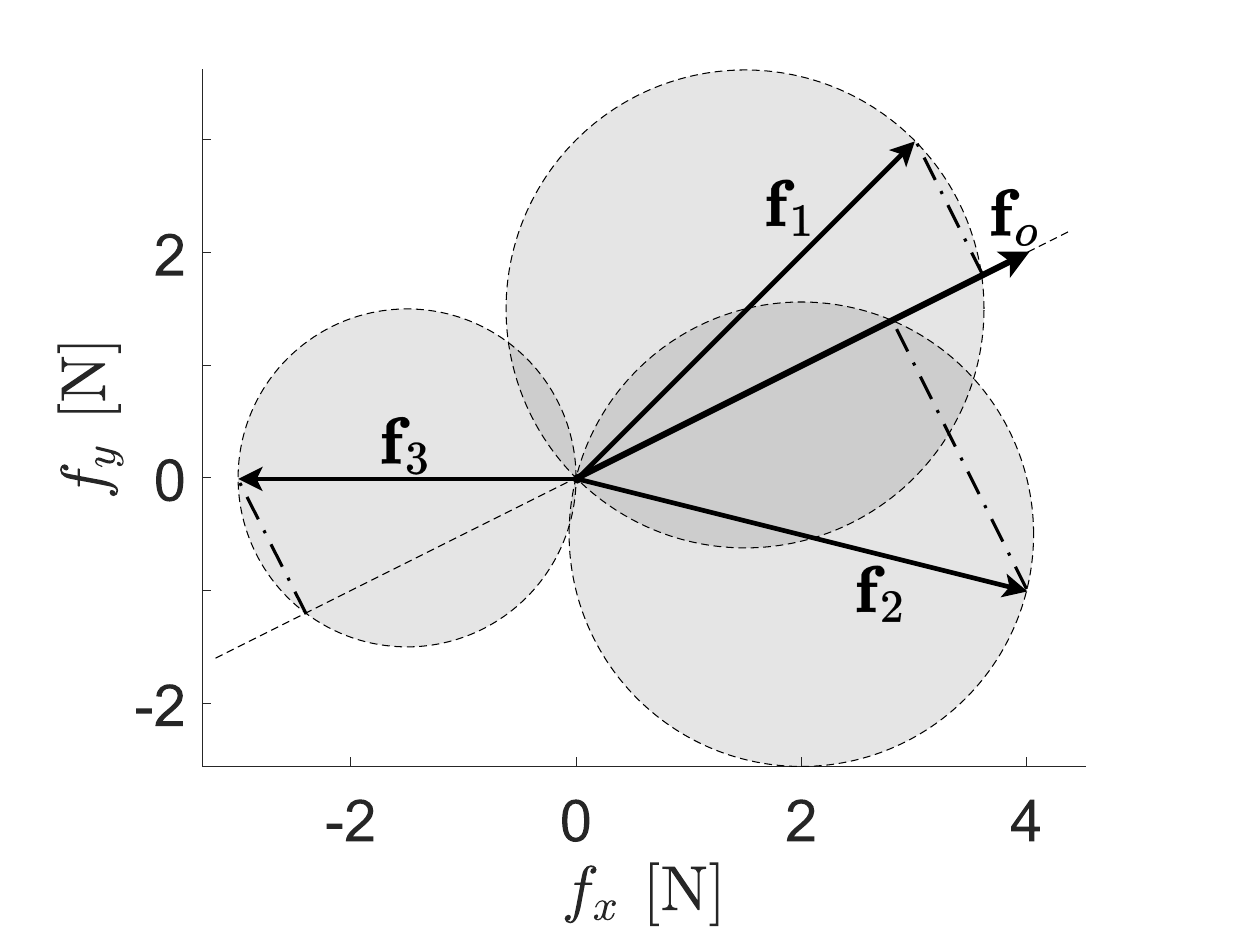}
    \caption{Three forces applied to a point mass.}
    \label{fig:Donneretal-ProblemSetup}
\end{figure}

\textcolor{black}{The three forces $\mb{f}_i$ sum to $\mb{f}_o$. The physical feasibility criteria derived in} \cite{Donner2018} \textcolor{black}{constrain both the manipulating and constraint components to lie within the circles shown in Figure} \ref{fig:Donneretal-ProblemSetup}. \textcolor{black}{The authors also state that the manipulating force component $\mb{f}_{m,i}$ of an applied force $\mb{f}_i$ cannot be larger than its projection on $\mb{f}_o$ (i.e., $\mb{f}_{m,i}^T \mb{f}_{m,i} \leq \mb{f}_i^T \mb{f}_{m,i}$) and that this projection must be positive. Any components which are in the opposite direction to $\mb{f}_o$ or perpendicular to $\mb{f}_o$ are components of the constraint force $\mb{f}_{c,i}$. }

\textcolor{black}{This problem can easily be solved using our methods. The force equilibrium equation is written as }

\begin{equation}
    \mb{f}_o = \mb{G} \mb{f} = \begin{bmatrix}
    \mb{I}_2 & \mb{I}_2 & \mb{I}_3
\end{bmatrix} \begin{bmatrix}
    \mb{f}_1 \\ \mb{f}_2 \\ \mb{f}_3
\end{bmatrix}.
\end{equation}

\textcolor{black}{Manipulating wrench distributions can then be obtained with}

\begin{equation}
    \mb{f}_m = \mb{G}^+_M \mb{f}_o = \frac{1}{m_o}\begin{bmatrix}
        m^*_1 \mb{I}_2 \\
        m^*_2 \mb{I}_2 \\
        m^*_3 \mb{I}_2
    \end{bmatrix}\mb{f}_o.
\end{equation}

\textcolor{black}{A dynamically equivalent system for the point mass $m^*_o$ is simply three smaller point masses $m^*_i$ that are located at the same point. By choosing admissible solutions for the virtual masses $m^*_i$, we can determine the set of manipulating wrench distributions. We need only consider constraint }(\ref{eq:virtualMassSum}), \textcolor{black}{since the other constraints are trivially satisfied. To keep consistent with the notation used thus far, we write the constraint as}

\begin{equation}
    \mb{R} \mb{m}^* = \begin{bmatrix}
        1 & 1 & 1 
    \end{bmatrix} \begin{bmatrix}
        m^*_1 \\ m^*_2 \\ m^*_3
    \end{bmatrix} = m^*_o
\end{equation}
\textcolor{black}{where $m^*_o$ is the virtual mass assigned to the point mass. }

\textcolor{black}{Matrix $\mb{R}$ is $1 \times 3$ and thus has a null-space of dimension 2. We therefore have two degrees of freedom within the space of admissible solutions for $\mb{m}^*$ that can be used to parametrize the feasible set of manipulating wrench distributions ($d = \text{dim}(\mathcal{N}(\mb{R})) = 2$). We choose $m^*_1$ and $m^*_2$ as our free variables and the value of $m^*_3$ is then determined to satisfy the mass sum constraint. The feasible set for $m^*_1$ and $m^*_2$ can be expressed as}

\begin{equation}
\begin{split}
    \mathcal{M} = \{ \begin{bmatrix}
        m^*_1 & m^*_2
    \end{bmatrix}^T \in & \mathbb{R}^2 \mid m^*_1 \geq 0, \\ &m^*_2 \geq 0, m^*_1 + m^*_2 \leq m^*_o \}.
\end{split}
\end{equation}

Figure \ref{fig:Donneretal-Results} \textcolor{black}{shows two of the manipulating wrench distributions given by the }\cite{Donner2018} \textcolor{black}{for this example. We consider both to be valid, but in }\cite{Donner2018}, \textcolor{black}{only the solution in Figure }\ref{subfig:Donneretal-Results1} \textcolor{black}{is considered valid according to their assumptions. }

\begin{figure}[t!]
\centering
    \subfloat[Manipulating force distribution with $\mb{m}^* = \begin{bmatrix}
        2.25 & 0.75 & 0
    \end{bmatrix}^T$.]{%
        \includegraphics[width=.7\linewidth]{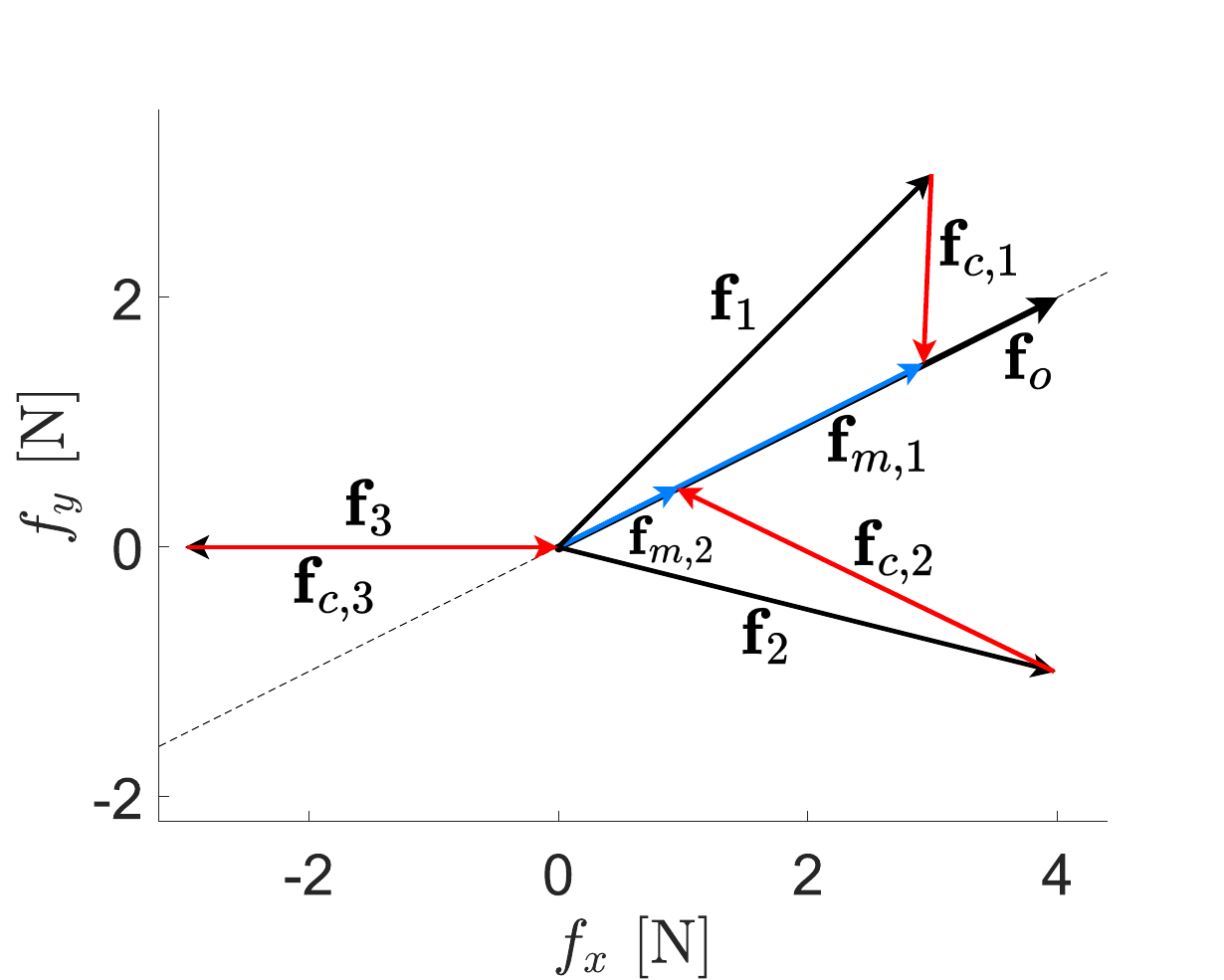}%
        \label{subfig:Donneretal-Results1}%
    } \\
    \subfloat[Manipulating force distribution with $\mb{m}^* = \begin{bmatrix}
        1 & 1 & 1
    \end{bmatrix}^T$.]{%
        \includegraphics[width=.7\linewidth]{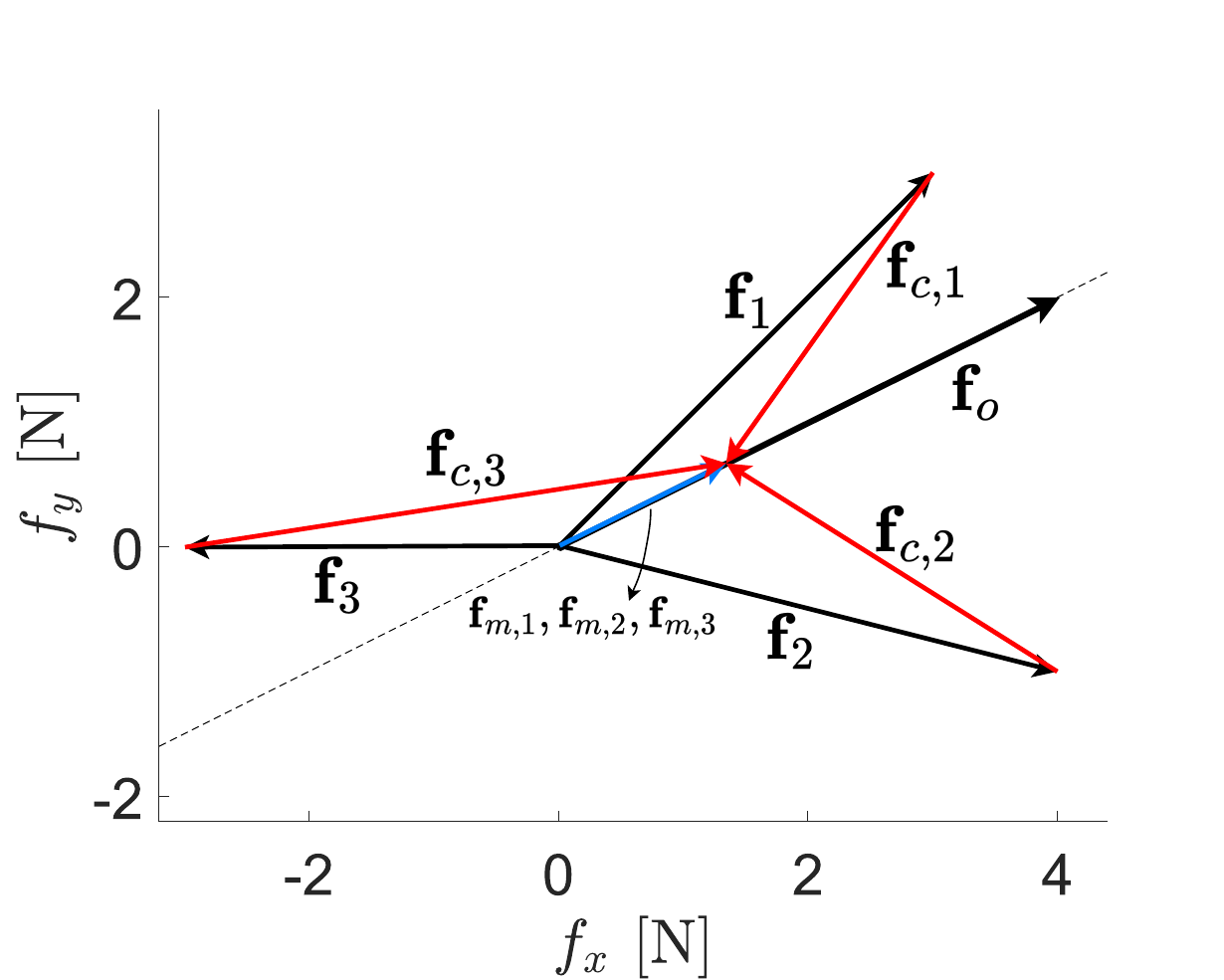}%
        \label{subfig:Donneretal-Results2}%
    }
    \caption{Two manipulating force distributions for three planar forces applied to a point mass.}
    \label{fig:Donneretal-Results}
\end{figure}

\textcolor{black}{The resultant force $\mb{f}_o$ induces instantaneous acceleration $\ddot{\mb{p}}^*_o$ on the point mass. The three LMIEs of mass $m^*_i$ each have instantaneous acceleration $\ddot{\mb{p}}_i = \ddot{\mb{p}}_o$. Then, each manipulating force can be found with $\mb{f}_{m,i} = m^*_i \ddot{\mb{p}}_i$. Any specific solution depends on our choice of virtual mass parameters. The force distribution in Figure} \ref{subfig:Donneretal-Results1} \textcolor{black}{corresponds to solution $\mb{m^*} = \begin{bmatrix}
    2.25 & 0.75 & 0
\end{bmatrix}$ and the distribution in Figure} \ref{subfig:Donneretal-Results2} \textcolor{black}{to $\mb{m^*} = \begin{bmatrix}
    1 & 1 & 1
\end{bmatrix}$, which is the minimum-norm solution. }

\cite{Donner2018} \textcolor{black}{only consider the solution in Figure }\ref{subfig:Donneretal-Results1} \textcolor{black}{to be valid since it is the only one which is consistent with their assumptions concerning physical plausibility. First, they state that $\mb{f}_3$ cannot have a manipulating force component, since it has no positive component along $\mb{f}_o$. However, in Figure} \ref{subfig:Donneretal-Results2}, \textcolor{black}{$\mb{f}_3$ is assigned a manipulating force. Second, the solution in Figure }\ref{subfig:Donneretal-Results2} \textcolor{black}{violates their physical plausibility criteria because the constraint forces extend beyond the feasible region defined by the circles in Figure }\ref{fig:Donneretal-ProblemSetup}. \textcolor{black}{It is our position that these arguments do not negate the validity of the manipulating wrench distribution shown in Figure} \ref{subfig:Donneretal-Results2}, \textcolor{black}{since this solution is consistent with the Udwadia-Kalaba equation. }

\textcolor{black}{We now demonstrate how our proposed framework resolves the ambiguity surrounding the internal loading state of the object when the manipulating wrench distribution admits infinitely many solutions. Despite the constraint force vector $\mb{f}_c$ being dependent on our choice of solution for $\mb{h}_m$, the internal loading state of the point mass can be uniquely defined by a single vector. We have already chosen $m^*_1$ and $m^*_2$ as our free variables. Following the convention proposed in Section }\ref{sec:internalLoadingState}, \textcolor{black}{we choose the manipulating force distribution that corresponds to the minimum-norm solution $\mb{m}^* = \begin{bmatrix}
    1 & 1 & 1
\end{bmatrix}$ (i.e., $\mb{f}_{m,i} = \frac{1}{3}\mb{f}_o \ \forall i$) as the particular solution $\mb{f}_{m,p}$. }

\textcolor{black}{The basis vectors which span $\mathcal{N}_m$ can be written as }

\begin{equation}
    \mb{k}_1 = \begin{bmatrix}
        \mb{u}_o \\ \mb{0}_{2 \times 1} \\ -\mb{u}_o
    \end{bmatrix}, \quad \mb{k}_2 = \begin{bmatrix}
        \mb{0}_{2 \times 1} \\ \mb{u}_o \\ -\mb{u}_o
    \end{bmatrix}.
\end{equation}
\textcolor{black}{where $\mb{u}_o$ is a unit vector in the direction of $\mb{f}_o$. $\mb{k}_1$ and $\mb{k}_2$ describe the effect on the manipulating force vector $\mb{f}_m$ of changing $m^*_1$ and $m^*_2$, respectively. The general equation for the manipulating force distribution for this case is then}

\begin{equation}
    \mb{f}_m = \mb{f}_{m,p} + \begin{bmatrix}
        \mb{k}_1 & \mb{k}_2
    \end{bmatrix} \begin{bmatrix}
        m^*_1 \\ m^*_2
    \end{bmatrix}
\end{equation}
\textcolor{black}{where, again, $\mb{f}_{m,p}$ is the particular solution which corresponds to $\mb{m}^* = \begin{bmatrix}
    1 & 1 & 1
\end{bmatrix}$ and $\begin{bmatrix}
    m^*_1 & m^*_2
\end{bmatrix}^T$ must belong to the feasible set $\mathcal{M}$.}

\textcolor{black}{Finally, we obtain }

\begin{equation}
    \mb{z}_1 = \begin{bmatrix}
        0.116 \\ -0.231 \\ -0.358 \\ 0.716 \\ 0.242 \\ -0.485
    \end{bmatrix}, \quad \mb{z}_2 = \begin{bmatrix}
        0.346 \\ -0.693 \\ -0.073 \\ 0.146 \\ -0.273 \\ 0.546
    \end{bmatrix}
\end{equation}
\textcolor{black}{as the two basis vectors which span the null-space of matrix $\begin{bmatrix}
    \mb{G}^T & \mb{k}_1 & \mb{k}_2
\end{bmatrix}^T$. Using }(\ref{eq:lambda_c}) \textcolor{black}{with $\mb{Z} = \begin{bmatrix}
    \mb{z}_1 & \mb{z}_2
\end{bmatrix}$, we find that every constraint force vector $\mb{f}_c$ in $\mathcal{N}_c$ evaluates to the same vector} 

\begin{equation}
    \boldsymbol{\lambda}_c  = \begin{bmatrix}
        3.220 \\ 0.659
    \end{bmatrix},
\end{equation}
\textcolor{black}{which unambiguously describes the internal loading state of the rigid body. }

\subsection{Planar Body}
Let us consider a planar rigid body in the shape of an equilateral triangle. A pure force $\mb{f}_i = \begin{bmatrix}
    f_{i,x} & f_{i,y}
\end{bmatrix}^T$ is applied at each vertex, and a pure torque $\tau$ is applied at the centroid. For simplicity, we consider each vertex to be $r = 1$ m from the CoM (see Figure \ref{fig:2DManipExample}).

\begin{figure}
    \centering
    \includegraphics[width=0.7\linewidth]{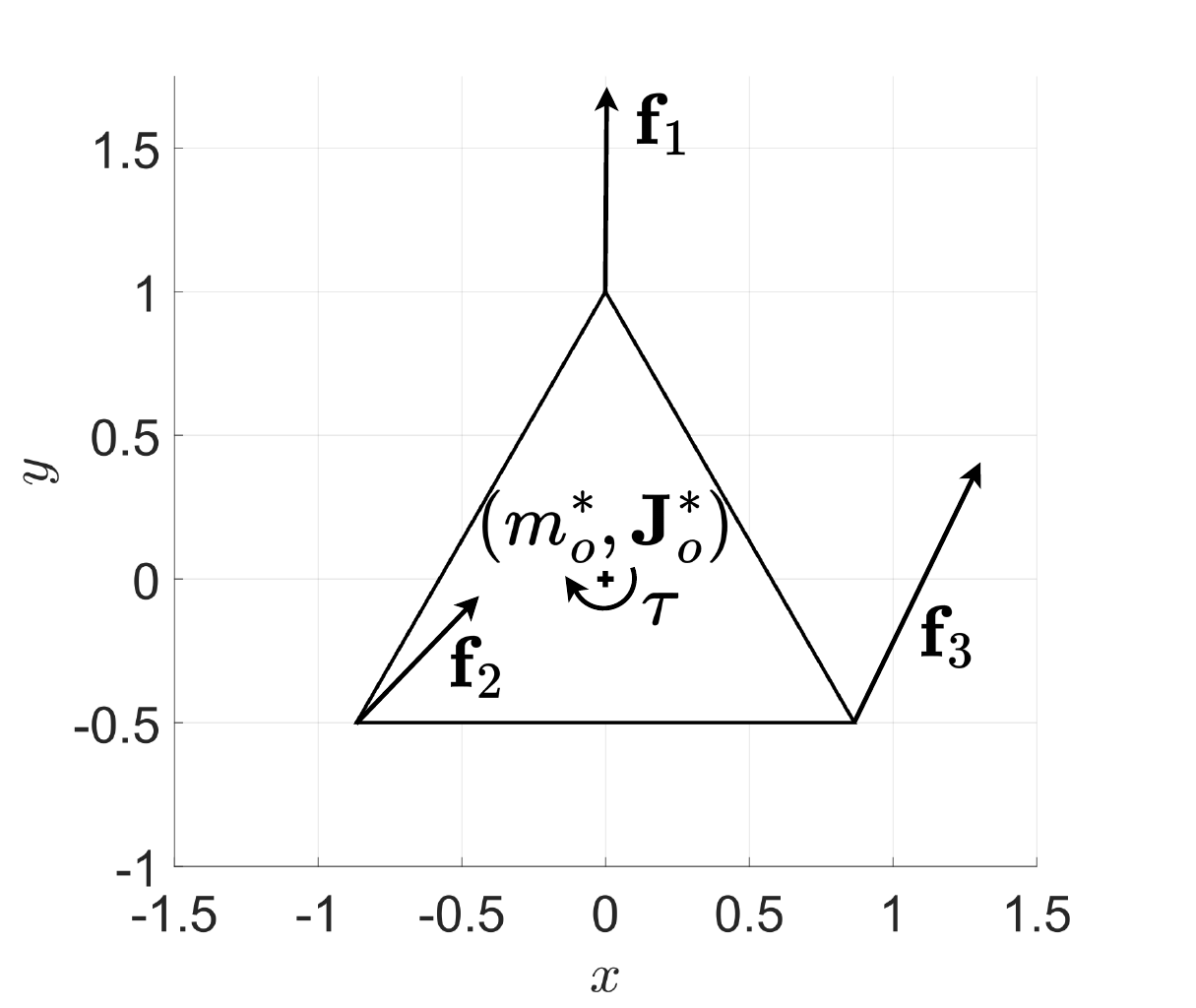}
    \caption{Three forces and one moment applied to a planar rigid body.}
    \label{fig:2DManipExample}
\end{figure}

This example provides an opportunity to better visualize the weighting between the force and torque components with the virtual inertias $m^*_i$ and $\mb{J}^*_i$ and its effect on the helicoidal virtual acceleration field. We may construct an dynamically equivalent system of this rigid body by defining a LMIE of mass $m^*_i$ for $i=1, 2, 3$ at each vertex and an additional LMIE with inertia tensor $\mb{J}^*_4$ at the CoM. 

We consider three cases for a desired resultant wrench $\mb{h}_o = \begin{bmatrix}
    f_{ox} & f_{oy} & 0 & 0 & 0 & t_o
\end{bmatrix}^T = \begin{bmatrix}
    1 & 2 & 0 & 0 & 0 & -2
\end{bmatrix}^T$ where we assign either 0\%, 50\%, or 100\% of the total desired wrench $t_o = -2$ Nm to the manipulating torque $\tau_m$ applied to the rigid body. We choose $m^*_o = 3$ kg and obtain $\mb{m}^* = \begin{bmatrix}
    1 & 1 & 1
\end{bmatrix}^T \text{kg}$ as the unique solution for the virtual masses of the LMIEs. The virtual mass of the LMIE at the CoM $m^*_4$ must be zero since only a pure torque is applied there. 

In the first case where none of $t_o$ is assigned to $\tau_m$, the resultant torque $t_o$ is entirely generated by the forces. Hence, the inertia of the object must come entirely from the virtual masses $m^*_i$. We therefore set $\mb{J}^*_o = \sum_{i=1}^3 \mb{S}(\mb{r}_i) m^*_i \mb{S}(\mb{r}_i)^T$ and obtain the solution shown in Figure \ref{subfig:PlanarAccelerationFields-Page-1}. Note that $\mb{f}_{m,i} = \ddot{\mb{p}}^*_i \ \forall i$. 

In the second case where half of $t_o$ is assigned to $\tau_m$, the solution for the virtual masses remains the same, but now the inertia of the rigid body must be shared by the virtual masses and the inertia tensor $\mb{J}^*_4$. To do so, we set $\mb{J}^*_o = 2\sum_{i=1}^3 \mb{S}(\mb{r}_i) m^*_i \mb{S}(\mb{r}_i)^T$ and $\mb{J}^*_4 = \sum_{i=1}^3 \mb{S}(\mb{r}_i) m^*_i \mb{S}(\mb{r}_i)^T$. This creates an equal weighting between the manipulating torque $\tau_m$ and the manipulating force-induced torque $\tau_{fm}$. Note that $\mb{J}^*_o$ being larger than in the first case is permitted due to the scaling invariance (see Section \ref{sec:ScaleInvarianceImplication}). We obtain the solution shown in Figure \ref{subfig:PlanarAccelerationFields-Page-2}.

In the third case where all of $t_o$ is assigned to $\tau_m$, $\mb{J}^*_4$ must represent the entirety of the inertia of the object. Since the inertia contribution from the masses is fixed to be $\sum_{i=1}^3 \mb{S}(\mb{r}_i) m^*_i \mb{S}(\mb{r}_i)^T$, we can simply let $\mb{J}^*_4$ and $\mb{J}^*_o$ approach infinity such that the contribution from the system of point masses is negligible. The solution is shown in Figure \ref{subfig:PlanarAccelerationFields-Page-3}. 

The field lines of the helicoidal acceleration field $\ddot{\mb{x}}_f$ have maximum curvature when all of the resultant torque $t_o$ is assigned to the forces (Figure \ref{subfig:PlanarAccelerationFields-Page-1}). Conversely, the field lines become parallel when all of $t_o$ is assigned to $\tau_m$ because $\ddot{\mb{x}}_f$ then becomes the acceleration induced by a pure force and all the points have the same acceleration $(\ddot{\mb{p}}_i = \ddot{\mb{p}}_o \ \forall i)$.

\begin{figure}[t!]
\centering
    \subfloat[$\tau_m = 0 \ \text{Nm}$.]{%
        \includegraphics[width=.7\linewidth]{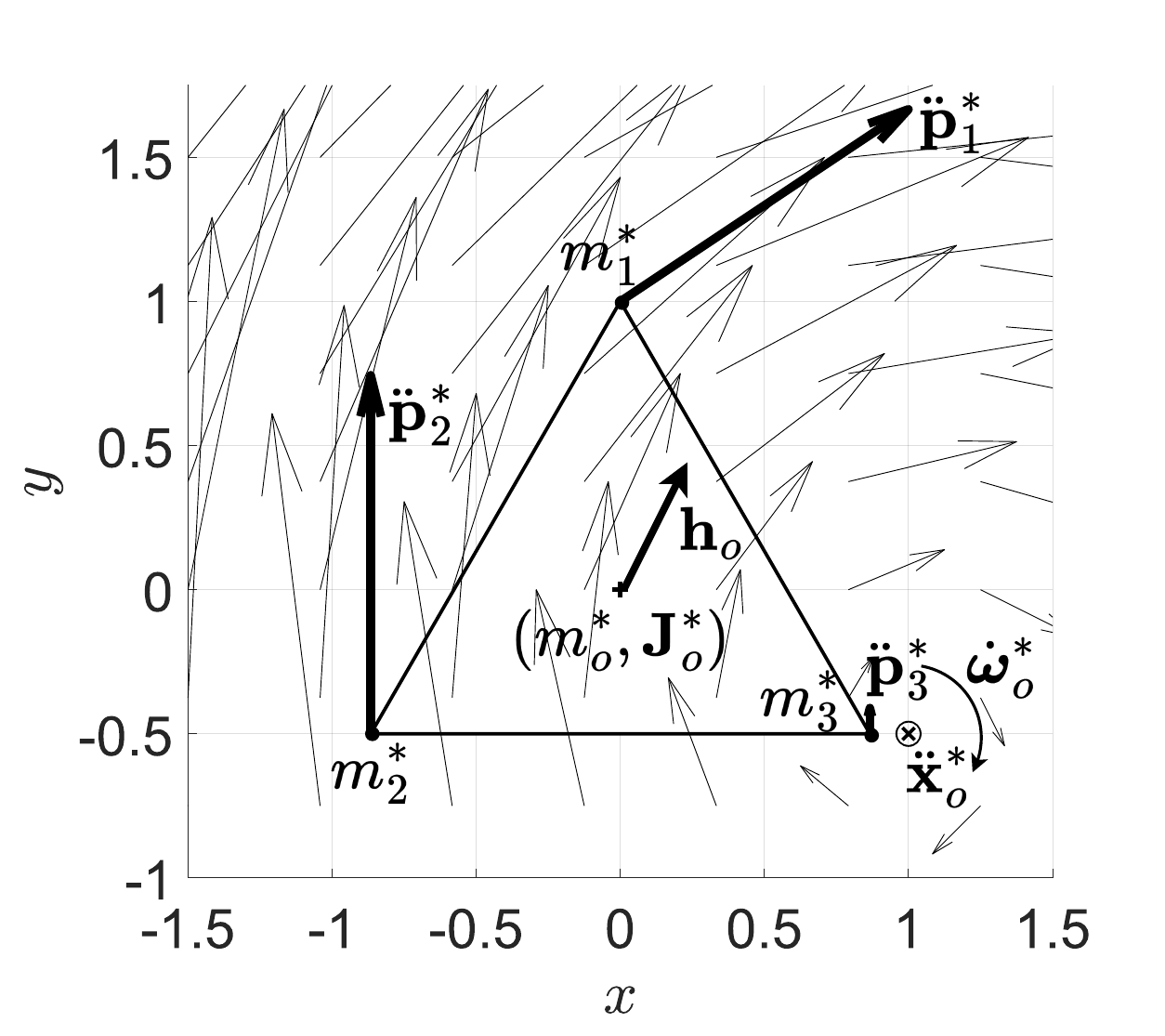}%
        \label{subfig:PlanarAccelerationFields-Page-1}%
    } \\
    \subfloat[$\tau_m = 0.5 \tau_o = -1 \ \text{Nm}$.]{%
        \includegraphics[width=.7\linewidth]{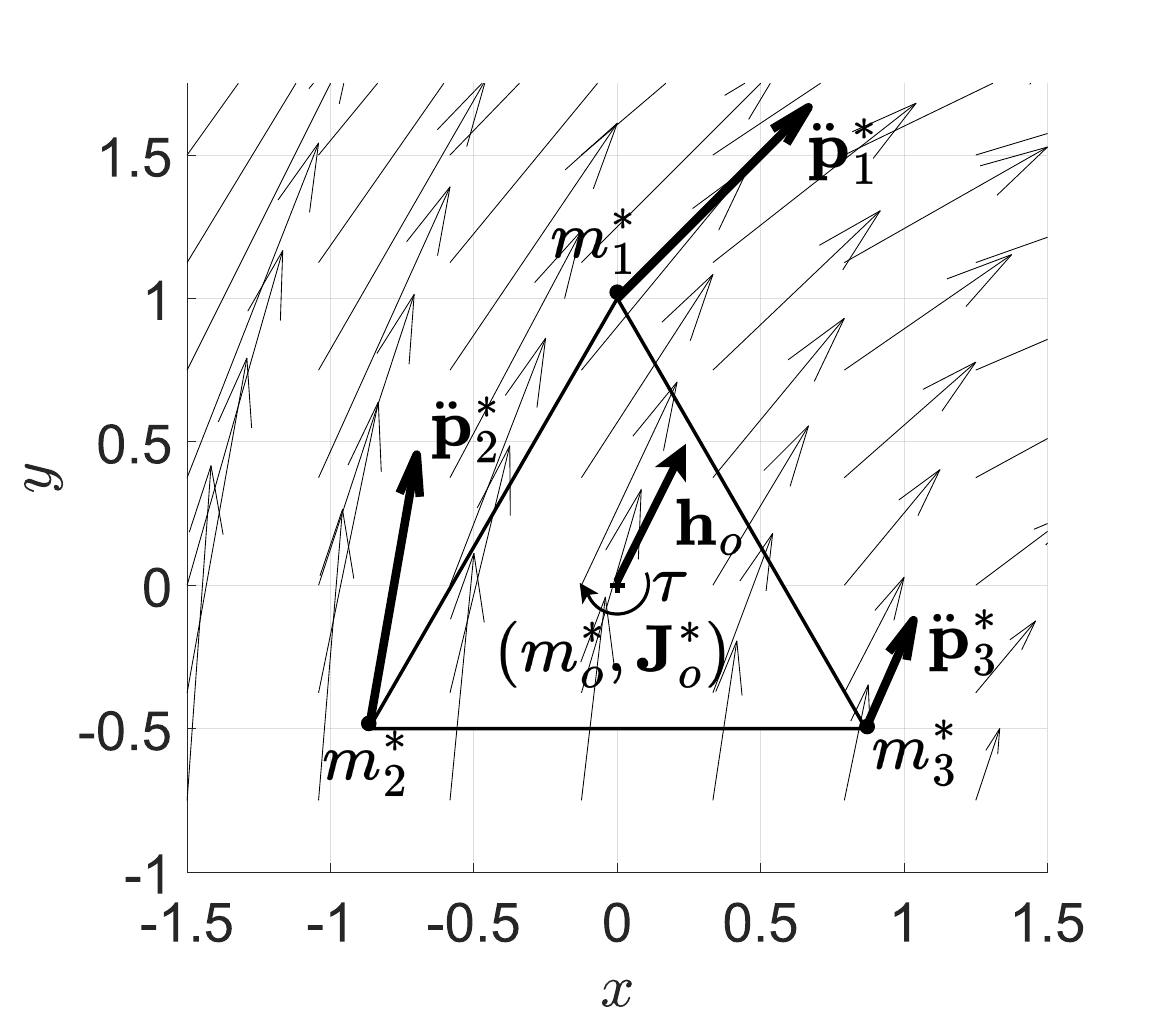}%
        \label{subfig:PlanarAccelerationFields-Page-2}%
    }\\
    \centering
    \subfloat[$\tau_m = \tau_o = -2 \ \text{Nm}$.]{%
        \includegraphics[width=.7\linewidth]{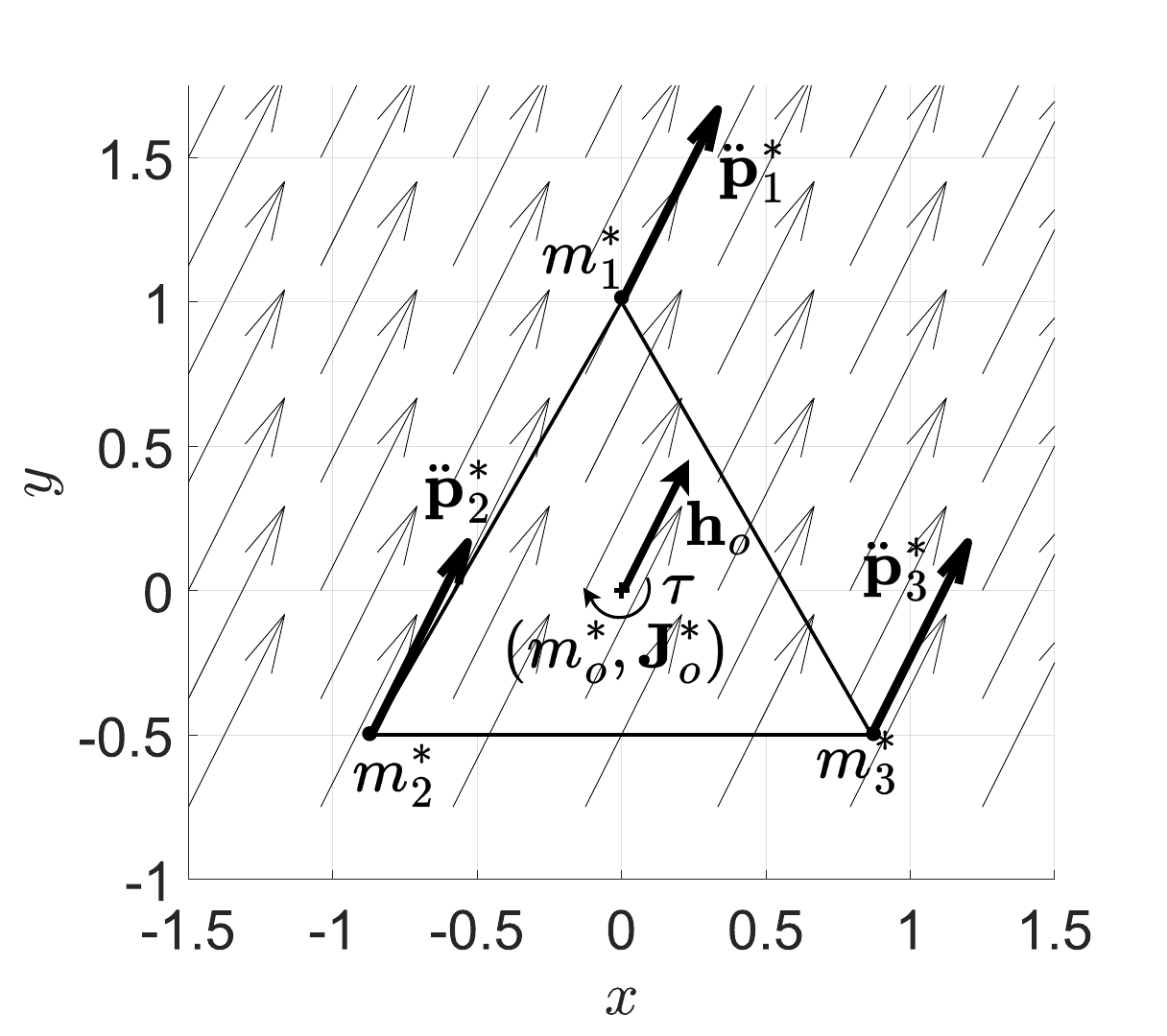}%
        \label{subfig:PlanarAccelerationFields-Page-3}%
    }
    \caption{Vector fields obtained from $\ddot{\mb{x}}^*_{f}$ considering different contributions from the manipulating torque $\tau_m$.}
    \label{fig:PlanarExample}
\end{figure}

\subsection{3-Dimensional Body}
For this example, we perform wrench decomposition (which also includes a wrench synthesis step) for the most general case where a 3-dimensional object is manipulated by the kinematic chains. Similar examples are often seen in robotics applications, either in the context of multifingered grasping or cooperative manipulation by multiple serial robots. We show how to determine the forces only solution and how to distribute the resultant torque among the pure torques that may be applied. 

As explained in Section \ref{sec:uniqueness}, if we wish to match the inertial properties of the object exactly, a total of 10 independent forces are required. As this is not practical, we consider the reduced case where we only consider constraints (\ref{eq:virtualMassSum}) and (\ref{eq:virtualCoMEquivalence}) and only four independent wrenches are necessary. 

Four wrenches $\mb{h}_i$ are applied to a unit sphere at locations defined by vectors $\mb{r}_i$ written as

\begin{equation}
    \mb{r}_i(\alpha_i, \gamma_i) = \begin{bmatrix}
        \cos \alpha_i \cos \gamma_i \\ \sin \alpha_i \cos \gamma_i \\ \sin \gamma_i
    \end{bmatrix} \text{m},
\end{equation}
where $\alpha_i$ and $\gamma_i$ are the azimuth and elevation angles that define a point on the surface of the sphere. The four wrenches are applied at coordinates $(0\degree, 65\degree)$, $(-60\degree, -25\degree)$, $(65\degree, -45\degree)$, and $(180\degree, 0\degree)$. The example is illustrated in Figure \ref{fig:sphereExample}. 

\begin{figure}
    \centering
    \includegraphics[width=0.7\linewidth]{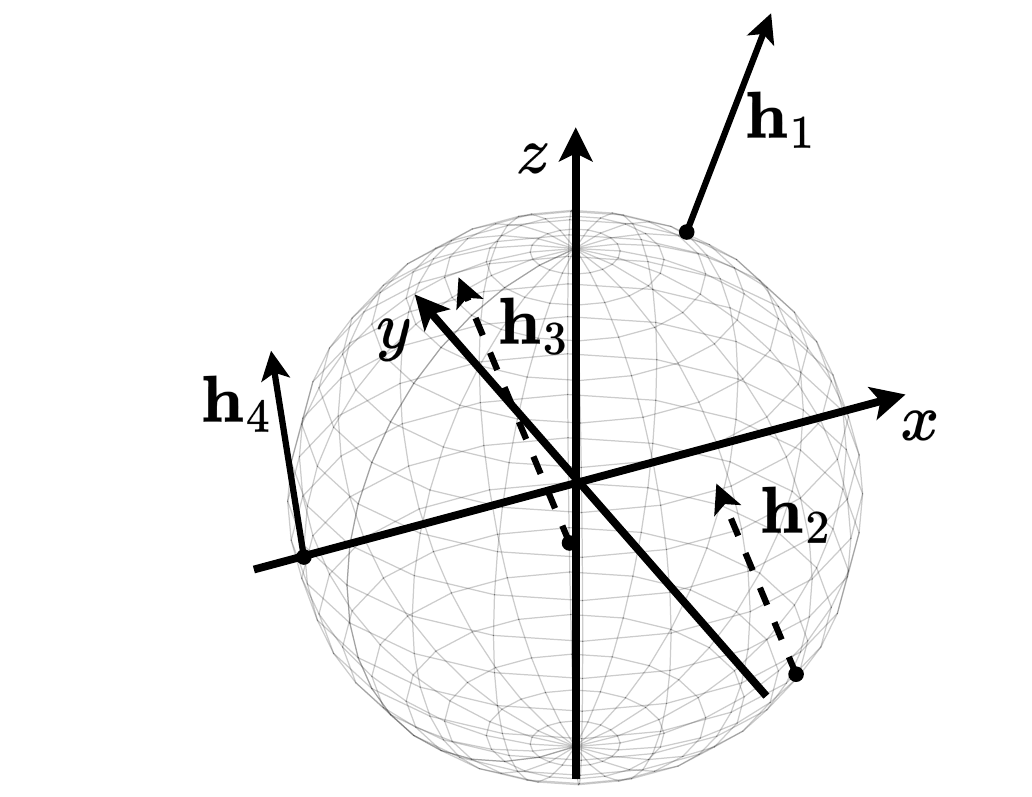}
    \caption{Unit sphere with four applied wrenches. A dashed line indicates that the wrench is applied on the back-side of the sphere.}
    \label{fig:sphereExample}
\end{figure}

We choose to set the applied wrenches arbitrarily as

\begin{equation}
\begin{aligned}
    \mb{h}_1 = \begin{bmatrix}
        1 \\ 0.5 \\ 1 \\ 0 \\ 0.5 \\ 0.5
    \end{bmatrix},& \ \mb{h}_2 = \begin{bmatrix}
        0 \\ 0 \\ 1 \\ 1 \\ 0 \\ 0
    \end{bmatrix} \\ \ \mb{h}_3 = \begin{bmatrix}
        -0.25 \\ 0.75 \\ 0 \\ 0 \\ -0.5 \\ 0.5
    \end{bmatrix},& \ \mb{h}_4 = \begin{bmatrix}
        0.5 \\ -0.25 \\ 1 \\ -0.5 \\ 0.75 \\ 0
    \end{bmatrix}.
\end{aligned}
\end{equation}
The resultant wrench is then $\mb{h}_o = \begin{bmatrix}
    1.250 & 1 & 3 & -0.208 & 1.957 & 1.846
\end{bmatrix}^T$. From here, the wrench decomposition problem is reduced to finding an admissible solution for the manipulating wrench distribution $\mb{h}_m$ for the resultant wrench (i.e., performing wrench synthesis), then using (\ref{eq:constraintWrenchAnalysis}) to calculate the constraint wrenches. $\mb{h}_{m,p}$ is chosen as the unique forces-only solution $\Tilde{\mb{h}}_m$ that corresponds to the vector $\mb{m}^*$ that solves (\ref{eq:VirtualMassSolutionReduced}). We choose $m^*_o = 4$ kg and obtain the solution $\mb{m}^* = \begin{bmatrix}
    1.123 & 0.790 & 0.967 & 1.121
\end{bmatrix}^T$. Since we wish to determine the manipulating wrench distribution which consists of only pure forces, the inertia of the system of masses $m^*_i$ must represent the total inertia of the dynamically equivalent system of the virtual rigid body. $\mb{J}^*_o$ is therefore obtained as

\begin{equation}
\begin{split}
    \mb{J}^*_o &=\sum_{i=1}^4 \mb{S}(\mb{r}_i) m^*_i \mb{S}(\mb{r}_i)^T \\ &= \begin{bmatrix}
        2.430 & 0.096 & -0.074 \\ 0.096 & 3.117 & 0.176 \\ -0.074 & 0.176 & 2.454
    \end{bmatrix}.
\end{split}
    \label{eq:virtualSphereInertia}
\end{equation}
We obtain the virtual instantaneous acceleration $\ddot{\mb{x}}^*_o = \begin{bmatrix}
    0.313 & 0.250 & 0.750 & -0.087 & 0.591 & 0.707
\end{bmatrix}^T$ and the solution is 

\begin{equation}
    \begin{aligned}
        \Tilde{\mb{h}}_{m,1} = \begin{bmatrix}
            0.952 \\ 0.705 \\ 0.562 \\ \mb{0}_{3 \times 1}
        \end{bmatrix}, \ & \Tilde{\mb{h}}_{m,2} = \begin{bmatrix}
            0.488 \\ 0.421 \\ 0.435 \\ \mb{0}_{3 \times 1}
        \end{bmatrix}, \\ \Tilde{\mb{h}}_{m,3} = \begin{bmatrix}
            -0.540 \\ 0.387 \\ 0.501 \\ \mb{0}_{3 \times 1}
        \end{bmatrix}, \ & \Tilde{\mb{h}}_{m,4} = \begin{bmatrix}
            0.350 \\ -0.513 \\ 1.503 \\ \mb{0}_{3 \times 1}
        \end{bmatrix}.
    \end{aligned}
    \label{eq:sphereAnalysisForcesOnly}
\end{equation}

The constraint wrench components that correspond to this manipulating wrench distribution can be found with (\ref{eq:constraintWrenchAnalysis}). We obtain 

\begin{equation}
\begin{aligned}
    \mb{h}_{c,1} = \begin{bmatrix}
        -0.048 \\ 0.205 \\ -0.438 \\ 0 \\ -0.500 \\ -0.500
    \end{bmatrix},& \ \mb{h}_{c,2} = \begin{bmatrix}
        0.488 \\ 0.421 \\ -0.565 \\ -1.000 \\ 0 \\ 0 
    \end{bmatrix}, \\
    \mb{h}_{c,3} = \begin{bmatrix}
        -0.290 \\ -0.364 \\ 0.501 \\ 0 \\ 0.500 \\ -0.500
    \end{bmatrix},& \ \mb{h}_{c,4} = \begin{bmatrix}
        -0.150 \\ -0.263 \\ 0.503 \\ 0.500 \\ -0.750 \\ 0 
    \end{bmatrix},
\end{aligned}
\label{eq:ForcesOnlyhc}
\end{equation}
which completes the wrench decomposition.

\textcolor{black}{Let us now compare these results with other existing methods in the literature. If we use Erhart and Hirche's method and calculate the forces-only manipulating wrench distribution with }(\ref{eq:BadParametrizedInverse}), \textcolor{black}{we obtain }

\begin{equation}
    \begin{aligned}
        \Tilde{\mb{h}}_{m,1} = \begin{bmatrix}
            1.145 \\ 0.627 \\ 0.463 \\ \mb{0}_{3 \times 1}
        \end{bmatrix}, \ & \Tilde{\mb{h}}_{m,2} = \begin{bmatrix}
            0.434 \\ 0.395 \\ 0.351 \\ \mb{0}_{3 \times 1}
        \end{bmatrix}, \\ \Tilde{\mb{h}}_{m,3} = \begin{bmatrix}
            -0.736 \\ 0.418 \\ 0.398 \\ \mb{0}_{3 \times 1}
        \end{bmatrix}, \ & \Tilde{\mb{h}}_{m,4} = \begin{bmatrix}
            0.408 \\ -0.439 \\ 1.789 \\ \mb{0}_{3 \times 1}
        \end{bmatrix}.
    \end{aligned}
    \label{eq:sphereAnalysisForcesOnlyEH}
\end{equation}
\textcolor{black}{However, substituting these wrenches into }(\ref{eq:staticEq}) \textcolor{black}{reveals that they generate a resultant wrench $\mb{h}_o = \begin{bmatrix}
    1.250 & 1 & 3 & -0.126 & 2.690 & 1.820
\end{bmatrix}$, which differs from the resultant wrench we prescribed. The solution in }(\ref{eq:sphereAnalysisForcesOnlyEH}) \textcolor{black}{is therefore not a valid manipulating wrench distribution. This is due to the incorrect form of $\mb{G}^+_M$ in }(\ref{eq:BadParametrizedInverse}), \textcolor{black}{which we have corrected in this paper. }

\textcolor{black}{Calculating the equilibrating wrench distribution using $\mb{G}^{\dagger}$, we obtain}

\begin{equation}
    \begin{aligned}
        \mb{f}_{e,1} = \begin{bmatrix}
            0.969 \\ 0.609 \\ 0.478 
        \end{bmatrix}, \ & \mb{f}_{e,2} = \begin{bmatrix}
            0.555 \\ 0.484 \\ 0.542 
        \end{bmatrix}, \\ \mb{f}_{e,3} = \begin{bmatrix}
            -0.602 \\ 0.350 \\ 0.496 
        \end{bmatrix}, \ & \mb{f}_{e,4} = \begin{bmatrix}
            0.328 \\ -0.443 \\ 1.484 
        \end{bmatrix}.
    \end{aligned}
    \label{eq:SphereEquilibratingForceDistribution}
\end{equation}
\textcolor{black}{This force distribution generates the correct resultant wrench while also satisfying the condition for no interaction forces }(\ref{eq:zeroInteractionForceCondition}). \textcolor{black}{However, these forces \emph{do not lie on a helicoidal field}. If instead we use the method proposed in }\cite{Kumar1988} \textcolor{black}{for calculating the parameters of the helicoidal field to obtain the equilibrating forces, we obtain the solution}

\begin{equation}
    \begin{aligned}
        \mb{f}_{e,1} = \begin{bmatrix}
            0.889 \\ 0.667 \\ 0.481 
        \end{bmatrix}, \ & \mb{f}_{e,2} = \begin{bmatrix}
            0.600 \\ 0.517 \\ 0.563 
        \end{bmatrix}, \\ \mb{f}_{e,3} = \begin{bmatrix}
            -0.592 \\ 0.370 \\ 0.477 
        \end{bmatrix}, \ & \mb{f}_{e,4} = \begin{bmatrix}
            0.313 \\ -0.459 \\ 1.386
        \end{bmatrix}.
    \end{aligned}
    \label{eq:KWHelicoidalFieldSolution}
\end{equation}
\textcolor{black}{This force distribution also satisfies }(\ref{eq:zeroInteractionForceCondition}), \textcolor{black}{but substituting into }(\ref{eq:staticEq}) \textcolor{black}{reveals that they generate wrench $\mb{h}_o = \begin{bmatrix}
    1.210 & 1.095 & 2.908 & -0.260 & 1.756 & 1.935
\end{bmatrix}^T$, which is again different from the prescribed resultant wrench. }

Returning to our example, if we instead assume that the object is handled by cooperating serial manipulators which can also apply torques, some or all of the resultant torque $\mb{t}_o$ can be assigned to the manipulating torques $\mb{t}_{m,i}$. Let us assume that all of $\mb{t}_o$ is assigned to the manipulating torques and that it is equally distributed among the four serial manipulators. The manipulating wrench distribution is then 

\begin{equation}
    \begin{aligned}
        \mb{h}_{m,1} = \begin{bmatrix}
            0.351 \\ 0.281 \\ 0.842 \\ -0.052 \\ 0.489 \\ 0.462
        \end{bmatrix}, \ & \mb{h}_{m,2} = \begin{bmatrix}
            0.247 \\ 0.197 \\ 0.592 \\ -0.052 \\ 0.489 \\ 0.462
        \end{bmatrix}, \\ \mb{h}_{m,3} = \begin{bmatrix}
            0.302 \\ 0.242 \\ 0.725 \\ -0.052 \\ 0.489 \\ 0.462
        \end{bmatrix}, \ & \mb{h}_{m,4} = \begin{bmatrix}
            0.350 \\ 0.280 \\ 0.841 \\ -0.052 \\ 0.489 \\ 0.462
        \end{bmatrix}.
    \end{aligned}
    \label{eq:SphereAnalysisEqualManipTorques}
\end{equation}
In this solution, the four manipulating torques are equal and parallel to $\mb{t}_o$. Since the manipulating forces must only generate the resultant force $\mb{f}_o$ and no torque, the helicoidal field defined by $\ddot{\mb{x}}^*_f$ is of infinite pitch and the field lines become parallel. The manipulating forces are therefore also parallel, but weighted by their respective virtual masses $m^*_i$. The constraint wrenches which correspond to this solution are 

\begin{equation}
\begin{aligned}
    \mb{h}_{c,1} = \begin{bmatrix}
        -0.649 \\ -0.219 \\ -0.158 \\ -0.052 \\ -0.011 \\ -0.039
    \end{bmatrix},& \ \mb{h}_{c,2} = \begin{bmatrix}
        0.247 \\ 0.197 \\ -0.408 \\ -1.052 \\ 0.489 \\ 0.461 
    \end{bmatrix} \\
    \mb{h}_{c,3} = \begin{bmatrix}
        0.552 \\ -0.508 \\ 0.725 \\ -0.052 \\ 0.989 \\ -0.039
    \end{bmatrix},& \ \mb{h}_{c,4} = \begin{bmatrix}
        -0.150 \\ 0.530 \\ -0.159 \\ 0.448 \\ -0.261 \\ 0.461 
    \end{bmatrix}.
\end{aligned}
\label{eq:EqualManipTorqueshc}
\end{equation}

Again, as described in Section \ref{sec:internalLoadingState}, the constraint wrenches shown in (\ref{eq:ForcesOnlyhc}), those in (\ref{eq:EqualManipTorqueshc}), and all other constraint wrench vectors which correspond to a valid manipulating wrench distribution for this set of applied wrenches all evaluate to the same vector $\boldsymbol{\lambda}_c$ which uniquely defines the location of vector $\mb{h}$ in sub-space $\mathcal{N}_c$.

\section{Simulation Experiment}
\label{sec:Simulation}
To validate our theory, we simulated four UR5 robots cooperatively manipulating a spherical payload using the PyBullet physics engine \cite{coumans2016pybullet} coupled with the Pinocchio Python library \cite{carpentier-sii19}.

\subsection{Experimental Setup}
We consider the payload to be a hollow sphere of radius $r_o = 0.2$ m, mass $m_o = 4$ kg and inertia tensor $\mb{J}_o = \frac{2}{3} m_o r_o^2 \mb{I}_3$. By selecting the vertices of a regular tetrahedron inscribed within the sphere, we are able to satisfy constraints (\ref{eq:virtualMassSum})-(\ref{eq:inertiaProportionality}) with only four manipulators. The azimuth and elevation angle pairs ($\alpha_i$, $\gamma_i$) that define the locations of the attachment points with respect to the payload's CoM are ($\pi / 4$, $1 / \sqrt{3}$), ($3\pi / 4$, $-1 / \sqrt{3}$), ($5\pi / 4$, $1 / \sqrt{3}$), and ($7\pi / 4$, $-1 / \sqrt{3}$). The bases of the UR5 robots are located at 

\begin{equation}
    \begin{aligned}
        \mb{d}_1 = \begin{bmatrix}
            0.4 \\ 0.4 \\ 0
        \end{bmatrix}\text{m}, \ \mb{d}_2 = \begin{bmatrix}
            -0.4 \\ 0.4 \\ 0
        \end{bmatrix} \text{m} \\
        \mb{d}_3 = \begin{bmatrix}
            -0.4 \\ -0.4 \\ 0
        \end{bmatrix}\text{m}, \ \mb{d}_4 = \begin{bmatrix}
            0.4 \\ -0.4 \\ 0
        \end{bmatrix}\text{m}.
    \end{aligned}
\end{equation}

We initialize the payload at an initial pose $\mb{x}_o^0 = \begin{bmatrix}
    \mb{p}_o^{0T} & \mb{q}_o^{0T}
\end{bmatrix}^T = \begin{bmatrix}
    0 & 0 & 0.6 & 0 & 0 & 0
\end{bmatrix}^T$ and prescribe a trajectory as well as an external wrench that the payload must sustain over the entire trajectory. The experimental setup is shown in Figure \ref{fig:experimentalSetup}. 

\begin{figure}
    \centering
    \includegraphics[width=0.8\linewidth]{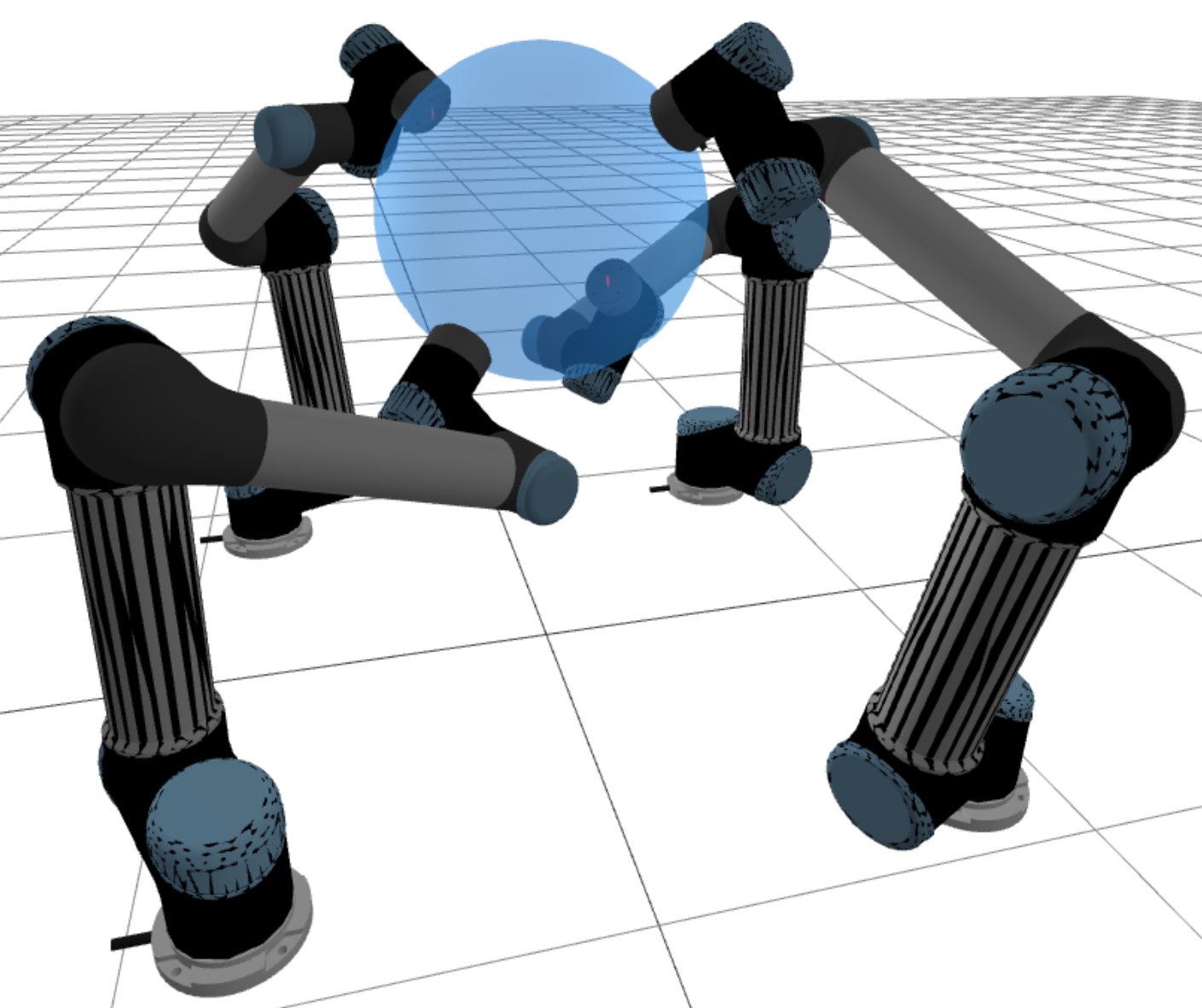}
    \caption{Simulation setup. Four UR5 robots cooperatively handle a common spherical payload.}
    \label{fig:experimentalSetup}
\end{figure}

For the trajectory, we choose a figure-8 Lissajous curve described by

\begin{equation}
    \mb{p}_o(t) = \mb{p}_o^0 + \begin{bmatrix}
        a_x \sin (\omega t) \\ \frac{1}{2}a_y \sin (2 \omega t) \\ a_z \sin (\omega t)
    \end{bmatrix}
    \label{eq:simTrajectory}
\end{equation}
where $a_x = 0.15$, $a_y = 0.15$, $a_z = 0.05$, and $\omega = 1$ rad/s. An external wrench $\mb{h}_{ext} = \begin{bmatrix}
    0 & 0 & 20 \ \text{N} & 0 & 0 & 50 \ \text{Nm}
\end{bmatrix}^T$ that the payload must apply to the environment is also prescribed.

The first and second analytical derivatives of (\ref{eq:simTrajectory}) allow us to obtain the instantaneous velocity $\dot{\mb{p}}_o$ and acceleration $\ddot{\mb{p}}_o$ of the payload at time $t$. The angular velocity and angular acceleration are zero throughout the entire trajectory. The end-effectors of the manipulators are constrained to always point toward the payload's CoM, which is also the geometric centre of the sphere.

For this simulation, the robots are controlled using the following force-control law: 

\begin{equation}
    \boldsymbol{\tau}_{c,i} = \boldsymbol{\tau}_{dyn,i} + \boldsymbol{\tau}_{ff,i}
\end{equation}
where $\boldsymbol{\tau}_{c,i}$ is the joint torque command sent to the $i$-th manipulator. $\boldsymbol{\tau}_{dyn,i}$ is the joint torque vector required to compensate the manipulator dynamics, written as

\begin{equation}
    \boldsymbol{\tau}_{dyn,i} = \mb{M}_{m,i}(\mb{\boldsymbol{\theta}}) \ddot{\boldsymbol{\theta}} + \mb{C}_{m,i}(\mb{\boldsymbol{\theta}}, \dot{\boldsymbol{\theta}}) \dot{\boldsymbol{\theta}} + \mb{g}_m(\boldsymbol{\theta}),
\end{equation}
and $\boldsymbol{\tau}_{ff,i}$ is the joint torque vector required to generate the feed-forward wrench $\mb{h}_{ff,i}$ written as

\begin{equation}
    \boldsymbol{\tau}_{ff,i} = \mb{J}_i^T \mb{h}_{ff,i} 
\end{equation}
where $\mb{J}_i$ is the kinematic Jacobian of the $i$-th manipulator. 

At each time step, we calculate the stacked vector of feed-forward wrenches $\mb{h}_{ff}$ that must be applied to the payload by the manipulators. This is based on the total wrench required to drive the payload, which is obtained as 

\begin{equation}
    \mb{h}_o = \mb{M}_o \ddot{\mb{x}}_o + \mb{g} + \mb{h}_{ext}.
    \label{eq:simulation_ho}
\end{equation}
Note that the Coriolis and centripetal term $\mb{C}_o \dot{\mb{x}}_o$ is omitted since the world frame is fixed and the orientation of the payload is kept constant. We initialize each manipulator at its starting configuration with the end-effector positioned at the payload attachment point $\mb{p}_i(t = 0)$.

Infinitely many solutions exist for $\mb{h}_{ff}$; the leading methodology in the literature is to use quadratic programming (QP) to iteratively converge on a locally optimal solution at every time-step. The general form of an optimization problem of this nature is written as the minimization of a quadratic cost function subject to equality and inequality constraints:

\begin{equation}
    \begin{aligned}
        \underset{\mb{h}_{ff}}{\min}& \qquad \frac{1}{2}\mathbf{h}_{ff}^T \mb{W}\mathbf{h}_{ff}\\ 
        \text{subject to}& \qquad \mathbf{h}_o = \mathbf{G} \mathbf{h}_{ff} \\ 
        & \qquad \mb{K} \mb{h}_{ff} \leq \mb{n}.
        \label{eq:QPdef}
    \end{aligned}
\end{equation}
where $\mb{W}$ is a positive-definite weighting matrix, and $\mb{K} \mb{h}_{ff} \leq \mb{n}$ represent the inequality constraints such as friction cones or actuator torque limits. Although included in the definition, we do not consider any of these inequalities in our implementation. Instead, we are only concerned with satisfying the wrench equilibrium equation.

In this experiment, we aim to analyze the internal loads generated by wrench synthesis methods frequently used in the literature and compare these results to the manipulating wrench distribution calculated using our proposed method. To quantify the internal forces generated by a given solution for the applied feed-forward wrench vector $\mb{h}_{ff}$, we decouple the motion of the payload from the motion of the manipulators. Instead of explicitly connecting the end-effectors of the manipulators to the payload directly in the physics engine, we use the prescribed trajectory (\ref{eq:simTrajectory}) to calculate the path that must be taken by each payload attachment point over the course of the trajectory. This serves as a ground truth to which we compare the displacement of the LMIEs induced by the feed-forward wrenches applied by the manipulators. If a given manipulator does not induce the proper constrained acceleration of its LMIE, the path taken by the LMIE will deviate from the desired path obtained from the trajectory, and a constraint wrench is needed to correct this deviation. 

To correct any errors in the motion induced by the feed-forward wrenches, we employ impedance control via a simple PD controller. This ensures that the end-effectors of the UR5 robots follow the correct path by compensating for any drift. This correcting force is given by

\begin{multline}
    \mb{f}_{pd,i} = K_p (\mb{p}_i(t) - \mb{p}_{ee,i}(t)) + \\
    K_d (\dot{\mb{p}}_i(t) - \dot{\mb{p}}_{ee,i}(t))
\end{multline}
where $\mb{p}_i(t)$ and $\dot{\mb{p}}_i(t)$ are the position and velocity of the $i$-th attachment point at time $t$, respectively, and $\mb{p}_{ee,i}(t)$ and $\dot{\mb{p}}_{ee,i}(t)$ are the position and velocity of the $i$-th UR5 end-effector carrying its LMIE of mass $m^*_i$ at time $t$, respectively. $K_p$ and $K_d$ are the proportional and derivative gains.

The unique forces-only solution was chosen as the manipulating wrench distribution when using our method for wrench synthesis. Conversely, when using the QP methods, the result often yields full 6D wrenches to be applied by the manipulators which also have pure torque components. These applied torques $\mb{t}_i$ also contribute to the manipulating wrench distribution, and this contribution must be subtracted from $\mb{h}_o$ when calculating the virtual acceleration field that defines the virtual accelerations $\ddot{\mb{p}}^*_i$ of the virtual masses $m^*_i$. The manipulating torque component of the applied torques is written as

\begin{equation}
    \sum_{i=1}^n\mb{t}_{m,i} = \left( \left( \sum^n_{i=1} \mb{t}_i \right) \cdot \hat{\mb{t}}_o\right) \cdot \hat{\mb{t}}_o
    \label{eq:AppliedTorqueProjection}
\end{equation}
where vectors $\mb{t}_i \ \forall i$ are the torque components of the feedforward wrenches $\mb{h}_{ff,i}$ calculated with the QP method and $\hat{\mb{t}}_o$ is a unit vector in the direction of $\mb{t}_o$. Equation (\ref{eq:AppliedTorqueProjection}) serves to project the sum of the applied torques onto the direction of the resultant torque $\mb{t}_o$.

The 6D vector $\ddot{\mb{x}}^*_f$ that defines the virtual accelerations $\ddot{\mb{p}}^*_i$ is then obtained with

\begin{equation}
    \ddot{\mb{x}}^*_f = \begin{bmatrix}
        \frac{1}{m_o} \mb{I}_3 & \mb{0}_3 \\
        \mb{0}_3 & \mb{J}^{-1}_o
    \end{bmatrix} \begin{bmatrix}
        \mb{f}_o \\ \mb{t}_o - \sum^n_{i=1} \mb{t}_{m,i}
    \end{bmatrix}.
    \label{eq:ForceAccelField}
\end{equation}
This enables us to calculate the constrained accelerations that must be imparted by the manipulating forces even when torques are applied by the manipulators. The actual accelerations induced by the force components of solutions to (\ref{eq:QPdef}) obtained with the QP method are compared against these constrained accelerations to calculate the theoretical internal forces. 

\subsection{Results}
We choose the forces-only solution for the manipulating wrench distribution. Setting $\mb{J}^*_i = \mb{0}_3$ and solving (\ref{eq:VirtualMassSolutionReduced}), we obtain the vector of virtual masses $\mb{m}^* = \begin{bmatrix}
    1 & 1 & 1 & 1
\end{bmatrix}^T$. A LMIE of mass $m^*_i$ is rigidly attached to the end-effector of the $i$-th manipulator. We first use our method to synthesize the manipulating force distribution using the parametrized Moore-Penrose pseudo-inverse of the grasp matrix (\ref{eq:GoodParametrizedInverse}) and the resultant wrench $\mb{h}_o$ obtained with (\ref{eq:simulation_ho}). We then compare the motion induced on the LMIEs with the ground truth obtained from the prescribed trajectory. 

With knowledge of the manipulating wrench distribution and the applied wrenches, we can also calculate the theoretical constraint forces $\mb{f}_c$ needed to induce the correct constrained accelerations on the LMIEs using (\ref{eq:ForceAccelField}) and (\ref{eq:constraintWrenchAnalysis}). We track the norm of each constraint force $\mb{f}_{c,i}$ and compare it to that of the corresponding correction force $\mb{f}_{pd,i}$. The results are shown in Figure \ref{fig:SimResults-Equimomental}. We see that the correction forces $\mb{f}_{pd, i}$ are very low, not exceeding 2.5 N. Predictably, the norms of the theoretical constraint forces obtained with our wrench decomposition method remain strictly at zero for the duration of the trajectory because we are explicitly prescribing the manipulating force distribution as our feed-forward force command.

\begin{figure}
    \centering
    \includegraphics[width=1\linewidth]{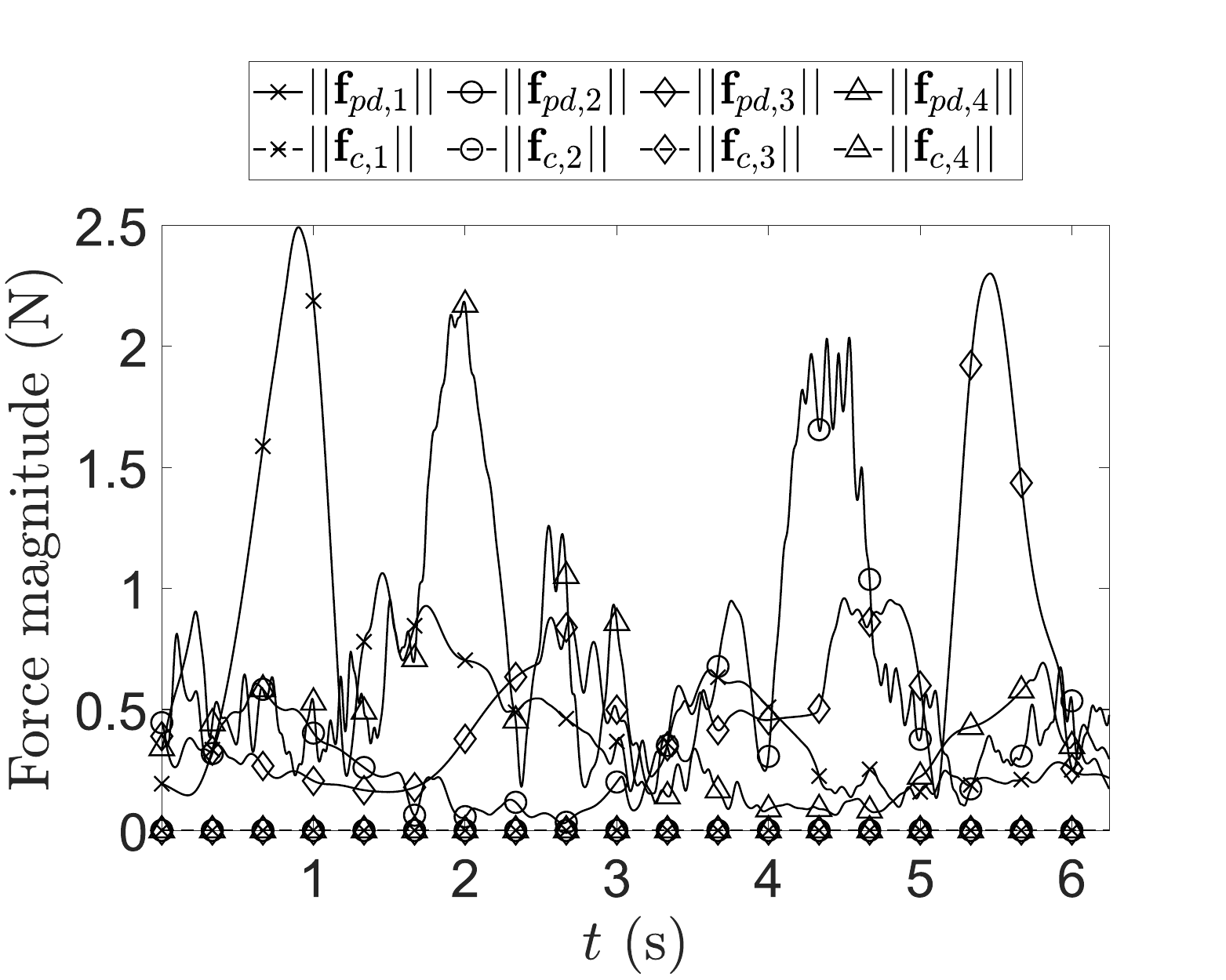}
    \caption{Magnitudes of correction forces $\mb{f}_{pd,i}$ and constraint forces $\mb{f}_{c,i}$ over one cycle of trajectory using our wrench synthesis method.}
    \label{fig:SimResults-Equimomental}
\end{figure}

We now compare with QP optimization using two different weighting matrices. The first method involves weighting each actuator based on its maximum torque value. This allocates more load to the stronger actuators to maximize the load capacity of the manipulators \cite{Xing2026}. Weighting matrix $\mb{W}$ is written as 

\begin{equation}
    \mb{W} = \begin{bmatrix}
        \mb{W}_{1}& \mb{0}_6 & \mb{0}_6 & \mb{0}_6 \\
        \mb{0}_6 & \mb{W}_{2} & \mb{0}_6 & \mb{0}_6 \\
        \mb{0}_6 & \mb{0}_6 & \mb{W}_{3} & \mb{0}_6 \\
        \mb{0}_6 & \mb{0}_6 & \mb{0}_6 & \mb{W}_{4}
    \end{bmatrix}
\end{equation}
where

\begin{equation}
    \mb{W}_i = \mb{J}_i \mb{W}_{\tau,i} \mb{J}_i^T.
\end{equation}
$\mb{W}_{\tau,i}$ is the diagonal vector containing the actuator torque limits written as

\begin{equation}
    \mb{W}_{\tau,, i} = \text{diag}\left(\frac{1}{\tau_{max,1}^2}, \ldots, \frac{1}{\tau_{max,6}^2}\right).
\end{equation}
The maximum torque values for the six joints on each of the four UR5 robots are chosen to be $\tau_{max,1} = \tau_{max,2} = \tau_{max,3} = 150$ Nm and $\tau_{max,4} = \tau_{max,5} = \tau_{max,6} = 28$ Nm.

The norms of the correction and constraint forces generated during the trajectory using this method are shown in Figure \ref{fig:SimResults-LoadCapacity}. In this case, the correction forces $\mb{f}_{pd,i}$ are significantly higher, which will inevitably cause internal stress within the payload. We also observe that the theoretical constraint forces calculated with our wrench decomposition method follow the measured values very closely. The measured values occasionally oscillate around the theoretical value; we suspect this is due to a poorly tuned PD controller for $\mb{f}_{pd,i}$, as little effort was invested in optimizing the gain parameters. 

\begin{figure}
    \centering
    \includegraphics[width=1\linewidth]{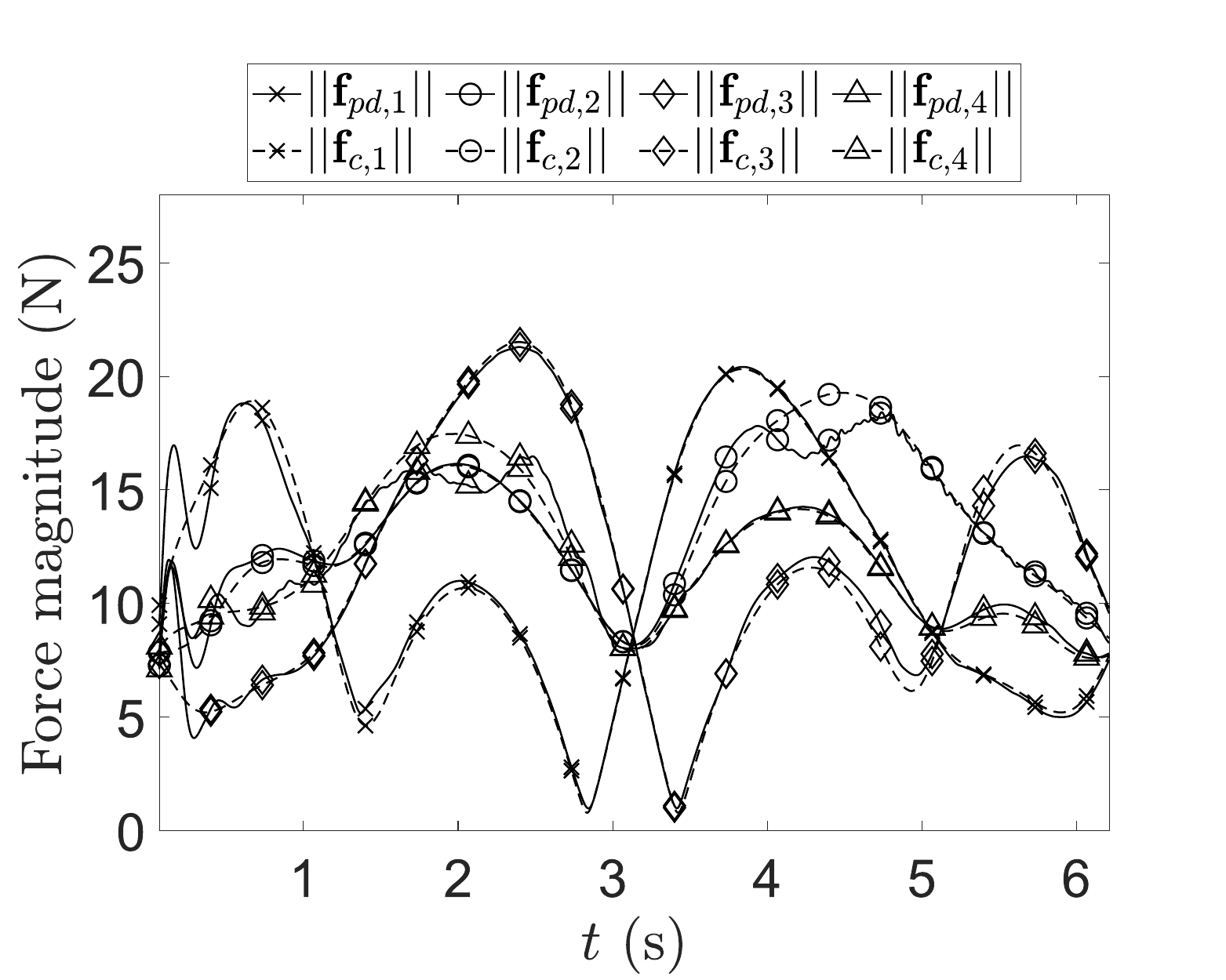}
    \caption{Magnitudes of correction forces $\mb{f}_{pd,i}$ and constraint forces $\mb{f}_{c,i}$ over one cycle of trajectory using torque limit weighting matrix}
    \label{fig:SimResults-LoadCapacity}
\end{figure}

The second QP method utilizes the operational space inertia matrix $\boldsymbol{\Lambda}_i$ of each manipulator which maps the spatial forces to the spatial accelerations through the joint inertia matrix \cite{Jain2014}. By choosing

\begin{equation}
    \mb{W} = \boldsymbol{\Lambda}_i^{-1} = \mb{J}_i \mb{M}_{m,i}^{-1} \mb{J}_i^T,
\end{equation}
where $\mb{M}_{m,i}$ is again the manipulator inertia matrix, we weight each actuator by the apparent inertia of the manipulator as felt at that specific actuator. For example, the shoulder joint of a horizontally outstretched manipulator will be assigned a smaller portion of the payload since it is already operating closer to its limits simply to support and accelerate the manipulator's own mass. Using this weighting matrix, we obtain the results shown in Figure \ref{fig:SimResults-KineticEnergy}.

\begin{figure}
    \centering
    \includegraphics[width=1\linewidth]{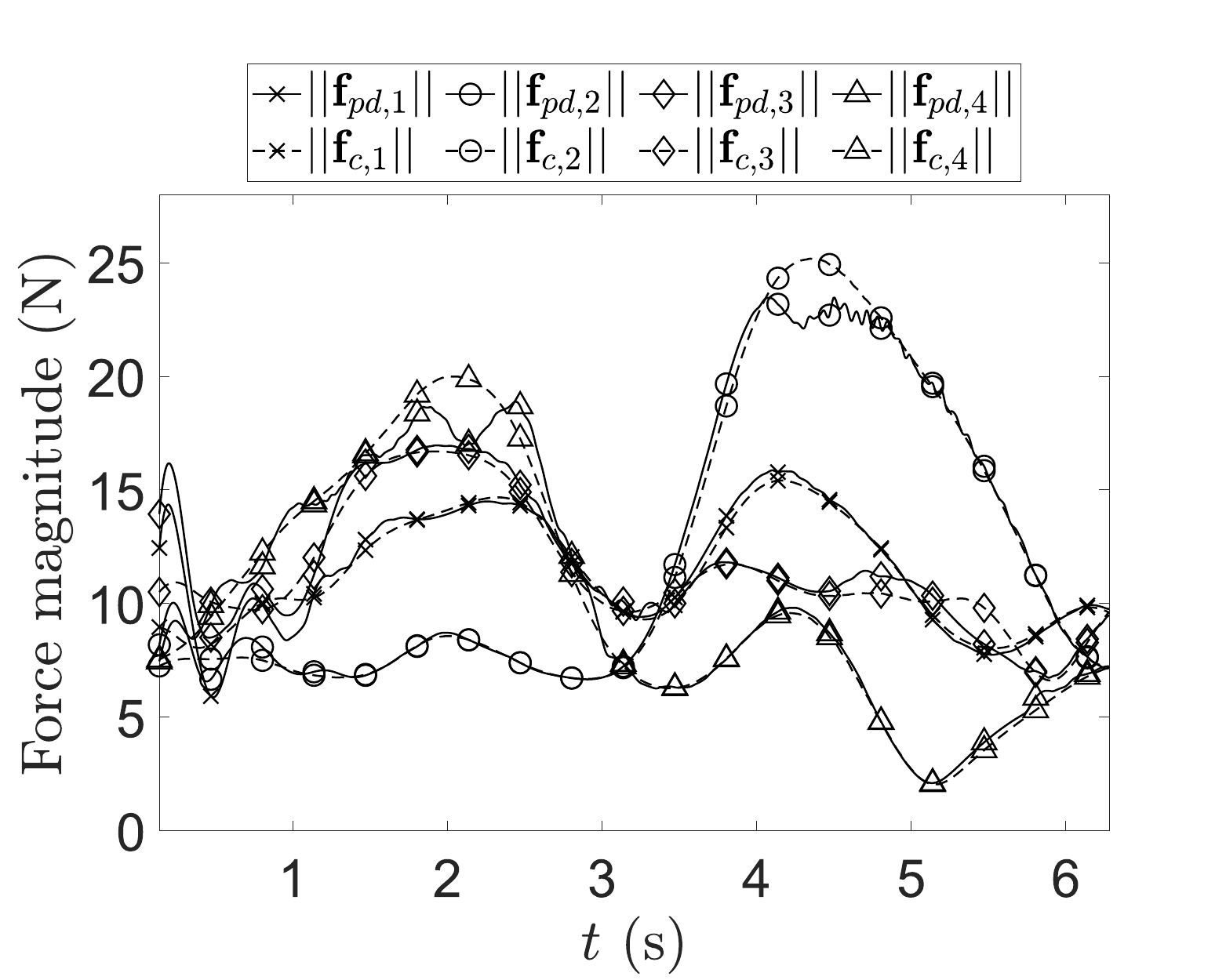}
    \caption{Magnitudes of correction forces $\mb{f}_{pd,i}$ and constraint forces $\mb{f}_{c,i}$ over one cycle of trajectory using operational space inertia matrix.}
    \label{fig:SimResults-KineticEnergy}
\end{figure}

Once again, we see that the correction forces needed to keep the manipulator end-effectors following the prescribed path of the payload are quite large, but that our proposed wrench decomposition method is capable of predicting them efficiently. 


\section{Conclusion}
The aim of this work is to propose a generalized theory of load distribution in rigid bodies subjected to wrenches applied by multiple closed-loop kinematic chains. We first examine the physical significance of well established principles of rigid body motion described by the Udwadia-Kalaba equation (\cite{Udwadia2010}) to better highlight the roles of constraint and manipulating wrenches. Then, we present our theorem which unifies many existing theories and methods related to wrench synthesis and analysis in the context of redundantly-actuated robot force control.

Our theorem fully characterizes the feasible set of manipulating wrench distributions $\mb{h}_m$ for a given resultant wrench $\mb{h}_o$. We examine in detail the uniqueness of the feasible manipulating wrench distribution set, show how the concept may be extended to wrench synthesis and wrench analysis, and determine a unique metric to describe the internal loading state of a rigid body for a given set of applied wrenches. We also provide a detailed review of the implications of our proposed theorem for other approaches in the literature. Finally, case studies are presented to illustrate the advantages of our proposed methods. 

This general theory and the methods derived from it are intended to supplement modern robotics controllers used in real-time force control applications. Our results are proposed in the hopes that we may consider the problem of load distributions in rigid bodies handled by redundantly-actuated closed kinematic chains---either the determination of load distributions with prescribed internal loads or the decomposition of known distributions into manipulating and constraint components---to be a solved one. Such a complete description of the nature of internal loads combined with the methods that we propose for wrench synthesis and analysis may prove to be a useful tool for the real-time force control of redundantly-actuated robotic systems since examples of multiple wrenches applied to a single object are seen in many areas of robotics. 

\section*{Statements and Declarations}
\subsection*{Ethical Considerations}
Not applicable.

\subsection*{Consent to Participate}
Not applicable. 

\subsection*{Consent for Publication}
Not applicable.

\begin{funding}
    The authors would like to acknowledge the financial support of the Natural Sciences and Engineering Research Council of Canada (NSERC) through a scholarship (Canada Graduate Research Scholarship — Doctoral) to the first author and discovery grant DG-89715 to the second author. They would also like to acknowledge the financial support of the Institut de recherche Robert-Sauvé en santé et sécurité du travail (IRSST) through a doctoral scholarship to the first author. 
\end{funding}

\begin{dci}
    The authors declared no potential conflicts of interest with respect to the research, authorship, and/or publication of this article.
\end{dci}

\begin{sm}
    Not applicable.
\end{sm}





\theendnotes

\bibliographystyle{SageH}
\bibliography{bibliography.bib}

\section*{Appendix}
\subsection{Proof of Theorem \ref{th:InternalLoadsAsHelicoidalField}}

\begin{proof}
    1. The resultant wrench about the centroid of the rigid body is obtained from (\ref{eq:FullDynamics}) and the virtual rigid body inertia matrix is 

    \begin{equation}
        \mb{M}^*_o = \begin{bmatrix}
            m^*_o \mb{I}_3 & \mb{0} \\
            \mb{0} & \mb{J}^*_o
        \end{bmatrix}.
    \end{equation}
    Hence, the virtual instantaneous acceleration of the virtual rigid body subjected to the resultant wrench is

    \begin{equation}
        \mb{M}^*_o \ddot{\mb{x}}^*_o = \mb{h}_o \quad \Longrightarrow \quad \ddot{\mb{x}}^*_o = (\mb{M}^*_o)^{-1} \mb{h}_o.
    \end{equation}

    2. The kinematic constraints (\ref{eq:LinearKinematicConstraint}) and (\ref{eq:AngularKinematicConstraint}) constrain the virtual instantaneous acceleration of the $i$-th LMIE at $\mb{r}_i$ to be 

    \begin{equation}
        \ddot{\mb{x}}^*_i = \begin{bmatrix}
            \ddot{\mb{p}}^*_o  + \dot{\boldsymbol{\omega}}^*_o \times \mb{r}_i + \boldsymbol{\omega}^*_o \times (\boldsymbol{\omega}^*_o \times \mb{r}_i) \\
            \dot{\boldsymbol{\omega}}^*_o
        \end{bmatrix},
    \end{equation}
    but these vectors do not lie on a helicoidal field because of the normal acceleration term $\boldsymbol{\omega}^*_o \times (\boldsymbol{\omega}^*_o \times \mb{r}_i)$. However, since constraint wrenches can be generated even when no wrenches are applied by the kinematic chains as shown in Section \ref{sec:CentripetalForceImplication}, we assume the virtual rigid body to be stationary and this term vanishes.

    The constrained virtual acceleration $\ddot{\mb{x}}^*_i$ at any location $\mb{r}_i$ with respect to the CoM of the rigid body is then obtained from only the linear acceleration $\ddot{\mb{p}}^*_o$ and the angular acceleration $\dot{\boldsymbol{\omega}}^*_o$ of the virtual rigid body.

    {\small
    \begin{equation}
        \ddot{\mb{x}}^*_i = \begin{bmatrix}
            \ddot{\mb{p}}^*_o  + \dot{\boldsymbol{\omega}}^*_o \times \mb{r}_i \\
            \dot{\boldsymbol{\omega}}^*_o
        \end{bmatrix}.
    \end{equation}}

    3. Each wrench $\mb{h}_i$ applied to an LMIE at $\mb{r}_i$ induces an unconstrained virtual acceleration given by

    \begin{equation}
        (\ddot{\mb{x}}^d_i)^* = (\mb{M}^*_i)^{-1} \mb{h}_i, \quad \mb{M}^*_i = \begin{bmatrix}
            m^*_i \mb{I}_3 & \mb{0} \\
            \mb{0} & \mb{J}^*_i
        \end{bmatrix}.
    \end{equation}

    4. According to the Udwadia-Kalaba equation (\cite{Udwadia2010}), a constraint wrench $\mb{h}_{c,i}$ must be applied at $\mb{r}_i$ if the unconstrained virtual acceleration $(\ddot{\mb{x}}^d_i)^* = (\mb{M}^*_i)^{-1} \mb{h}_i$ that $\mb{h}_i$ would produce differs from the local constrained virtual acceleration $\ddot{\mb{x}}^*_i$ demanded by the rigid body motion. From (\ref{eq:constrainedEquationofMotion}), the constraint wrench can be written as 

    \begin{equation}
        \mb{h}_{c,i} = \mb{M}^*_i\left(  \ddot{\mb{x}}^*_i - (\mb{M}^*_i)^{-1} \mb{h}_i \right)
    \end{equation}
    which ensures that the constrained acceleration of the LMIE at $\mb{r}_i$ is compatible with the kinematic constraints. It is now clear that $\mb{h}_{c,i} = \mb{0} \iff (\mb{M}^*_i)^{-1}\mb{h}_i = \ddot{\mb{x}}^*_i$

    5. Since any nonzero $\mb{h}_{c,i}$ indicates the existence of a constraint wrench, we conclude 

    \begin{equation}
        \mb{h}_{c,i} = \mb{0} \; \forall i \iff \mb{h}_i = \mb{h}_{m,i} \ \forall i,
    \end{equation}
    i.e., for there to be no constraint wrenches, each local unconstrained acceleration $(\ddot{\mb{x}}^d_i)^*$ induced by $\mb{h}_i$ must lie on the helicoidal acceleration field defined by $\ddot{\mb{x}}^*_o$.
\end{proof}

\subsection{Correction of Theorem \ref{th:GoodParamMPInv}}

\begin{proof}
    From (\ref{eq:RBDynamics}), $\ddot{\mb{p}}_o$ and $\dot{\boldsymbol{\omega}}_o$ are explicitly written as
    
    \begin{equation}
    \label{eq:LinAccel}
        \ddot{\mb{p}}^*_o = \frac{1}{m^*_o}\mb{f}_o,
    \end{equation}
    and
    
    \begin{equation}
    \label{eq:AngularAccel}
        \dot{\boldsymbol{\omega}}^*_o = [\mb{J}^*_o]^{-1} \mb{t}_o,
    \end{equation}
    where $\mb{f}_o$ and $\mb{t}_o$ are the resultant force and torque, respectively. 

    Substituting (\ref{eq:LMIEConstrainedAcceleration}) into (\ref{eq:LMIEDynamics}), we obtain
    
    \begin{equation}
    \label{eq:ManipulatingWrenchDynamics}
        \mb{h}_{m,i} = \begin{bmatrix}
            m^*_i \mb{I}_3 & \mb{0} \\
            \mb{0} & \mb{J}^*_i
        \end{bmatrix} \begin{bmatrix}
            \ddot{\mb{p}}^*_o + \dot{\boldsymbol{\omega}}^*_o \times \mb{r}_i \\ \dot{\boldsymbol{\omega}}^*_o
        \end{bmatrix}.
    \end{equation}
    Finally, we substitute (\ref{eq:LinAccel}) and (\ref{eq:AngularAccel}) into (\ref{eq:ManipulatingWrenchDynamics}) to obtain explicit expressions for the force and torque components of the manipulating wrenches. The manipulating force is derived as 
    
    \begin{align}
        \mb{f}_{m,i} &= \frac{m^*_i}{m^*_o}\mb{f}_o + m^*_i ([\mb{J}^*_o]^{-1} \mb{t}_o) \times \mb{r}_i \\
        \mb{f}_{m,i} &= \frac{m^*_i}{m^*_o}\mb{f}_o - m^*_i \mb{r}_i \times ([\mb{J}^*_o]^{-1} \mb{t}_o) \\
        \label{eq:ManipulatingForce}
        \mb{f}_{m,i} &= \frac{m^*_i}{m^*_o}\mb{f}_o + m^*_i \mb{S}( \mb{r}_i)^T ([\mb{J}^*_o]^{-1} \mb{t}_o),
    \end{align}
    where $\mb{S}(\mb{r}_i)^T = -\mb{S}(\mb{r}_i)$ is used. The manipulating torque is written as 
    
    \begin{equation}
        \mb{t}_{m,i} = \mb{J}^*_i ([\mb{J}^*_o]^{-1} \mb{t}_o).
    \end{equation}
    When we write these equations for every manipulating wrench and collect them in matrix form, we obtain 
    
    \begin{equation}
    \begin{split}
        \mb{h}_m &= \mb{G}^+_M \mb{h}_o \\ &= \begin{bmatrix}
            \frac{m^*_1}{m^*_o} \mb{I}_3 &  m^*_1 \mb{S}(\mb{r}_1)^T [\mb{J}^*_o]^{-1} \\
            \mb{0}_3 & \mb{J}^*_1 [\mb{J}^*_o]^{-1} \\
            \vdots & \vdots \\
            \frac{m^*_n}{m^*_o} \mb{I}_3 &  m^*_n \mb{S}(\mb{r}_n)^T [\mb{J}^*_o]^{-1} \\
            \mb{0}_3 & \mb{J}^*_n [\mb{J}^*_o]^{-1}
        \end{bmatrix} \begin{bmatrix}
            \mb{f}_o \\
            \mb{t}_o
        \end{bmatrix},
    \end{split}
    \end{equation}
    where $\mb{G}^+_M$ differs from the form shown in (\ref{eq:BadParametrizedInverse}) which violates the non-commutative property of the matrices in (\ref{eq:ManipulatingForce}).
\end{proof}

\subsection{Proof of Corollary \ref{cor:MPInverseG}}
\begin{proof}
The grasp matrix $\mb{G}$ has the form shown in (\ref{eq:GForcesandMoments}). The unweighted Moore-Penrose pseudo-inverse $\mb{G}^{\dagger}$ is written as 

\begin{equation}
    \mb{G}^{\dagger} = \mb{G}^T(\mb{G}\mb{G}^T)^{-1}
\end{equation}
and is the solution to the quadratic optimization problem

\begin{equation}
\label{eq:ManipulatingWrenchOptimizationProblem}
    \begin{aligned}
        \text{minimize}& \quad \mb{h}^T \mb{h} \\ 
        \text{subject to}& \quad \mb{h}_o = \mb{G} \mb{h}.
    \end{aligned}
\end{equation}

It can easily be verified that

\begin{equation}
    \mb{G}^T \mb{G} = \begin{bmatrix}
        n \mb{I}_3 & \sum_{i=1}^n \mb{S}(\mb{r}_i)^T \\ \sum_{i=1}^n \mb{S}(\mb{r}_i) & n \mb{I}_3 + \sum_{i=1}^n \mb{S}(\mb{r}_i) \mb{S}(\mb{r}_i)^T
    \end{bmatrix}
    \label{eq:G^TG}
\end{equation}
where $\sum_{i=1}^n \mb{S}(\mb{r}_i)$ and $\sum_{i=1}^n \mb{S}(\mb{r}_i)^T$ can be rewritten as 

\begin{equation}
\begin{split}
    &\sum_{i=1}^n \mb{S}(\mb{r}_i) = \mb{S}\left( \sum_{i=1}^n \mb{r}_i \right) \  \text{and} \\ &\sum_{i=1}^n \mb{S}(\mb{r}_i)^T = \mb{S} \left( \sum_{i=1}^n \mb{r}_i \right)^T,
\end{split}
    \label{eq:SkewSymMatrices}
\end{equation}
respectively. Eq. (\ref{eq:SkewSymMatrices}) is equivalent to the skew symmetric matrix that corresponds to the cross product with the vector in constraint (\ref{eq:virtualCoMEquivalence}) when $m^*_i = 1 \ \forall i$. Therefore, when this constraint is satisfied, these matrices become $\mb{0}_3$ and the inverse $(\mb{G}^T \mb{G})^{-1}$ takes the simpler form

\begin{equation}
    (\mb{G}^T \mb{G})^{-1} = \frac{1}{n} \begin{bmatrix}
        \mb{I}_3 & \mb{0}_3 \\
        \mb{0}_3 & \mb{I}_3 + \frac{1}{n} \sum_{i=1}^n \mb{S}(\mb{r}_i) \mb{S}(\mb{r}_i)^T
    \end{bmatrix}.
    \label{eq:inv(G^TG)}
\end{equation}
Finally, we premultiply (\ref{eq:inv(G^TG)}) by $\mb{G}^T$ to obtain

\begin{equation}
    \mb{G}^{\dagger} = \frac{1}{n} \begin{bmatrix}
            \mb{I}_3 & \mb{S}(\mb{r}_1)^T \Bar{\mb{J}}^{-1} \\ \mb{0}_{3 \times 3} & \Bar{\mb{J}}^{-1} \\ 
            \vdots & \vdots \\
            \mb{I}_3 & \mb{S}(\mb{r}_n)^T \Bar{\mb{J}}^{-1} \\ \mb{0}_{3 \times 3} & \Bar{\mb{J}}^{-1}
        \end{bmatrix}
\end{equation}
with $\Bar{\mb{J}} = \mb{I}_3 + \frac{1}{n} \sum_{i=1}^n \mb{S}(\mb{r}_i) \mb{S}(\mb{r}_i)^T$. Note that the error related to matrix non-commutativity that is present in (\ref{eq:BadParametrizedInverse}) is also present in (\ref{eq:BadMPInverseG}). This is corrected in (\ref{eq:GoodMPInverseG}).

If $m^*_i = 1 \ \forall i$ does not solve (\ref{eq:virtualCoMEquivalence}), the inverse of (\ref{eq:G^TG}) is more complex, and $\mb{G}^{\dagger}$ and $\mb{G}^+_M$ for $m^*_i = 1$ and $\mb{J}^*_i = \mb{I}_3$ no longer return the same solution.    
\end{proof}

\end{document}